\documentclass[11pt]{article}

\usepackage{natbib}

\title{Best of Both Worlds Policy Optimization}

\setcitestyle{authoryear,round,citesep={;},aysep={,},yysep={;}}
\usepackage[english]{babel}

\usepackage{titletoc}
\usepackage[toc,page,header]{appendix}

\usepackage{amsmath}
\usepackage{amssymb}
\usepackage{amsthm}
\usepackage{graphicx}

\usepackage{xspace}
\usepackage{amsmath}
\usepackage{graphicx}
\usepackage{framed}
\usepackage{bbm}
\usepackage{algorithm}
\usepackage{amsthm}
\usepackage[algo2e, vlined, ruled, linesnumbered]{algorithm2e}
\usepackage{natbib}
\usepackage{mathrsfs}
\usepackage{amssymb}
\usepackage{bm}
\usepackage[colorlinks=true, linkcolor=blue, citecolor=blue, breaklinks=true, linktocpage=true]{hyperref}
\usepackage{tabulary}

\usepackage{prettyref}
\usepackage{mathtools}
\usepackage{amsmath}
\usepackage[dvipsnames]{xcolor}
\usepackage{multirow}
\usepackage{nicefrac}
\usepackage{amssymb}% http://ctan.org/pkg/amssymb
\usepackage{pifont}% http://ctan.org/pkg/pifont
\definecolor{Green}{rgb}{0.13, 0.65, 0.3}
\usepackage{colortbl}

\newcommand{\calC}{\mathcal{C}}

\DeclareMathOperator*{\argmin}{argmin} % no space, limits underneath in displays
\DeclareMathOperator*{\argmax}{argmax} % no space, limits underneath in displays

\newcommand{\Reg}{\text{\rm Reg}}

\newcommand{\calT}{\mathcal{T}}

\newcommand{\calA}{\mathcal{A}}

\newcommand{\calE}{\mathcal{E}}

\newcommand{\calS}{\mathcal{S}}

\newcommand{\calP}{\mathcal{P}}

\newcommand{\calF}{\mathcal{F}}

\newcommand{\ind}{\mathbb{I}}
\newcommand{\hatP}{\widehat{P}}

\newcommand{\hatL}{\widehat{L}}

\newcommand{\picirc}{\mathring{\pi}}

\newcommand{\hatQ}{\widehat{Q}}

\newcommand{\poly}{\textit{poly}}
\newcommand{\E}{\mathbb{E}}

\newcommand{\hatell}{\widehat{\ell}}

\newcommand{\calX}{\mathcal{X}}

\newcommand{\one}{\mathbf{1}}
\newcommand{\inner}[1]{\langle#1\rangle}

\newcommand{\term}{\textbf{term}}

\newcommand{\tildeP}{\widetilde{P}}

\newcommand{\regterm}{\textbf{{\upshape ftrl-reg}}}
\newcommand{\biasterm}{\textbf{{\upshape bias}}}

\newcommand{\scalemath}[2]{\scalebox{#1}{$\displaystyle #2$}}

\newcommand{\nonl}{\renewcommand{\nl}{\let\nl}}

\usepackage{prettyref}
\newcommand{\pref}[1]{\prettyref{#1}}

\newcommand{\savehyperref}[2]{\texorpdfstring{\hyperref[#1]{#2}}{#2}}
\newrefformat{eq}{\savehyperref{#1}{Eq.~\textup{(\ref*{#1})}}}
\newrefformat{eqn}{\savehyperref{#1}{Equation~\ref*{#1}}}
\newrefformat{lem}{\savehyperref{#1}{Lemma~\ref*{#1}}}
\newrefformat{lemma}{\savehyperref{#1}{Lemma~\ref*{#1}}}
\newrefformat{def}{\savehyperref{#1}{Definition~\ref*{#1}}}
\newrefformat{line}{\savehyperref{#1}{Line~\ref*{#1}}}
\newrefformat{thm}{\savehyperref{#1}{Theorem~\ref*{#1}}}
\newrefformat{corr}{\savehyperref{#1}{Corollary~\ref*{#1}}}
\newrefformat{cor}{\savehyperref{#1}{Corollary~\ref*{#1}}}
\newrefformat{sec}{\savehyperref{#1}{Section~\ref*{#1}}}
\newrefformat{subsec}{\savehyperref{#1}{Section~\ref*{#1}}}
\newrefformat{app}{\savehyperref{#1}{Appendix~\ref*{#1}}}
\newrefformat{assum}{\savehyperref{#1}{Assumption~\ref*{#1}}}
\newrefformat{ex}{\savehyperref{#1}{Example~\ref*{#1}}}
\newrefformat{fig}{\savehyperref{#1}{Figure~\ref*{#1}}}
\newrefformat{alg}{\savehyperref{#1}{Algorithm~\ref*{#1}}}
\newrefformat{rem}{\savehyperref{#1}{Remark~\ref*{#1}}}
\newrefformat{conj}{\savehyperref{#1}{Conjecture~\ref*{#1}}}
\newrefformat{prop}{\savehyperref{#1}{Proposition~\ref*{#1}}}
\newrefformat{proto}{\savehyperref{#1}{Protocol~\ref*{#1}}}
\newrefformat{prob}{\savehyperref{#1}{Problem~\ref*{#1}}}
\newrefformat{claim}{\savehyperref{#1}{Claim~\ref*{#1}}}
\newrefformat{que}{\savehyperref{#1}{Question~\ref*{#1}}}
\newrefformat{op}{\savehyperref{#1}{Open Problem~\ref*{#1}}}
\newrefformat{fn}{\savehyperref{#1}{Footnote~\ref*{#1}}}
\newrefformat{tab}{\savehyperref{#1}{Table~\ref*{#1}}}
\newrefformat{fig}{\savehyperref{#1}{Figure~\ref*{#1}}}
\newrefformat{proc}{\savehyperref{#1}{Procedure~\ref*{#1}}}

\usepackage{caption}

\usepackage{accents}
\usepackage{fancyhdr}

  \newtheorem{theorem}{Theorem}
  \newtheorem{lemma}[theorem]{Lemma}

  \newtheorem{definition}[theorem]{Definition}

\usepackage{times}
\date{}

\author{%
    Christoph Dann~\thanks{Google Research. Email: \texttt{cdann@cdann.net}.}\and Chen-Yu Wei~\thanks{MIT Institute for Data, Systems, and Society. Email: \texttt{chenyuw@mit.edu}.}\and Julian Zimmert~\thanks{Google Research. Email: \texttt{zimmert@google.com}.}
}

\begin{document}
%\doparttoc 
%\faketableofcontents 
%\part{} % Start the document part
%\parttoc % Insert the document TOC

\maketitle

\begin{abstract}
Policy optimization methods are popular reinforcement learning algorithms in practice. Recent works have built theoretical foundation for them by proving $\sqrt{T}$ regret bounds even when the losses are adversarial. Such bounds are tight in the worst case but often overly pessimistic. In this work, we show that in tabular Markov decision processes (MDPs), by properly designing the regularizer, the exploration bonus and the learning rates, one can achieve a more favorable $\text{polylog}(T)$ regret when the losses are stochastic, without sacrificing the worst-case guarantee in the adversarial regime. To our knowledge, this is also the first time a gap-dependent $\text{polylog}(T)$ regret bound is shown for policy optimization. Specifically, we achieve this by leveraging a Tsallis entropy or a Shannon entropy regularizer in the policy update. Then we show that under known transitions, we can further obtain a first-order regret bound in the adversarial regime by leveraging the log~barrier regularizer. 
\end{abstract}
\section{Introduction}
Policy optimization methods have seen great empirical success in various domains \citep{schulman2017proximal, levine2013guided}. An appealing property of policy optimization methods is the local-search nature, which lends itself to an efficient implementation as a search over the whole MDP is avoided. However, this property also makes it difficult to obtain global optimality guarantees for these algorithms and a large portion of the literature postulates
strong and often unrealistic assumptions to ensure global exploration \citep[see e.g.,][]{abbasi2019politex, agarwal2020optimality, neu2021online, wei2021learning}.
Recently, the need for extra assumptions has been overcome by adding exploration bonuses to the update \citep{cai2020provably, shani2020optimistic, agarwal2020pc, zanette2020learning, luo2021policy}. These works demonstrate an additional robustness property of policy optimization, which is able to handle adversarial losses or some level of corruption. \citet{luo2021policy} and \citet{chen2022policy} even managed to obtain the optimal $\sqrt{T}$ rate.  

However, when the losses are in fact stochastic, the $\sqrt{T}$ minimax regret is often overly pessimistic and $\log(T)$ with problem-dependent factors is the optimal rate \citep{lai1985asymptotically}. 
Recently, \citet{jin2021best} obtained a best-of-both-worlds algorithm that automatically adapts to the nature of the environment, a method which relies on FTRL with a global regularizer over the occupancy measure.

In this work, we show that by properly assigning the bonus and tuning the learning rates, policy optimization 
can also achieve the best of both worlds, which gives a more computationally favorable solution than \citet{jin2021best} for the same setting. 
%We obtain these results for all of the most prominent regularizers used for $K$-armed bandits (Shannon entropy, Tsallis entropy, and log barrier). 
Specifically, we show that policy optimization with Tsallis entropy or Shannon entropy regularizer achieves $\sqrt{T}$ regret in the adversarial regime and $\text{polylog}(T)$ regret in the stochastic regime. The $\sqrt{T}$ can further be improved to $\sqrt{L}$ if the transition is known and if a log-barrier regularizer is used, where $L$ is the cumulative loss of the best policy. Though corresponding results in multi-armed bandits have been well-studied, new challenges arise in the MDP setting which require non-trivial design for the exploration bonus and the learning rate scheduling. The techniques we develop to address these issues constitute the main contribution of this work. 
\section{Related Work}
For multi-armed bandits, the question whether there is a single algorithm achieving near-optimal regret bounds in both the adversarial and the stochastic regimes was first asked by \citet{bubeck2012best}. A series of followup works  refined the bounds through different techniques \citep{seldin2014one, auer2016algorithm, seldin2017improved, wei2018more, zimmert2019optimal, ito2021parameter}. One of the most successful approaches is developed by \citet{wei2018more, zimmert2019optimal, ito2021parameter}, who demonstrated that a simple Online Mirror Descent (OMD) or Follow the Regularized Leader (FTRL) algorithm, which was originally designed only for the adversarial case, is able to achieve the best of both worlds. This approach has been adopted to a wide range of problems including semi-bandits \citep{zimmert2019beating}, graph bandits \citep{erez2021best,ito2022nearly}, partial monitoring \citep{tsuchiya2022best},  multi-armed bandits with switching costs \citep{rouyer2021algorithm, amir2022better}, tabular MDPs \citep{jin2020simultaneously, jin2021best}, and others. Though under a similar framework, each of them addresses new challenges that arises in their specific setting. 

Previous works that achieve the best of both worlds in tabular MDPs \citep{jin2020simultaneously, jin2021best} are based on FTRL over the \emph{occupancy measure} space. This approach has several shortcomings, making it less favorable in practice. First, the feasible set of occupancy measure depends on the transition kernel, so the extension to a model-free version is difficult. Second, since the occupancy measure space is a general convex set that may change over time as the learner gains more knowledge about transitions, it requires solving a different convex programming in each round. In contrast, policy optimization is easier to extend to settings where transitions are hard to learn, and it is computationally simple --- in tabular MDPs, it is equivalent to running an individual multi-armed bandit algorithm on each state. 

Due to its local search nature, exploration under policy optimization is non-trivial, especially when coupled with bandit feedback and adversarial losses. In a simpler setting where the loss feedback has full information, \citet{he2022near, cai2020provably} showed $\sqrt{T}$ regret for linear mixture MDPs using policy optimization. In another simpler setting where the loss is stochastic, \citet{agarwal2020pc, zanette2021cautiously} showed $\text{poly}(1/\epsilon)$ sample complexity for linear MDPs. The work by \citet{shani2020optimistic} first studied policy optimization with bandit feedback and adversarial losses, and obtained a $T^{2/3}$ regret for tabular MDPs. \citet{luo2021policy} improved it to the optimal $\sqrt{T}$, and provided extensions to linear-Q and linear MDPs. In this work, we demonstrate another power of policy optimization by showing a best-of-both-world regret bound in tabular MDPs. To our knowledge, this is also the first time a gap-dependent $\text{polylog}(T)$ regret bound is shown for policy optimization. 

We also note that a first-order bound has been shown for adversarial MDPs by \citet{lee2020bias}. Their algorithm is based on regularization on the occupancy measure, and does not rely on knowledge of the transition kernel. On the other hand, our first-order bound currently relies on the learner knowing the transitions. Whether it can be achieved under unknown transitions is an open question. 

\section{Notation and Setting }

\paragraph{Notation} For $f\in\mathbb{R}$ and $g\in \mathbb{R}_+$, we use $f\lesssim g$ or $f\leq O(g)$ to mean that $f\leq c\cdot g$ for some absolute constant $c>0$. $[x]_+\triangleq \max\{x,0\}$. $\triangle(\calX)$ denotes the probability simplex over the set $\calX$. 

\subsection{MDP setting}

We consider episodic fixed-horizon MDPs. Let $T$ be the total number of episodes. The MDP is described by a tuple $(\calS,\calA, H, P, \{\ell_t\}_{t=1}^T)$, where $\calS$ is the state set, $\calA$ is the action set, $H$ is the horizon length, $P: \calS\times \calA\rightarrow \triangle(\calS)$ is the transition kernel so that $P(s'|s,a)$ is the probability of moving to state $s'$ after taking action $a$ on state $s$, and $\ell_t: \calS\times \calA\rightarrow [0,1]$ is the loss function in episode $t$. We define $S=|\calS|$ and $A=|\calA|$, which are both assumed to be finite. Without loss of generality, we assume $A\leq T$. A policy $\pi: \calS\rightarrow \triangle(\calA)$ describes how the player interacts with the MDP, with $\pi(\cdot|s)\in\triangle(\calA)$ being the action distribution the player uses to select actions in state $s$. If for all $s$, $\pi(\cdot|s)$ is only supported on one action, we call $\pi$ a deterministic policy, and we abuse the notation $\pi(s)\in\calA$ to denote the action $\pi$ chooses on state $s$. 

Without loss of generality, we assume that the state space can be partitioned into $H+1$ disjoint layers $\calS=\calS_0\cup \calS_1\cup \cdots \cup\calS_H$, and the transition is only possible from one layer to the next (i.e., $P(\cdot|s,a)$ is only supported on $\calS_{h+1}$ if $s\in\calS_h$). Without loss of generality, we assume that $\calS_0=\{s_0\}$ (initial state) and $\calS_H=\{s_H\}$ (terminal state). Also, since there is at least one state on each layer, it holds that $H\leq S$. Let $h(s)$ denotes the layer where state~$s$ lies. 

The environment decides $P$ and $\{\ell_t\}_{t=1}^T$ ahead of time. In episode $t$, the learner decides on a policy $\pi_t$. Starting from the initial state $s_{t,0}=s_0$, the learner repeatedly draws action $a_{t,h}$ from $\pi_t(\cdot|s_{t,h})$ and transitions to the next state $s_{t,h+1}\in \calS_{h+1}$ following $s_{t,h+1}\sim P(\cdot|s_{t,h},a_{t,h})$, until it reaches the terminal state $s_{t,H}=s_H$. The learner receives $\{\ell_t(s_{t,h}, a_{t,h})\}_{h=0}^{H-1}$ at the end of episode $t$. 

For a policy $\pi$ and a loss function $\ell$, we define $V^{\pi}(s_H;\ell)=0$ and recursively define 
\begin{align}
    Q^{\pi}(s,a;\ell) &= \ell(s,a) + \E_{s'\sim P(\cdot|s,a)}\left[V^{\pi}(s';\ell)\right], \nonumber \\
    V^{\pi}(s;\ell) &= \sum_{a\in\calA} \pi(a|s)Q^{\pi}(s,a;\ell), \label{eq: Q, V def}
\end{align}
which are the standard \emph{state-action value function} and \emph{state value function} under policy $\pi$ and loss function $\ell$. 
%The occupancy measure under policy $\pi$ is defined as
%\begin{align*}
%    \mu^{\pi}(s,a) = \sum_{h=0}^{H}\Pr(s_h=s, a_h=a~|~\pi, P)
%\end{align*}
%where $\Pr(\cdot|\pi, P)$ denotes the probability measure induced by policy $\pi$ and transition kernel $P$. Further define $\mu^{\pi}(s)=\sum_a \mu^\pi(s,a)$. Notice that $\mu^{\pi}(s)\leq 1$ because every $s$ appears only in one layer. 

The learner's \emph{regret} with respect to a policy $\pi$ is defined as 
\begin{align*}
    \Reg(\pi) = \E\left[\sum_{t=1}^T (V^{\pi_t}(s_0; \ell_t) - V^{\pi}(s_0; \ell_t))\right]. 
\end{align*}

\subsection{Known and unknown transition}
Following \citet{jin2020simultaneously, jin2021best}, we consider both scenarios where the learner knows the transition kernel $P$ and where he does not know it. 

The empirical transition is defined by the following: 
\begin{align*}
    \hatP_t(s'|s,a) = \frac{n_t(s,a,s')}{n_t(s,a)} 
\end{align*}
where $n_t(s,a)$ is the number of visits to $(s,a)$ prior to episode $t$, and $n_t(s,a,s')$ is the number of visits to $s'$ after visiting $(s,a)$, prior to episode $t$. If $n_t(s,a)=0$, we define $\hatP_t(\cdot|s,a)$ to be uniform over the states on layer $h(s)+1$. 

In the unknown transition case, we define the confidence set of the transition: 
\begin{align}
            &\scalemath{0.96}{\calP_t = \Bigg\{\tildeP:  \forall h, \forall (s,a)\in \calS_h\times \calA,\ \  \tildeP(\cdot|s,a)\in\triangle(\calS_{h+1}),} \nonumber \\
         &\scalemath{0.96}{\left|\tildeP(s'|s,a) - \hatP_t(s'|s,a)\right|\leq 2\sqrt{\frac{\hatP_t(s'|s,a)]\iota}{n_t(s,a)}} + \frac{14\iota}{3n_t(s,a)}\Bigg\} } \label{eq: definition of transition P}
        \end{align}
        where $\iota=\ln(SAT/\delta)$. As shown in \cite{jin2020simultaneously}, $P\in \bigcap_{t=1}^T \calP_t$ with probability at least $1-4\delta$. Through out the paper, we use $\delta=\frac{1}{T^3}$. 

For an arbitrary transition kernel $\tildeP$, define 
\begin{align*}
    \mu^{\tildeP, \pi}(s,a) = \sum_{h=0}^{H}\Pr(s_h=s, a_h=a~|~\pi, \tildeP), 
\end{align*}
where $\Pr(\cdot|\pi, \tildeP)$ denotes the probability measure induced by policy $\pi$ and transition kernel $\tildeP$. Furthermore, define $\mu^{\tildeP,\pi}(s)=\sum_a \mu^{\tildeP,\pi}(s,a)$. %Notice that $\mu^{\tildeP,\pi}(s)\leq 1$ because every $s$ appears only in one layer. 
We write $\mu^{\pi}(s)=\mu^{P,\pi}(s)$ and $\mu^\pi(s,a)=\mu^{P,\pi}(s,a)$ where $P$ is the true transition. 
%We write $\mu^{\pi}(s), \mu^\pi(s,a), \mu^{\pi}(s'|s,a), V^{\pi}(s;\ell), Q^\pi(s,a;\ell)$ for the above notations if $\tildeP=P$ (the true transition kernel). 
Define the upper and lower confidence measure as
\begin{align*}
    \overline{\mu}_t^{\pi}(s) = \max_{\tildeP\in\calP_t} \mu^{\tildeP,\pi}(s), \qquad \underline{\mu}_t^{\pi}(s) = \min_{\tildeP\in\calP_t} \mu^{\tildeP,\pi}(s). 
\end{align*}
Finally, define $V^{\tildeP, \pi}(s;\ell)$ and $Q^{\tildeP,\pi}(s,a;\ell)$ to be similar to \eqref{eq: Q, V def}, with the transition kernel replaced by $\tildeP$. 

\subsection{Adversarial versus stochastic regimes}
We analyze our algorithm in two regimes: the \emph{adversarial} regime and the \emph{stochastic} regime. In both regimes, the transition $P$ is fixed throughout all episodes. In the \emph{adversarial} regime, the loss functions $\{\ell_t\}_{t=1}^T$ are determined arbitrarily ahead of time. In the \emph{stochastic} regime, $\ell_t$ are generated randomly, and there exists a deterministic policy $\pi^\star$, a gap function $\Delta: \calS\times \calA\rightarrow \mathbb{R}_{\geq 0}$, and $\{\lambda_t(\pi)\}_{t,\pi}\subset\mathbb{R}$ such that for any policy $\pi$ and any~$t$,  
\begin{align*}
    &\E\left[V^{\pi}(s_0;\ell_t) - V^{\pi^\star}(s_0;\ell_t)\right] = \sum_{s}\sum_{a\neq \pi^\star(s)}\mu^{\pi}(s,a)\Delta(s,a) - \lambda_t(\pi).  
\end{align*}
If $\lambda_t(\pi)\leq 0$ for all $\pi$, the condition above certifies that $\pi^\star$ is the optimal policy in episode $t$, and every time $\pi$ visits state~$s$ and chooses an action $a\neq \pi^\star(s)$, the incurred regret against $\pi^\star$ is at least $\Delta(s,a)$. The amount $[\lambda_t(\pi)]_+$ thus quantifies how much the condition above is violated. The stochastic regime captures the standard RL setting (i.e.,  $\{\ell_t\}$ are i.i.d.) with $\lambda_t(\pi)= 0$ and $\Delta(s,a)=\E\left[Q^{\pi^\star}(s,a;\ell_t) - V^{\pi^\star}(s;\ell_t)\right]$. Define $\Delta_{\min}=\min_{s}\min_{a\neq \pi^\star(s)}\Delta(s,a)$. Also, define 
$
    \calC= \left(\E\left[\sum_{t=1}^T \lambda_t(\pi_t)\right]\right)_+$ 
    and $ \calC(\pi)=\left(\sum_{t=1}^T \lambda_t(\pi)\right)_+. 
$
%Furthermore, define
%\begin{align*}
%    \picirc = \argmin_{\pi} \E\left[\sum_{t=1}^T V^{\pi}(s_0;\ell_t) \right],  
%\end{align*}
%which is the best policy in hindsight. Our goal is to bound $\Reg(\picirc)$. Notice that except for the stochastic case with $C=0$, $\picirc$ and $\pi^\star$ are different in general. 

\section{Main Results and Techniques Overview}
Our main results with Tsallis entropy and log barrier regularizers are the following (see \pref{sec: proof sketch} and \pref{app: final reg} for results with Shannon entropy): 
\begin{theorem}\label{thm: known main lemma}
    Under known transitions, \pref{alg: template} with Tsallis entropy regularizer ensures for any $\pi$ 
    \begin{align*}
        \Reg(\pi) \lesssim \sqrt{H^3SAT\ln (T)} +\poly(H,S,A)\ln(T)
    \end{align*}
    in the adversarial case, 
    and 
    \begin{align}
        &\Reg(\pi) \lesssim U+\sqrt{U\calC} + \poly(H,S,A)\ln(T) \label{eq: C bound known transition}
    \end{align}
    in the stochastic case, where
    $U=\sum_{s}\sum_{a\neq \pi^\star(s)} \frac{H^2\ln(T)}{\Delta(s,a)}$\footnote{A lower bound in \cite{xu2021fine} shows that an $S/\Delta_{\min}$ dependence is inevitable even when the transition is known. However, this lower bound only holds when there exist multiple optimal actions on $\Omega(S)$ of the states, while our gap bound is finite only when the optimal action is unique on all states. Therefore, our upper bound does not violate their lower bound. }. 
\end{theorem}
Our bounds in both regimes are similar to those of \citet{jin2021best} up to the definition of $U$ under their parameter $\gamma=\frac{1}{H}$ (tuning $\gamma$ trades their bounds between the two regimes; see their Appendix A.3). Compared with our definition of $U$, theirs involves an additional additive term $\frac{\text{poly}(H)S\ln(T)}{\Delta_{\min}}$ even under the assumption that the optimal action is unique on all states.  
\begin{theorem}\label{thm: unknown}
    Under unknown transitions, \pref{alg: template} with Tsallis entropy regularizer ensures for any $\pi$ 
    \begin{align*}
        \Reg(\pi) \lesssim \sqrt{H^4S^2AT\ln (T)\iota}  +\poly(H,S,A)\ln(T)\iota
    \end{align*}
    in the adversarial case, 
    and 
    \begin{align}
        &\Reg(\pi) \lesssim U+\sqrt{U(\calC+\calC(\pi))}  +\poly(H,S,A)\ln(T)\iota \label{eq: unknown stochastic bound}
    \end{align}
    in the stochastic case, where
    $U=\frac{H^4S^2A\ln(T)\iota}{\Delta_{\min}}$ and $\iota=\ln(SAT)$. 
\end{theorem}
In \citet{jin2021best}, for the stochastic case under unknown transition, a similar guarantee as \eqref{eq: unknown stochastic bound} is proven only for $\pi=\pi^\star$, with the case for general $\pi$ left open. We generalize their result by resolving some technical difficulties in their analysis. Overall, our bound in the stochastic regime improves that of \citet{jin2021best}, and the bound in the adversarial regime matches that of \citet{luo2021policy}. Notice that comparing \eqref{eq: unknown stochastic bound} with \eqref{eq: C bound known transition}, the bound under unknown transition involves an additional term $\calC(\pi)$. It remains open whether it can be removed. 

Finally, we provide a first-order best-of-both-world result under known transition. 

\begin{theorem}\label{thm: small loss}
    Under known transitions, \pref{alg: template} with log barrier regularizer ensures for any $\pi$ 
    \begin{align*}
        \Reg(\pi)\lesssim \sqrt{H^2SA\sum_{t=1}^T V^{\pi}(s_0;\ell_t)\ln^2(T)} + \poly(H,S,A)\ln^2(T)
    \end{align*}
    in the adversarial case, and 
    \begin{align*}
        \Reg(\pi)\lesssim U+\sqrt{U\calC} + \poly(H,S,A)\ln^2(T)
    \end{align*}
    in the stochastic case, where 
    $U = \sum_s\sum_{a\neq \pi^\star(s)}\frac{H^2\ln^2(T)}{\Delta(s,a)}$. 
\end{theorem}

In the next two subsections, we overview the challenges we faced and the techniques we used to obtain our results. 
\subsection{Exploration bonus for policy optimization}\label{sec: dilated bonus}
%A policy optimization algorithm for tabular MDPs can be viewed as running a bandit algorithm on each state $s$ individually, with $Q^{\pi_t}(s,\cdot~;\ell_t)$ being the loss vector
 %on state~$s$. The performance of such an algorithm can be analyzed through the %\emph{performance difference lemma}: 
 %\begin{align*}
 %    &\sum_{t=1}^T \left(V^{\pi_t}(s_0;\ell_t) - V^{\pi}(s_0;\ell_t)\right) \\
 %    &\leq  \sum_s\mu^\pi(s)\sum_{t=1}^T\sum_a  \left(\pi_t(a|s)-\pi(a|s)\right)Q^{\pi_t}(s,a;\ell_t), 
 %\end{align*}
%where $\sum_t\sum_a (\pi_t(a|s)-\pi(a|s))Q^{\pi_t}(s,a;\ell_t)$ is essentially the regret of the bandit algorithm on state $s$. As argued in \citet{luo2021policy}, if we naively run a bandit algorithm with unbiased estimators of $Q^{\pi_t}(s,a;\ell_t)$, the regret on state~$s$ will scale with $\frac{1}{\mu^{\pi_t}(s)}$, leading to a factor $\frac{\mu^\pi(s)}{\mu^{\pi_t}(s)}$ in the final regret bound.  %Since $Q^{\pi_t}(s,a;\ell_t)$ is not directly observable, the learner has to update the algorithm with an estimator $\hatQ_t(s,a)$ 
%of $Q^{\pi_t}(s,a;\ell_t)$. Because state~$s$ is visited with probability $\mu^{\pi_t}(s)$, the variance of $\hatQ_t(s,a)$ scales at least with $\frac{1}{\mu^{\pi_t}(s)}$, leading to a factor $\frac{\mu^\pi(s)}{\mu^{\pi_t}(s)}$ in the regret bound of   
%To address this issue, \citet{luo2021policy}  adds bonuses to the bandit algorithms. Their key lemma is summarized below. 
In the tabular case, a policy optimization algorithm can be viewed as running an individual bandit algorithm on each state. Our algorithm is built upon the policy optimization framework developed by \citet{luo2021policy}, who achieve near-optimal worst-case regret in adversarial MDPs.  Their key idea is summarized in the next lemma. 

\begin{lemma}[Lemma B.1 of \citet{luo2021policy}]\label{lem: dilated bonus thm}
    Suppose that for some $\{b_t\}_{t=1}^T$ and $\{\calP_t\}_{t=1}^T$, where each $b_t: \calS\rightarrow \mathbb{R}_{\geq 0}$ is a non-negative bonus function and each $\calP_t$ is a set of transitions, it holds that
    \begin{align}
    &\scalemath{1}{B_t(s,a)=b_t(s)+} \scalemath{1}{\left(1+\frac{1}{H}\right)\max_{\tildeP\in\calP_t}\E_{s'\sim \tildeP(\cdot|s,a), a'\sim \pi_t(\cdot|s')}\left[B_t(s',a')\right]}. \label{eq: dil bonus}
\end{align}
    Also, suppose that the following holds for a policy $\pi$ and a function $X^\pi: \calS\rightarrow \mathbb{R}$: 
    \begin{align}
        &\scalemath{1}{\E\left[\sum_{t=1}^T  \sum_{a} \left(\pi_t(a|s) -  \pi(a|s)\right) \left(Q^{\pi_t}(s,a;\ell_t) -  B_t(s,a)\right)\right]}   \nonumber \\
        &\scalemath{1}{\leq X^\pi(s) + \E\left[\sum_{t=1}^T b_t(s)  + \frac{1}{H}\sum_{t=1}^T\sum_{a}\pi_t(a|s)B_t(s,a)\right]}  \label{eq: po regret bound}
    \end{align}
    Then $\Reg(\pi)$ is upper bounded by
    \begin{align}
        &\sum_s \mu^{\pi}(s)X^\pi(s) + 3\E\left[\sum_{t=1}^T V^{\tildeP_t, \pi_t}(s_0; b_t)\right]+\E[\calF]  \label{eq: final regret dilated} 
    \end{align}
    where $\tildeP_t$ is the $\tildeP$ that attains the maximum in \eqref{eq: dil bonus}, and $\calF=HT\ind\{\exists t\in[T], P\notin\calP_t\}$.  
    %where $\tildeP_t$ is the $\tildeP$ that attains the maximum in \eqref{eq: dil bonus}. 
\end{lemma}
We refer the reader to Section 3 of \citet{luo2021policy} for intuition about \pref{lem: dilated bonus thm}. %Through \pref{lem: dilated bonus thm}, the algorithm of \citet{luo2021policy} simply runs individual bandit algorithms on every state $s$ with $Q^{\pi_t}(s,\cdot~;\ell_t)-B_t(s,\cdot)$ being the loss vector. 
%The algorithm of \citet{luo2021policy} essentially runs an individual bandit algorithm on every state~$s$, using $Q^{\pi_t}(s,\cdot~;\ell_t)-B_t(s,\cdot)$ as the loss (though $Q^{\pi_t}(s,\cdot~;\ell_t)$ is not observable so the learner needs to come up with its loss estimator). The bandit algorithm needs to satisfy the regret bound specified in \eqref{eq: po regret bound}. The choice of $b_t(s)$ can be figured out through reverse engineering: First, derive a regret bound assuming the bandit algorithm is only fed with loss estimators of $Q^{\pi_t}(s,a;\ell_t)$ (like a standard policy optimization algorithm without bonuses). Second, define $b_t(s)$ as the time-varying part of the derived regret bound. Then, use this $b_t(s)$ in \eqref{eq: dil bonus} to find the corresponding $B_t(s,a)$. Finally, re-analyze the regret of the bandit algorithm, but now with $Q^{\pi_t}(s,a;\ell_t)-B_t(s,a)$ as the loss. Usually, the additional regret due to the inclusion of $-B_t(s,a)$ can be controlled as $\frac{1}{H}\sum_{t=1}^T\sum_a \pi_t(a|s)B_t(s,a)$ by proper choices of parameters, which allows us to prove \eqref{eq: po regret bound}. 
\pref{lem: dilated bonus thm} gives a general recipe to design the exploration bonus for policy optimization algorithms. %, reducing the problem to finding a bonus function $b_t(s)$ that satisfies \eqref{eq: dil bonus} and \eqref{eq: po regret bound} and makes the final regret bound \eqref{eq: final regret dilated} small. 
Roughly speaking, the bonus function $b_t(s)$ is chosen to be the instantaneous regret of the bandit algorithm on state~$s$, which scales inversely with the probability of visiting state $s$ (i.e., $\frac{1}{\mu^{\pi_t}(s)}$). \pref{lem: dilated bonus thm} suggests that the bandit algorithm on state~$s$ should update itself using $Q^{\pi_t}(s,a;\ell_t)-B_t(s,a)$ as the loss, where $B_t(s,a)$ is the bonus. This ends up with a final regret term of $\sum_t V^{\pi_t}(s_0; b_t)$ in \eqref{eq: final regret dilated}. In contrast, if the bonus is not included, the final regret would involve $\sum_t V^{\pi}(s_0; b_t)$, which scales with the distribution mismatch coefficient $\frac{\mu^{\pi}(s)}{\mu^{\pi_t}(s)}$ that could be prohibitively large in the worst case.   %Notice that when using $\ell_t(s,a)-b_t(s)$ as the equivalent loss, we would expect the bandit algorithm on state $s$ to update with $Q^{\pi_t}(s,a;\ell_t) - B_t(s,a)$ where $B_t(s,a)= Q^{\pi_t}(s,a;b_t)$. The formula in \eqref{eq: dil bonus} matches this in a high-level sense, though still being slightly different due to technical details. The factor $(1+\frac{1}{H})$ in \eqref{eq: dil bonus} allows a more flexible bonus design, and thus  \citet{luo2021policy} call it \emph{dilated bonus}. 
%\citet{luo2021policy} use this framework to obtain a $\sqrt{T}$ regret bound, while our goal is to improve the regret to $\text{polylog}(T)$ in the stochastic case, meanwhile keeping the $\sqrt{T}$ regret in the adversarial case.  %  to achieve the best-of-both-world result, it remains non-trivial to design a refined bonus function $b_t(s)$ and carefully perform learning rate tuning, so that the final regret bound \eqref{eq: final regret dilated} can be bounded logarithmically in the stochastic world. 

The bonus function $b_t(s)$ we use is slightly different from that in \citet{luo2021policy} though. We notice that the $b_t(s)$ defined in \citet{luo2021policy} has two parts: the first part is \emph{FTRL regret overhead}, which comes from the regret bound of the FTRL algorithm under the given loss estimator, and the second part comes from the \emph{estimation error} in estimating the transition kernel. In order to apply the self-bounding technique to obtain the best-of-both-worlds result, the second term in \eqref{eq: final regret dilated} can only involve the the first part (FTRL regret overhead) but not the second part (estimation error). Therefore, we split their bonus into two: our $b_t(s)$ only includes the first part, and $c_t(s)$ only includes the second part. This allows us to use self-bounding on the second term in \eqref{eq: final regret dilated}. Our $c_t(s)$ goes to the first term in \eqref{eq: final regret dilated} instead and is handled differently from \citet{luo2021policy}. More details are given in \pref{sec: algorithm} and \pref{sec: proof sketch}. %The details will be clearer in \pref{sec: algorithm} and \pref{sec: proof sketch}.

\subsection{Adaptive learning rate tuning and bonus design}\label{sec: FTRL technique}
%Another important element of our algorithm is the adaptive tuning of the learning rates. One interesting phenomenon is that to achieve best of both worlds using policy optimization, we have to adopt state-dependent (or state-action-dependent) learning rates, which is different from the global regularization approach by \citet{jin2021best} where a single learning rate suffices. 
Our algorithm heavily relies on carefully tuning the learning rates and assigning a proper amount of bonus. These two tasks are intertwined with each other and introduce new challenges that are not seen in the global regularization approach \citep{jin2021best} or policy optimization approach that only aims at a worst-case bound \citep{luo2021policy}. Below we give a high-level overview for the challenges. 

In the FTRL analysis, a major challenge is to handle losses that are overly negative\footnote{Losses here refer not only to the loss from the environment, but also loss estimators or bonuses constructed by the algorithm. }. Typically, if the learning rate is $\eta$ and the negative loss of action $a$ has a magnitude of $R$, we need $\eta p(a)^\beta R\leq 1$ in order to keep the algorithm stable, where $p(a)$ is the probability of choosing action $a$, and $\beta\in[0,1]$ is a parameter related to the choice of the regularizer ($\frac{1}{2}$ for Tsallis entropy, $0$ for Shannon entropy, and $1$ for log barrier). In our case, there are two places we potentially encounter overly negative losses. One is when applying the standard \emph{loss-shifting} technique for best-of-both-world bounds (see \citet{jin2021best}). The loss-shifting effectively creates a negative loss in the analysis. The other overly negative loss is the bonus we use to obtain the first-order bound. 

For the first case, we develop a simple trick that only performs loss-shifting when the introduced negative loss is not too large, and further show that the extra penalty due to ``not performing loss-shifting'' is well-controlled. This is explained in \pref{sec: tsallis alg}.  For the second case, we develop an even more general technique (which can also cover the first case). This technique can be succinctly described as ``inserting virtual episodes'' when $\eta p(a)^\beta R$ is potentially too large. In virtual episodes, the losses are assumed to be all-zero (because the learner actually does not interact with the environment in these episodes) and the algorithm only updates over some bonus term. The goal of the virtual episodes is solely to tune down the learning rate $\eta$ and prevent $\eta p(a)^\beta R$ from being to large in real episodes. Similarly, we are able to show that the extra penalty due to virtual episodes is well-controlled. This is explained in \pref{sec: log barrier}.

\begin{algorithm*}[t]
\LinesNumbered
\caption{Policy Optimization} \label{alg: template}
    \textbf{Define}: $\psi_t(\pi;s)$, $b_t(s)$ are defined according to \pref{fig: psi and b choices},\  $\gamma_t\triangleq \min\left\{\frac{10^6H^4A^2}{t}, 1\right\}$.   \\
    \For{$t=1,2,\ldots$}{
        \begin{align}
            \pi_{t}(\cdot|s) = \argmin_{\pi\in\triangle(\calA)}\left\{ \sum_{\tau=1}^{t-1} \sum_a \pi(a)\left(\hatQ_\tau(s,a) - B_\tau(s,a) - C_\tau(s,a)\right) +   \psi_t(\pi;s)\right\}   \label{eq: pi_t def}
        \end{align}
        Add a virtual round if needed (only when aiming to get a first-order bound with log barrier  --- see \pref{sec: log barrier} \eqref{eq: check virtual round}). \\ 
        Execute $\pi_t$ in episode $t$, and receive $\{\ell_t(s_{t,h}, a_{t,h})\}_{h=0}^{H-1}$. \\
        \textbf{Define $\calP_t$: }
        Under known transition, define $\calP_t=\{P\}$. Under unknown transition, define $\calP_t$ by \eqref{eq: definition of transition P}.  
        %\begin{align}
        %    \calP_t &= \Bigg\{\tildeP:~~  \tildeP(\cdot|s,a)\in\Delta(\calS), ~~ \left|\tildeP(s'|s,a) - \hatP_t(s'|s,a)\right|\leq 2\sqrt{\frac{\hatP_t(s'|s,a)]\iota}{n_t(s,a)}} + \frac{14\iota}{3n_t(s,a)}, \nonumber \\
        %    &\qquad \qquad ~~ \forall (s,a,s')\in \calS_h\times \calA\times \calS_{h+1},~~\forall h\in\{0,1,\ldots, H-1\}\Bigg\}. \label{eq: definition of transition P}
        %\end{align}
        \\
        \textbf{Define $\hatQ_t$: }
        For $s\in\calS_h$, let $\ind_t(s,a)=\ind\{(s_{t,h},a_{t,h})=(s,a)\}$,   
        $L_{t,h} = \sum_{h'=h}^{H-1} \ell_t(s_{t,h'}, a_{t,h'})$, and 
        \begin{align}
            \hatQ_t(s,a) = \frac{\ind_t(s,a)L_{t,h}}{\mu_t(s)\pi_t(a|s)}, \qquad \text{where\ }  \mu_t(s) =  \overline{\mu}_t^{\pi_t}(s) + \gamma_t.  \label{eq: hatQ def} 
        \end{align}
        \textbf{Define $C_t$: }
Let $c_t(s)=\frac{\mu_t(s) - \underline{\mu}_t^{\pi_t}(s)}{\mu_t(s)}H$, and compute $C_t(s,a)$ by
         \begin{align}               
           C_t(s,a)&=\max_{\tildeP\in\calP_t}\E_{s'\sim \tildeP(\cdot|s,a), a'\sim \pi_t(\cdot|s')}\Big[c_t(s') + C_t(s',a')\Big], \label{eq: C def}  % \nonumber \\
           % C_t(s,a) &= c_t(s,a) + \max_{\tildeP\in\calP_t}\E_{s'\sim \tildeP(\cdot|s,a), a'\sim \pi_t(\cdot|s')}\left[C_t(s',a')\right].   
        \end{align} 
        
        \textbf{Define $B_t$: } Compute $B_t(s,a)$ by \eqref{eq: dil bonus}. 
        %\begin{align}
        %     & B_t(s,a) = b_t(s) +  \left(1+\frac{1}{H}\right) \max_{\tildeP\in \calP_t}\E_{s'\sim \tildeP(\cdot|s,a), a'\sim \pi_t(\cdot|s')}\left[ B_t(s',a')\right].  \label{eq: dil bonus alg} 
        %\end{align}
    }
\end{algorithm*}

\section{Algorithm}\label{sec: algorithm}
The template of our algorithm is \pref{alg: template}, in which we can plug different regularizers. The template applies to both known transition and unknown transition cases --- the only difference is in the definition of the confidence set $\calP_t$. 

The policy update \eqref{eq: pi_t def} is equivalent to running individual FTRL on each state with an adaptive learning rate. The loss estimator  $\hatQ_t(s,a)$ defined in \eqref{eq: hatQ def} is similar to that in \citet{luo2021policy}: if $(s,a)$ is visited, it is the cumulative loss starting from $(s,a)$ divided by the \emph{upper occupancy measure} \citep{jin2019learning} of $(s,a)$; otherwise it is zero. One difference is that the ``implicit exploration'' factor $\gamma_t$ added to the denominator is of order $\frac{1}{t}$ in our case, while it is of order $\frac{1}{\sqrt{t}}$ in \citet{luo2021policy}. This smaller $\gamma_t$ allows us to achieve logarithmic regret in the stochastic regime. 

There are two bonus functions $c_t(s)$ and $b_t(s)$ defined in \eqref{eq: C def} and \pref{fig: psi and b choices}, respectively. As discussed in \pref{sec: dilated bonus}, the bonus functions are defined to be the instantaneous regret of the bandit algorithm on state $s$. 
The first bonus function $c_t(s)$ comes from the bias of the loss estimator. Our choice of $c_t(s)$ is such that $\forall a, Q^{\pi_t}(s,a;\ell_t)-\E[\hatQ_t(s,a)]\leq c_t(s)$. The second bonus function $b_t(s)$ is related to the regret of the FTRL algorithm under the given loss estimator, which is regularizer dependent. We will elaborate how to choose $b_t(s)$ for different regularizers later in this section. 
%According to the bonus design principle discussed in \pref{lem: dilated bonus thm}, we would like to subtract $c_t(s)$ from the loss function when performing the policy update. Notice that $C_t(s,a)$ defined in \eqref{eq: C def} is similar to the $Q$-function over the loss function $c_t(s)$, and we end up with updating the bandit algorithm on state $s$ with an additional $-C_t(s,a)$ term.  

Finally, dynamic programming are used to obtain $C_t(s,a)$ and $B_t(s,a)$, which are trajectory sums of $c_t(s)$ and $b_t(s)$, with an $(1+\frac{1}{H})$ dilation on $B_t(s,a)$. They are then used in the policy update \eqref{eq: pi_t def}. 
In the following subsections, we discuss how we choose $b_t(s)$ and tune the learning rate for each regularizer.

\begin{figure*}[!htbp]
\caption{Definitions of $\psi_t(\pi;s)$ and $b_t(s)$ for different regularizers (to be used in \pref{alg: template}). } \label{fig: psi and b choices}

\relscale{0.86}
\begin{framed}
\textbf{Tsallis entropy}: 
    \begin{align}
         \psi_t(\pi;s) &= -\frac{2}{\eta_t(s)}\sum_a \sqrt{\pi(a)},  \\ 
        b_t(s) &=   4\left(\frac{1}{\eta_t(s)} - \frac{1}{\eta_{t-1}(s)}\right)\left(\xi_t(s) + \sqrt{A}\cdot\ind\left[\frac{\eta_t(s)}{\mu_t(s)} >\frac{1}{8H}\right]\right) + \nu_t(s),
                \label{eq: b_t Tsallis}
    \end{align}
    where
        \begin{align}
            \scalemath{0.97}{\eta_t(s) = \frac{1}{1600H^4\sqrt{A} + 4H\sqrt{\sum_{\tau=1}^t \frac{1}{\mu_\tau(s)} }}, \ \ 
            \xi_t(s) =\sum_{a}\sqrt{\pi_t(a|s)}(1-\pi_t(a|s)), \ \  \nu_t(s) = 8\eta_t(s)\sum_a \pi_t(a|s)C_t(s,a)^2.} \label{eq: Tsallis eta}%\\ 
          %\nu_t(s)&=\eta_t(s)\sum_a \pi_t(a|s)^{\frac{3}{2}}\left(C_t(s,a) - \left\langle \pi_t(\cdot|s), C_t(s,\cdot)\right\rangle\right)^2.
        \end{align}
        \vspace*{-15pt}
\end{framed}

\begin{framed}
\textbf{Shannon entropy}: 
   \begin{align}
            \psi_{t}(\pi;s) &= \sum_a \frac{1}{\eta_t(s,a)}\pi(a)\ln \pi(a), \\
            b_t(s) &=  8\sum_a\left(\frac{1}{\eta_t(s,a)} - \frac{1}{\eta_{t-1}(s,a)}\right)\left(\xi_t(s,a) +1-\frac{\min_{\tau\in[t]}\mu_\tau(s)}{\min_{\tau\in[t-1]}\mu_\tau(s)} \right) + \nu_t(s), \label{eq: def bt in shannon}
    \end{align}
    where
        \begin{align}
            \frac{1}{\eta_t(s,a)} &= \frac{1}{\eta_{t-1}(s,a)} + 4\left(\frac{H}{\mu_t(s) \sqrt{\sum_{\tau=1}^{t-1} \frac{\xi_\tau(s,a)}{\mu_\tau(s)}+ \frac{1}{\mu_t(s)}} } + \frac{H}{\sqrt{t}}\right)\sqrt{\ln T}, \quad \text{with\ } \frac{1}{\eta_0(s,a)} = 1600H^4A\sqrt{\ln T}, \label{eq: eta for Shannon}\\
            \xi_t(s,a) &= \min\{\pi_t(a|s)\ln (T), 1-\pi_t(a|s)\},  \qquad \nu_t(s) = 8\sum_a \eta_t(s,a)\pi_t(a|s)C_t(s,a)^2.
        \end{align}
         \vspace*{-15pt}
\end{framed}

\begin{framed}
\textbf{Log barrier (for first-order bound under known transition)}: 
   \begin{align}
            \psi_{t}(\pi;s) &= \sum_a \frac{1}{\eta_t(s,a)}\ln \frac{1}{\pi(a)}, \\
            b_t(s) &=  8\sum_a\left(\frac{1}{\eta_{t+1}(s,a)} - \frac{1}{\eta_{t}(s,a)}\right)\log(T) + \nu_t(s), \label{eq: b_t for log barrier}
            %b_t(s) &=  H\sum_a\left(\sqrt{\sum_{\tau=1}^t \frac{\xi_\tau(s,a)}{\mu_\tau(s)} + \max_{\tau\leq t} \frac{1}{\mu_\tau(s)}} - \sqrt{\sum_{\tau=1}^{t-1} \frac{\xi_\tau(s,a)}{\mu_\tau(s)} + \max_{\tau\leq t-1} \frac{1}{\mu_\tau(s)}}\right) + \sum_a \eta_t(s,a)\pi_t(a|s)C_t(s,a)^2,
    \end{align}
    where
        \begin{align}
            (s_t^\dagger, a_t^\dagger)&=\argmax_{s,a}\frac{\eta_t(s,a)}{\mu_t(s)} \tag{break tie arbitrarily}\\
            \frac{1}{\eta_{t+1}(s,a)} &=
            \begin{cases}
                \frac{1}{\eta_{t}(s,a)} + \frac{4\eta_{t}(s,a)\zeta_t(s,a) }{\mu_{t}(s)^2\log(T)} &\text{if $t$ is a real episode}\\ 
                \frac{1}{\eta_t(s,a)}\left(1+\frac{1}{24H\log T}\right)  &\text{if $t$ is a virtual episode and $(s_t^\dagger, a_t^\dagger)=(s,a)$} \\
                \frac{1}{\eta_t(s,a)} &\text{if $t$ is a virtual episode and $(s_t^\dagger, a_t^\dagger)\neq (s,a)$}
            \end{cases}
            \label{eq: log barrier eat def} \\
             \frac{1}{\eta_1(s,a)}&=4H^4, \\
            \zeta_t(s,a) &= \big(\ind_t(s,a)-\pi_t(a|s)\ind_t(s)\big)^2L_{t,h}^2 \qquad \text{where\ } \ind_t(s)=\sum_a \ind_t(s,a)\tag{suppose that $s\in\calS_h$}, \\ \nu_t(s) &= 8\sum_a  \eta_t(s,a)\pi_t(a|s)C_t(s,a)^2.
        \end{align}
         \vspace*{-15pt}
\end{framed}
\relscale{1.0}
\end{figure*}

\subsection{Tsallis entropy}\label{sec: tsallis alg}
$b_t(s)$ corresponds to the instantaneous regret of the bandit algorithm on state~$s$ under the given loss estimator. To obtain its form, we first analyze the regret assuming $B_t(s,a)$ is not included, i.e., only update on $\hatQ_t(s,a)-C_t(s,a)$ ($B_t(s,a)$ will be added back for analysis after the form of $b_t(s)$ is decided). %This allows us to derive a form of $b_t(s)$. After this, we will re-derive regret bound with $B_t(s,a)$ included. 
Inspired by \citet{zimmert2019optimal} for multi-armed bandits, our target is to show that the instantaneous regret (see \pref{app: FTRL analysis} for details) on state~$s$ is upper bounded by %have demonstrated that a useful idea is to balance the per-step \emph{penalty term} and the \emph{stability term} appeared in the FTRL regret bound (see \pref{app: FTRL analysis} for their definitions). Assuming that the learning rate on state~$s$ is $\eta_t(s)$ in episode~$t$, inspired by the \citet{zimmert2019optimal} on multi-armed bandits, our target is to show that the two terms can be upper bounded by 
\begin{align}
     \underbrace{\left(\frac{1}{\eta_t(s)} - \frac{1}{\eta_{t-1}(s)}\right)\xi_t(s)}_{\text{penalty term}} +\underbrace{ \frac{H^2\eta_t(s)\xi_t(s)}{\mu_t(s)}}_{\text{stability term}} + \nu_t(s)  \label{eq: two terms balancing}
\end{align}
where $\xi_t(s)=\sum_a \sqrt{\pi_t(a|s)}(1-\pi_t(a|s))\leq \sqrt{A}$, and $\nu_t(s)$ is some overhead due to the inclusion of $-C_t(s,a)$. The factor $\xi_t(s)$ allows us to use the self-bounding technique that leads to best-of-both-worlds bounds, which cannot be relaxed to $\sqrt{A}$ in general. Compared to the bound for multi-armed bandits in \cite{zimmert2019optimal}, the extra  $\frac{1}{\mu_t(s)}$ scaling in the stability term comes from importance weighting because state $s$ is visited with probability roughly $\mu_t(s)$.  
This desired bound suggests a learning rate scheduling of 
\begin{align}
    \eta_t(s)\approx \frac{1}{H\sqrt{\sum_{\tau=1}^t \frac{1}{\mu_\tau(s)} }}   \label{eq: eta t in tsallis}
\end{align}
to balance the penalty and the stability terms. This is exactly how we tune $\eta_t(s)$ in \eqref{eq: Tsallis eta}. However, to obtain the $\xi_t(s)$ factor in the stability term in \eqref{eq: two terms balancing}, we need to perform ``loss-shifting'' in the analysis, %which adds a negative loss $-\langle \pi(\cdot|s), \hatQ_t(s,\cdot)\rangle$ to every action, whose magnitude can be of order $\frac{H}{\mu_t(s)}$. This 
which necessitates the condition $ \frac{\eta_t(s)H}{\mu_t(s)}\lesssim 1$ as discussed in \pref{sec: FTRL technique}. From the choice of $\eta_t(s)$ in \eqref{eq: eta t in tsallis}, this condition may not always hold, but every time it is violated, $\eta_t(s)$ is decreased by a relatively large factor in the next episode. 

Our strategy is that whenever the condition $\frac{H\eta_t(s)}{\mu_t(s)}\lesssim 1$ is violated, we do not perform loss-shifting. This still allows us to prove a stability term of $\frac{H^2\eta_t(s)\sqrt{A}}{\mu_t(s)}$ for that episode. The key in the analysis is to show that the extra cost due to ``not performing loss-shifting'' is only logarithmic in $T$ (see the proof of \pref{lem: V(bt) known Tsallis}). Combining this idea with the instantaneous regret bound in \eqref{eq: two terms balancing} and the choice of $\eta_t(s)$ in \eqref{eq: eta t in tsallis}, we are able to derive the form of $b_t(s)$ in \eqref{eq: b_t Tsallis}.  After figuring out the form of $b_t(s)$ assuming $B_t(s,a)$ is not incorporated in the updates, we incorporate it back and re-analyze the stability term. The extra stability term due to $b_t(s)$ leads to a separate quantity $\frac{1}{H} \pi_t(a|s)B_t(s,a)$, which is an overhead allowed by~\eqref{eq: po regret bound}. 

%in the analysis, which is done by adding a common negative loss of $-\sum_a \pi_t(a|s)\hatQ_t(s,a)$ to every action on $s$. This negative loss can be of order $\frac{H}{\mu_t(s)}$, and to show the desired bound, we need $\eta_t(s)\cdot \frac{H}{\mu_t(s)}\leq 1$. 
%Notice that the stability term scales with $\frac{1}{\mu_t(s)}$ because the loss estimators on state $s$ is multiplied by the importance weight $\frac{\ind_t(s)}{\mu_t(s)}$, whose variance is of order $\frac{1}{\mu_t(s)}$. 

%

\subsection{Shannon entropy}
The design of $b_t(s)$ under Shannon entropy follows similar procedures as in \pref{sec: tsallis alg}, except that the tuning of the learning rate is inspired by \citet{ito2022nearly}. One improvement over theirs is that we adopt coordinate-dependent learning rates that can give us a refined gap-dependent bound in the stochastic regime (in multi-armed bandits, this improves their $\max_{a}\frac{A}{\Delta(a)}$ dependence to $\sum_a \frac{1}{\Delta(a)}$). With Shannon entropy, there is less learning rate tuning issue because its optimal learning rate decreases faster than other regularizers, and there is no need to perform loss-shifting \citep{ito2022nearly}.  %This coordinate-dependent version might be less easy to implement because it is not equivalent to the standard exponential weight algorithm; however, it is straightforward to adapt our analysis to the coordinate-independent version (which is then equivalent to exponential weight) with the price of not having refined gap-independent bound in the stochastic regime. 
The regret bound under Shannon entropy is overall worse than that of Tsallis entropy by a $\ln^2(T)$ factor.

\subsection{Log barrier}\label{sec: log barrier}
As shown by \citet{wei2018more, ito2021parameter}, FTRL with a log barrier regularizer is also able to achieve the best of both worlds, with the additional benefit of having \emph{data-dependent} bounds. In this subsection, we demonstrate the possibility of this by showing that under \emph{known} transition, \pref{alg: template} is able to achieve a first-order bound in the adversarial regime, while achieving $\text{polylog}(T)$ regret in the stochastic regime. %For the unknown transition case, our current technique does not lead to a first-order bound, and thus does not give benefits over the other two regularizers, so we omit the related discussions. 

To get a first-order best-of-both-world bound with log barrier, inspired by \citet{ito2021parameter}, we need to prove the following instantaneous regret for the bandit algorithm on~$s$: 
\begin{align}
    \scalemath{1}{ \underbrace{\sum_a\left(\frac{1}{\eta_t(s,a)} - \frac{1}{\eta_{t-1}(s,a)}\right)\ln T}_{\text{penalty term}} + \underbrace{\sum_a\frac{\eta_t(s,a)\zeta_t(s,a)}{\mu_t(s)^2}}_{\text{stability term}} }+\nu_t(s) \label{eq: desired log barrier bound}
\end{align}
where $\zeta_t(s,a)=\left(\ind_t(s,a)-\pi(a|s)\ind_t(s)\right)^2L_{t,h}^2$ for $s\in\calS_h$. This suggests a learning rate scheduling of 
$1/\eta_{t+1}(s,a)=1/\eta_{t}(s,a)+\eta_t(s,a)\zeta_t(s,a)/\mu_t(s)^2$. 
Similar to the Tsallis entropy case, obtaining the desired stability term in \eqref{eq: desired log barrier bound} requires loss-shifting, so we encounter the same issue as before and can resolve it in the same way. With this choice of $\eta_t(s,a)$, we can derive the desired form of $b_t(s)$ from \eqref{eq: desired log barrier bound}. 
However, the magnitude of this bonus is larger than in the Tsallis entropy case because of the $\frac{1}{\mu_t(s)^2}$ scaling here. Therefore, an additional problem arises: the $B_t(s,a)$ derived from this $b_t(s)$ can be large that makes $\eta_t(s,a)\pi_t(a|s)B_t(s,a)>\frac{1}{H}$
happen, which violates the condition under which we can bound the extra stability term (due to the inclusion of $b_t(s)$) by $\frac{1}{H}\pi_t(a|s)B_t(s,a) $. Notice that this was not an issue under Tsallis entropy. 

To resolve this, we note that $\eta_t(s,a)\pi_t(a|s)B_t(s,a)$ can be as large as $\text{poly}(H,S)\max_{s',a'}\frac{\eta_t(s',a')^2}{\mu_t(s')^2}$ (\pref{lem: eta B log barrier}), so all we need is to make $\frac{\eta_t(s,a)}{\mu_t(s)}\leq \frac{1}{\text{poly}(H,S)}$ for all $s,a$.  

Our solution is to insert \emph{virtual episodes} when $\frac{\eta_t(s,a)}{\mu_t(s)}$ is too large on some $(s,a)$. In virtual episodes, the learner does not actually interact with the environment; instead, the goal is purely to tune down $\eta_t(s,a)$. 
To decide whether to insert a virtual episode, in episode $t$, after the learner computes $\pi_t(\cdot|s)$ on all states, he checks if 
\begin{align}
    \max_{s,a}\frac{\eta_t(s,a)}{\mu_t(s)} > \frac{1}{60\sqrt{H^3S}}.    \label{eq: check virtual round}
\end{align}
If so, then episode $t$ is made a virtual episode in which the losses are assumed to be zero everywhere.\footnote{Inserting a virtual episode shifts the index of future real episodes. Since there are only $O(HSA\log^2 T)$ virtual episodes, we still use $T$ to denote the total number of episodes.} In a virtual episode, let $(s_t^\dagger, a_t^\dagger)=\argmax_{s,a}\frac{\eta_t(s,a)}{\mu_t(s)}$, and we tune down $\eta_t(s_t^\dagger, a_t^\dagger)$ by a factor of $(1+\frac{1}{24H\log T})$. Also, a bonus $b_t(s)$ is assigned to $s_t^\dagger$ to reflect the increased penalty term on state $s_t^\dagger$ due to the decrease in learning rate (by \eqref{eq: desired log barrier bound}). Combinig all the above, we get the bonus and learning rate specified in \eqref{eq: b_t for log barrier} and \eqref{eq: log barrier eat def}. Again, since every time a virtual episode happens, there exists some $\eta_t(s,a)$ decreased by a significant factor, it cannot happen too many times. 

%Again, based on the bonus design principle, we would like subtract $b_t(s)$ from the loss function, and end up with updating the bandit algorithm on state $s$ with an additional $-B_t(s,a)$ term. 

%Overall, the bandit algorithm on state $s$ updates over the loss vector $\hatQ_t(s,a) - B_t(s,a) - C_t(s,a)$, as shown in \eqref{eq: pi_t def}. 

%With a comparison with the bonus term defined in \citet{luo2021policy}, one can find that their $B_t(s,a)$ is roughly equivalent to our $B_t(s,a)+C_t(s,a)$. The reason we split the bonus into two parts is because 
%on the right-hand side of \eqref{eq: po regret bound}, we only want the $b_t(s)$ to be related to the $B_t(s,a)$ part of the bonus but not the $C_t(s,a)$ part. This allows us to have a tighter $b_t(s)$ compared to \citet{luo2021policy}. 

%For $s\in\calS_h$, let $\one_t(s,a)=\one\{(s_{t,h},a_{t,h})=(s,a)\}$,  
%        $L_{t,h} = \sum_{h'=h}^{H-1} \ell_t(s_{t,h'}, a_{t,h'})$, and 
%        \begin{align}
%            \hatQ_t(s,a) = \frac{\one_t(s,a)L_{t,h}}{\mu_t(s)\pi_t(a|s)}.  \label{eq: hatQ def} 
%        \end{align}
%Let $c_t(s,a), C_t(s,a)$ satisfy
%         \begin{align}               
%           c_t(s,a)&=H\max_{\tildeP\in\calP_t}\E_{s'\sim \tildeP(\cdot|s,a)}\left[\frac{\overline{\mu}_t^{\pi_t}(s') - \underline{\mu}_t^{\pi_t}(s') }{\mu_t(s')}\right],  \nonumber \\
%            C_t(s,a) &= c_t(s,a)  \nonumber\\
%            \quad + & \max_{\tildeP\in\calP_t}\E_{s'\sim \tildeP(\cdot|s,a), a'\sim \pi_t(\cdot|s')}\left[C_t(s',a')\right].   \label{eq: C def} 
%        \end{align}

\section{Sketch of Regret Analysis}\label{sec: proof sketch}

Our goal is to show \eqref{eq: po regret bound} and bound the right-hand side of \eqref{eq: final regret dilated} (for all regularizers and known/unknown transitions). To show \eqref{eq: po regret bound}, for a fixed $\pi$, we do the following decomposition: 
\begin{align}
    &\scalemath{1}{\sum_{t,a} \left(\pi_t(a|s) -  \pi(a|s)\right) \left(Q^{\pi_t}(s,a;\ell_t) -  B_t(s,a)\right)} \label{eq: transition decomposition} \\
    &\scalemath{1}{ = \underbrace{\sum_{t,a}  \left(\pi_t(a|s) -  \pi(a|s)\right) \left(\hatQ_t(s,a) -  B_t(s,a)-C_t(s,a)\right)}_{\regterm^\pi(s)}} \nonumber \\
    &\scalemath{1}{+ \underbrace{\sum_{t,a} \left(\pi_t(a|s) -  \pi(a|s)\right) \left(Q^{\pi_t}(s,a;\ell_t) -  \hatQ_t(s,a) + C_t(s,a)\right)}_{\biasterm^\pi(s)}}. \nonumber
\end{align}
The next lemma bounds the expectation of $\regterm^\pi(s)$. 
\begin{lemma}\label{lem: regterm unknown} 
    $\E\left[\regterm^\pi(s)\right]$ is upper bounded by 
    \begin{align*}
        \scalemath{1}{O(H^4SA\ln(T)) + 
        \E\left[\sum_{t=1}^T b_t(s) + \frac{1}{H}\sum_{t=1}^T \sum_a \pi_t(a|s)B_t(s,a)\right].} 
    \end{align*}
\end{lemma}
The proof of \pref{lem: regterm unknown} is in \pref{app: FTRL bound for alg}. Notice that depending on the regularizers and whether the transition is known/unknown, the definitions of $b_t(s)$ are different, so we prove it individually for each case. %However, they all admit the same form as shown in \pref{lem: regterm unknown}. 

Combining \eqref{eq: transition decomposition} with \pref{lem: regterm unknown}, we see that the condition in \pref{lem: dilated bonus thm} is satisfied with $X^\pi(s)=O(H^4SA\ln(T))+\E[\biasterm^\pi(s)]$. By \pref{lem: dilated bonus thm}, we can upper bound $\E[\Reg(\pi)]$ by the order of
\begin{align}
    \scalemath{1}{H^5SA\ln(T)+\scalemath{1}{\E \left[\sum_{s}\mu^{\pi}(s)\biasterm^\pi(s) + \sum_{t=1}^T V^{\tildeP_t, \pi_t}(s_0; b_t) \right].}} \label{eq: final regret tmp}
\end{align}
The next lemma bounds the bias part in \eqref{eq: final regret tmp}. See \pref{app: bias part} for the proof.  % next two subsections, we bound the two parts in \eqref{eq: final regret tmp} individually. 
% \subsection{Bounding the bias part in \eqref{eq: final regret tmp}}
 \begin{lemma}\label{lem: biasterm known}
    With known transitions, $\E \left[\sum_{s}\mu^{\pi}(s)\biasterm^\pi(s) \right]\lesssim H^5SA^2\ln(T)$, and with unknown transitions,  
    \begin{align*}
        &\E \left[\sum_{s}\mu^{\pi}(s)\biasterm^\pi(s) \right] \lesssim H^2S^4A^2\ln(T)\iota  +  \sqrt{H^3S^2A\E\left[\sum_{t=1}^T  \sum_{s,a} \left[\mu^{\pi_t}(s,a) - \mu^{\pi}(s,a)\right]_+ \right]\ln(T)\iota}. 
    \end{align*}
\end{lemma}
Next, we bound the bonus part in \eqref{eq: final regret tmp} for all regularizers we consider. The proofs are in \pref{app: b_t part}. 
\begin{lemma}\label{lem: V(bt) known Tsallis}
    Using Tsallis entropy as the regularizer, with known transitions,  
    \begin{align*}
        &\E\left[\sum_{t=1}^T V^{\tildeP_t, \pi_t}(s_0; b_t)\right]\lesssim H^4SA^2\ln(T) + H\sum_{s,a} \sqrt{\E\left[\sum_{t=1}^T \mu_t(s)\pi_t(a|s)(1-\pi_t(a|s))\right]\ln(T)}. 
    \end{align*}
    With unknown transitions, the right-hand side above further has an additional term $O(HS^4A^2\ln(T)\iota)$. 
\end{lemma}

%\begin{align*}
%    &H^5S^3A\ln(T)\sqrt{\iota}\\
%    &+\E\left[H\sqrt{\ln(T)\sum_{t=1}^T \sum_{s,a} \mu^{\pi_t}(s)\pi_t(a|s)(1-\pi_t(a|s))}\right]. 
%\end{align*}

 \begin{lemma}\label{lem: V(bt) known shannon}
    Using Shannon entropy as the regularizer, With known transitions,  
    \begin{align*}
        &\E\left[\sum_{t=1}^T V^{\tildeP_t, \pi_t}(s_0; b_t)\right]\lesssim H^4SA^2\sqrt{\ln^3(T)} + H\sum_{s,a} \sqrt{\E\left[\sum_{t=1}^T \mu_t(s)\pi_t(a|s)(1-\pi_t(a|s))\right]\ln^3(T)}. 
    \end{align*}
    With unknown transitions, the right-hand side above further has an additional term $O(HS^4A^2\ln(T)\iota)$.  
\end{lemma}

\begin{lemma}\label{lem: log barrier bt lemma}
    Using log barrier as the regularizer, with known transitions, 
    \begin{align*}
        &\E\left[\sum_{t=1}^T V^{\pi_t}(s_0;b_t)\right]\lesssim H^3S^2A^2\ln(T)\ln(SAT)  + \sum_{s,a}\sqrt{\E\left[\sum_{t=1}^T (\ind_t(s,a)-\pi_t(a|s)\ind_t(s))^2L_{t,h(s)}^2\right] \ln^2(T)}. 
    \end{align*}
    %where $h(s)$ is the layer where $s$ is in (i.e., $s\in\calS_{h(s)}$). 
\end{lemma}

\paragraph{Final regret bounds}
To obtain the final regret bounds, we combine \pref{lem: biasterm known} with each of \pref{lem: V(bt) known Tsallis}, \pref{lem: V(bt) known shannon}, and \pref{lem: log barrier bt lemma} based on \eqref{eq: final regret tmp}. Then we use the standard self-bounding technique to derive the bounds for each regime. The details are provided in \pref{app: final reg}. 

\section{Conclusion}
In this work, we develop policy optimization algorithms for tabular MDPs that achieves \emph{the best of both worlds}. Compared to previous solutions with a similar guarantee \citep{jin2020simultaneously, jin2021best}, our algorithm is computationally much simpler; compared to most existing RL algorithms, our algorithm is more robust (handling adversarial losses) and more adaptive (achieving fast rate in stochastic environments) simultaneously. Built upon the flexible policy optimization framework, our work paves a way towards developing more robust and adaptive algorithms for more general settings. Future directions include obtaining data-dependent bounds under unknown transitions, and incorporating function approximation.

\bibliography{example_paper}

\begin{thebibliography}{36}
\providecommand{\natexlab}[1]{#1}
\providecommand{\url}[1]{\texttt{#1}}
\expandafter\ifx\csname urlstyle\endcsname\relax
  \providecommand{\doi}[1]{doi: #1}\else
  \providecommand{\doi}{doi: \begingroup \urlstyle{rm}\Url}\fi

\bibitem[Abbasi-Yadkori et~al.(2019)Abbasi-Yadkori, Bartlett, Bhatia, Lazic,
  Szepesvari, and Weisz]{abbasi2019politex}
Yasin Abbasi-Yadkori, Peter Bartlett, Kush Bhatia, Nevena Lazic, Csaba
  Szepesvari, and Gell{\'e}rt Weisz.
\newblock Politex: Regret bounds for policy iteration using expert prediction.
\newblock In \emph{International Conference on Machine Learning}, pages
  3692--3702. PMLR, 2019.

\bibitem[Agarwal et~al.(2020{\natexlab{a}})Agarwal, Henaff, Kakade, and
  Sun]{agarwal2020pc}
Alekh Agarwal, Mikael Henaff, Sham Kakade, and Wen Sun.
\newblock Pc-pg: Policy cover directed exploration for provable policy gradient
  learning.
\newblock \emph{arXiv preprint arXiv:2007.08459}, 2020{\natexlab{a}}.

\bibitem[Agarwal et~al.(2020{\natexlab{b}})Agarwal, Kakade, Lee, and
  Mahajan]{agarwal2020optimality}
Alekh Agarwal, Sham~M Kakade, Jason~D Lee, and Gaurav Mahajan.
\newblock Optimality and approximation with policy gradient methods in markov
  decision processes.
\newblock In \emph{Conference on Learning Theory}, pages 64--66. PMLR,
  2020{\natexlab{b}}.

\bibitem[Amir et~al.(2022)Amir, Azov, Koren, and Livni]{amir2022better}
Idan Amir, Guy Azov, Tomer Koren, and Roi Livni.
\newblock Better best of both worlds bounds for bandits with switching costs.
\newblock \emph{arXiv preprint arXiv:2206.03098}, 2022.

\bibitem[Auer and Chiang(2016)]{auer2016algorithm}
Peter Auer and Chao-Kai Chiang.
\newblock An algorithm with nearly optimal pseudo-regret for both stochastic
  and adversarial bandits.
\newblock In \emph{Conference on Learning Theory}, pages 116--120. PMLR, 2016.

\bibitem[Bubeck and Slivkins(2012)]{bubeck2012best}
S{\'e}bastien Bubeck and Aleksandrs Slivkins.
\newblock The best of both worlds: Stochastic and adversarial bandits.
\newblock In \emph{Conference on Learning Theory}, pages 42--1. JMLR Workshop
  and Conference Proceedings, 2012.

\bibitem[Cai et~al.(2020)Cai, Yang, Jin, and Wang]{cai2020provably}
Qi~Cai, Zhuoran Yang, Chi Jin, and Zhaoran Wang.
\newblock Provably efficient exploration in policy optimization.
\newblock In \emph{International Conference on Machine Learning}, pages
  1283--1294. PMLR, 2020.

\bibitem[Chen et~al.(2021)Chen, Luo, and Wei]{chen2021impossible}
Liyu Chen, Haipeng Luo, and Chen-Yu Wei.
\newblock Impossible tuning made possible: A new expert algorithm and its
  applications.
\newblock In \emph{Conference on Learning Theory}, pages 1216--1259. PMLR,
  2021.

\bibitem[Chen et~al.(2022)Chen, Luo, and Rosenberg]{chen2022policy}
Liyu Chen, Haipeng Luo, and Aviv Rosenberg.
\newblock Policy optimization for stochastic shortest path.
\newblock \emph{arXiv preprint arXiv:2202.03334}, 2022.

\bibitem[Erez and Koren(2021)]{erez2021best}
Liad Erez and Tomer Koren.
\newblock Best-of-all-worlds bounds for online learning with feedback graphs.
\newblock \emph{arXiv preprint arXiv:2107.09572}, 2021.

\bibitem[He et~al.(2022)He, Zhou, and Gu]{he2022near}
Jiafan He, Dongruo Zhou, and Quanquan Gu.
\newblock Near-optimal policy optimization algorithms for learning adversarial
  linear mixture mdps.
\newblock In \emph{International Conference on Artificial Intelligence and
  Statistics}, pages 4259--4280. PMLR, 2022.

\bibitem[Ito(2021)]{ito2021parameter}
Shinji Ito.
\newblock Parameter-free multi-armed bandit algorithms with hybrid
  data-dependent regret bounds.
\newblock In \emph{Conference on Learning Theory}, pages 2552--2583. PMLR,
  2021.

\bibitem[Ito et~al.(2022)Ito, Tsuchiya, and Honda]{ito2022nearly}
Shinji Ito, Taira Tsuchiya, and Junya Honda.
\newblock Nearly optimal best-of-both-worlds algorithms for online learning
  with feedback graphs.
\newblock \emph{arXiv preprint arXiv:2206.00873}, 2022.

\bibitem[Jin et~al.(2020)Jin, Jin, Luo, Sra, and Yu]{jin2019learning}
Chi Jin, Tiancheng Jin, Haipeng Luo, Suvrit Sra, and Tiancheng Yu.
\newblock Learning adversarial markov decision processes with bandit feedback
  and unknown transition.
\newblock In \emph{International Conference on Machine Learning}, 2020.

\bibitem[Jin and Luo(2020)]{jin2020simultaneously}
Tiancheng Jin and Haipeng Luo.
\newblock Simultaneously learning stochastic and adversarial episodic mdps with
  known transition.
\newblock \emph{Advances in neural information processing systems},
  33:\penalty0 16557--16566, 2020.

\bibitem[Jin et~al.(2021)Jin, Huang, and Luo]{jin2021best}
Tiancheng Jin, Longbo Huang, and Haipeng Luo.
\newblock The best of both worlds: stochastic and adversarial episodic mdps
  with unknown transition.
\newblock \emph{Advances in Neural Information Processing Systems},
  34:\penalty0 20491--20502, 2021.

\bibitem[Lai et~al.(1985)Lai, Robbins, et~al.]{lai1985asymptotically}
Tze~Leung Lai, Herbert Robbins, et~al.
\newblock Asymptotically efficient adaptive allocation rules.
\newblock \emph{Advances in applied mathematics}, 6\penalty0 (1):\penalty0
  4--22, 1985.

\bibitem[Lattimore and Szepesv{\'a}ri(2018)]{lattimore2018bandit}
Tor Lattimore and Csaba Szepesv{\'a}ri.
\newblock \emph{Bandit algorithms}.
\newblock Cambridge University Press (preprint), 2018.

\bibitem[Lee et~al.(2020)Lee, Luo, Wei, and Zhang]{lee2020bias}
Chung-Wei Lee, Haipeng Luo, Chen-Yu Wei, and Mengxiao Zhang.
\newblock Bias no more: high-probability data-dependent regret bounds for
  adversarial bandits and mdps.
\newblock \emph{Advances in Neural Information Processing Systems}, 2020.

\bibitem[Levine and Koltun(2013)]{levine2013guided}
Sergey Levine and Vladlen Koltun.
\newblock Guided policy search.
\newblock In \emph{International conference on machine learning}, pages 1--9.
  PMLR, 2013.

\bibitem[Luo(2022)]{luo2022homework3}
Haipeng Luo.
\newblock Homework 3 solution, introduction to online optimization/learning.
\newblock
  \url{http://haipeng-luo.net/courses/CSCI659/2022_fall/homework/HW3_solutions.pdf},
  November 2022.

\bibitem[Luo et~al.(2021)Luo, Wei, and Lee]{luo2021policy}
Haipeng Luo, Chen-Yu Wei, and Chung-Wei Lee.
\newblock Policy optimization in adversarial mdps: Improved exploration via
  dilated bonuses.
\newblock \emph{Advances in Neural Information Processing Systems},
  34:\penalty0 22931--22942, 2021.

\bibitem[Neu and Olkhovskaya(2021)]{neu2021online}
Gergely Neu and Julia Olkhovskaya.
\newblock Online learning in mdps with linear function approximation and bandit
  feedback.
\newblock \emph{arXiv preprint arXiv:2007.01612v2}, 2021.

\bibitem[Rouyer et~al.(2021)Rouyer, Seldin, and
  Cesa-Bianchi]{rouyer2021algorithm}
Chlo{\'e} Rouyer, Yevgeny Seldin, and Nicol{\`o} Cesa-Bianchi.
\newblock An algorithm for stochastic and adversarial bandits with switching
  costs.
\newblock In \emph{International Conference on Machine Learning}, pages
  9127--9135. PMLR, 2021.

\bibitem[Schulman et~al.(2017)Schulman, Wolski, Dhariwal, Radford, and
  Klimov]{schulman2017proximal}
John Schulman, Filip Wolski, Prafulla Dhariwal, Alec Radford, and Oleg Klimov.
\newblock Proximal policy optimization algorithms.
\newblock \emph{arXiv preprint arXiv:1707.06347}, 2017.

\bibitem[Seldin and Lugosi(2017)]{seldin2017improved}
Yevgeny Seldin and G{\'a}bor Lugosi.
\newblock An improved parametrization and analysis of the exp3++ algorithm for
  stochastic and adversarial bandits.
\newblock In \emph{Conference on Learning Theory}, pages 1743--1759. PMLR,
  2017.

\bibitem[Seldin and Slivkins(2014)]{seldin2014one}
Yevgeny Seldin and Aleksandrs Slivkins.
\newblock One practical algorithm for both stochastic and adversarial bandits.
\newblock In \emph{International Conference on Machine Learning}, pages
  1287--1295. PMLR, 2014.

\bibitem[Shani et~al.(2020)Shani, Efroni, Rosenberg, and
  Mannor]{shani2020optimistic}
Lior Shani, Yonathan Efroni, Aviv Rosenberg, and Shie Mannor.
\newblock Optimistic policy optimization with bandit feedback.
\newblock In \emph{International Conference on Machine Learning}, pages
  8604--8613. PMLR, 2020.

\bibitem[Tsuchiya et~al.(2022)Tsuchiya, Ito, and Honda]{tsuchiya2022best}
Taira Tsuchiya, Shinji Ito, and Junya Honda.
\newblock Best-of-both-worlds algorithms for partial monitoring.
\newblock \emph{arXiv preprint arXiv:2207.14550}, 2022.

\bibitem[Wei and Luo(2018)]{wei2018more}
Chen-Yu Wei and Haipeng Luo.
\newblock More adaptive algorithms for adversarial bandits.
\newblock In \emph{Conference On Learning Theory}, 2018.

\bibitem[Wei et~al.(2021)Wei, Jahromi, Luo, and Jain]{wei2021learning}
Chen-Yu Wei, Mehdi~Jafarnia Jahromi, Haipeng Luo, and Rahul Jain.
\newblock Learning infinite-horizon average-reward mdps with linear function
  approximation.
\newblock In \emph{International Conference on Artificial Intelligence and
  Statistics}, pages 3007--3015. PMLR, 2021.

\bibitem[Xu et~al.(2021)Xu, Ma, and Du]{xu2021fine}
Haike Xu, Tengyu Ma, and Simon Du.
\newblock Fine-grained gap-dependent bounds for tabular mdps via adaptive
  multi-step bootstrap.
\newblock In \emph{Conference on Learning Theory}, pages 4438--4472. PMLR,
  2021.

\bibitem[Zanette et~al.(2020)Zanette, Lazaric, Kochenderfer, and
  Brunskill]{zanette2020learning}
Andrea Zanette, Alessandro Lazaric, Mykel Kochenderfer, and Emma Brunskill.
\newblock Learning near optimal policies with low inherent bellman error.
\newblock In \emph{International Conference on Machine Learning}, pages
  10978--10989. PMLR, 2020.

\bibitem[Zanette et~al.(2021)Zanette, Cheng, and
  Agarwal]{zanette2021cautiously}
Andrea Zanette, Ching-An Cheng, and Alekh Agarwal.
\newblock Cautiously optimistic policy optimization and exploration with linear
  function approximation.
\newblock \emph{arXiv preprint arXiv:2103.12923}, 2021.

\bibitem[Zimmert and Seldin(2019)]{zimmert2019optimal}
Julian Zimmert and Yevgeny Seldin.
\newblock An optimal algorithm for stochastic and adversarial bandits.
\newblock In \emph{The 22nd International Conference on Artificial Intelligence
  and Statistics}, pages 467--475. PMLR, 2019.

\bibitem[Zimmert et~al.(2019)Zimmert, Luo, and Wei]{zimmert2019beating}
Julian Zimmert, Haipeng Luo, and Chen-Yu Wei.
\newblock Beating stochastic and adversarial semi-bandits optimally and
  simultaneously.
\newblock In \emph{International Conference on Machine Learning}, pages
  7683--7692. PMLR, 2019.

\end{thebibliography}

\newpage

\appendix
\appendixpage

{
\startcontents[section]
\printcontents[section]{l}{1}{\setcounter{tocdepth}{2}}
}

%\part{Appendix} % Start the appendix part
%\parttoc % Insert the appendix TOC

%{
%  \hypersetup{linkcolor=black}
%  \tableofcontents
%}

\section{Additional Definitions}
 Define $\mu^{\tildeP,\pi}(s'|s,a)$ as the probability of visiting $s'$ conditioned on that $(s,a)$ is already visited, under transition kernel $\tildeP$ and policy $\pi$. In other words, $\mu^{\tildeP,\pi}(s'|s,a)$ is defined as
\begin{align*}
     \begin{cases}
        0 &\text{if\ } h(s')<h(s), \\
        0 &\text{if $h(s)=h(s')$, $s\neq s'$,}\\
        1 &\text{if\ } s'=s, \\
        \Pr\{s_{h(s')}=s'~|~(s_h,a_h)=(s,a)\} &\text{if\ } h(s')>h(s).
    \end{cases}
\end{align*}
Further define $\mu^{\tildeP, \pi}(s'|s)=\sum_{a} \mu^{\tildeP, \pi}(s'|s,a)\pi(a|s)$. We write $\mu^{\pi}(s'|s,a)=\mu^{P,\pi}(s'|s,a)$ and $\mu^{\pi}(s'|s)=\mu^{P,\pi}(s'|s)$ where $P$ is the true transition. 

\section{Concentration Bounds}

\begin{lemma}\label{lem: tildeP-P}
    If $P\in\calP_t$, then for all $\tildeP\in\calP_t$, 
    \begin{align*}
        \left|\tildeP(s'|s,a) - P(s'|s,a)\right| \leq \min\left\{4\sqrt{\frac{P(s'|s,a)\iota}{n_t(s,a)}} + \frac{40\iota}{3n_t(s,a)},\ \ 1\right\}. 
    \end{align*}
\end{lemma}

\begin{lemma}[Lemma D.3.7 of \citet{jin2021best}]\label{lem: pigeon hole}
    With probability at least $1-\delta$, for any $h$, 
    \begin{align*}
        \sum_{t=1}^T \sum_{(s,a)\in\calS_h\times \calA} \frac{\mu^{\pi_t}(s,a)}{n_t(s,a)} &\lesssim |\calS_h|A\ln T + \ln(1/\delta) %\\
        %\sum_{t=1}^T \sum_{(s,a)\times \calS_h\times \calA} \frac{\mu^{\pi_t}(s,a)}{\sqrt{n_t(s,a)}} &\lesssim \sqrt{|\calS_h|AT} + SA\ln T + \ln(H/\delta).  
    \end{align*}
\end{lemma}

\begin{definition}\label{def: good event}
    Define $\calE$ to be the event that $P\in\calP_t$ for all $t$ and the bound in \pref{lem: pigeon hole} holds. By \eqref{eq: definition of transition P} and \pref{lem: pigeon hole}, $\Pr\{\calE\}\geq 1-5H\delta$. 
\end{definition}

\section{Difference Lemmas}
\allowdisplaybreaks
\raggedbottom

\begin{lemma}[Performance difference]\label{lem: performance diff}
     For any policies $\pi_1$ and $\pi_2$, and any loss function $\ell: \calS\times \calA\rightarrow \mathbb{R}$, 
     \begin{align*}
         V^{\pi_1}(s_0; \ell) - V^{\pi_2}(s_0; \ell) = \sum_s \mu^{\pi_2}(s)(\pi_1(a|s)-\pi_2(a|s))Q^{\pi_1}(s,a;\ell).  
     \end{align*}
\end{lemma}

\begin{lemma}\label{lem: variant performance diff lemma}
     For any policies $\pi_1$ and $\pi_2$ and any function $L: \calS\times \calA\rightarrow \mathbb{R}$,
     \begin{align*}
         \sum_s \mu^{\pi_2}(s)(\pi_1(a|s)-\pi_2(a|s))L(s,a) = V^{\pi_1}(s_0; \ell) - V^{\pi_2}(s_0; \ell)
     \end{align*}
     where
     \begin{align*}
         \ell(s,a) \triangleq  L(s,a) - \E_{s'\sim P(\cdot|s,a), a'\sim \pi_1(\cdot|s')}[L(s',a')]. 
     \end{align*}
\end{lemma}
\begin{proof}
    This is simply a different way to write the performance difference lemma (\pref{lem: performance diff}). One only needs to verify that $Q^{\pi_1}(s,a;\ell) = L(s,a)$. This can be shown straightforwardly by backward induction from $s\in\calS_H$ to $s\in\calS_0$ and using the definition of $\ell(s,a)$. 
\end{proof} 

\begin{lemma}[Occupancy measure difference, Lemma D.3.1 of \cite{jin2021best}] \label{lem: occupancy diff}
\begin{align*}
    \mu^{P_1, \pi}(s) - \mu^{P_2, \pi}(s) &= \sum_{(u,v,w)\in\calS\times\calA\times \calS} \mu^{P_1, \pi}(u,v)\left[P_1(w|u,v) - P_2(w|u,v)\right]\mu^{P_2, \pi}(s|w) \\
    &= \sum_{(u,v,w)\in\calS\times\calA\times \calS} \mu^{P_2, \pi}(u,v)\left[P_1(w|u,v) - P_2(w|u,v)\right]\mu^{P_1, \pi}(s|w)
\end{align*}
\end{lemma}

\begin{lemma}[Generalized version of Lemma 4 in \cite{jin2019learning}]\label{lem: complicated lemma}
Suppose the high probability event $\calE$ defined in \pref{def: good event} holds. Let $\tildeP_t^s$ be a transition kernel in $\calP_t$ which may depend on $s$, and let $g_t(s)\in[0,G]$. Then
\begin{align*}
        \sum_{t=1}^T \sum_s \left| \mu^{\pi_t}(s) -  \mu^{\tildeP_t^s, \pi_t}(s)\right|g_t(s) \lesssim \sqrt{HS^2A  \ln(T)\iota \sum_{t=1}^T \sum_s \mu^{\pi_t}(s)g_t(s)^2} + HS^4AG\ln(T)\iota. 
\end{align*}
\end{lemma}

\begin{proof}
    We first show that for any $t,s$, 
    \begin{align}
    \left|\mu^{\pi}(s) - \mu^{\tildeP_t^s, \pi}(s)\right| \lesssim \sum_{(u,v,w)\times \calS\times \calA\times \calS} \mu^{\pi}(u,v)\sqrt{\frac{P(w|u,v)\iota}{n_t(u,v)}} \mu^{\pi}(s|w) + HS^2 \sum_{(u,v)\times \calS\times\calA} \frac{\mu^\pi(u,v)\iota}{n_t(u,v)}. \label{eq: tmp11}
\end{align}
Below, the summation range of $(u,w,v)$ and $(x,y,z)$ are both $\bigcup_{h=0}^{H-1}\left(\calS_h\times \calA\times\calS_{h+1}\right)$ if without specifying. 
\begin{align*}
    &\left|\mu^{\pi}(s) - \mu^{\tildeP_t^s, \pi}(s)\right| \\
    &\leq \sum_{u,v,w} \mu^{\pi}(u,v)\left|P(w|u,v) - \tildeP_t^s(w|u,v)\right|\mu^{\tildeP^s_t, \pi}(s|w) \tag{by \pref{lem: occupancy diff}}\\
    &= \sum_{u,v,w} \mu^{\pi}(u,v)\left|P(w|u,v) - \tildeP_t^s(w|u,v)\right|\mu^{ \pi}(s|w) \\
    &\qquad + \sum_{u,v,w} \mu^{\pi}(u,v)\left|P(w|u,v) - \tildeP_t^s(w|u,v)\right|\left(\mu^{\tildeP_t^s, \pi}(s|w) - \mu^{ \pi}(s|w)\right) \\
    &\leq \sum_{u,v,w} \mu^{\pi}(u,v)\left|P(w|u,v) - \tildeP_t^s(w|u,v)\right|\mu^{\pi}(s|w) \\
    &\qquad + \sum_{u,v,w} \mu^{\pi}(u,v)\left|P(w|u,v) - \tildeP_t^s(w|u,v)\right| \sum_{x,y,z} \mu^{\pi}(x,y|w)\left|\tildeP_t^s(z|x,y) - P(z|x,y)\right|\mu^{\tildeP_t^s,\pi}(s|z) \tag{by \pref{lem: occupancy diff}}\\
    &\lesssim \sum_{u,v,w} \mu^{\pi}(u,v)\left(\sqrt{\frac{P(w|u,v)\iota}{n_t(u,v)}} + \frac{\iota}{n_t(u,v)}\right)\mu^{\pi}(s|w) \\
    &\qquad + \sum_{u,v,w}\sum_{x,y,z} \mu^{\pi}(u,v)\left(\sqrt{\frac{P(w|u,v)\iota}{n_t(u,v)}} + \frac{\iota}{n_t(u,v)}\right)\mu^{\pi}(x,y|w)\min\left\{\sqrt{\frac{P(z|x,y)\iota}{n_t(x,y)}} + \frac{\iota}{n_t(x,y)},\  1\right\}   \tag{by \pref{lem: tildeP-P} and the assumption that $\calE$ holds}\\
    &\leq  \sum_{u,v,w} \mu^{\pi}(u,v)\sqrt{\frac{P(w|u,v)\iota}{n_t(u,v)}}\mu^{\pi}(s|w)\\
    &\qquad + \sum_{u,v,w} \mu^\pi(u,v)\frac{\iota}{n_t(u,v)}\mu^\pi(s|w) \tag{$=: \term_1$}\\
    &\qquad + \sum_{u,v,w}\sum_{x,y,z} \mu^{\pi}(u,v)\sqrt{\frac{P(w|u,v)\iota}{n_t(u,v)}}\mu^{\pi}(x,y|w)\sqrt{\frac{P(z|x,y)\iota}{n_t(x,y)}} \tag{$=: \term_2$}\\
    &\qquad  + \sum_{u,v,w}\sum_{x,y,z}\mu^\pi(u,v) \sqrt{\frac{P(w|u,v)\iota}{n_t(u,v)}} \mu^\pi(x,y|w)\min\left\{\frac{\iota}{n_t(x,y)}, \ 1\right\} \tag{$=: \term_3$}\\
    &\qquad + \sum_{u,v,w}\sum_{x,y,z} \mu^\pi(u,v)\frac{\iota}{n_t(u,v)}\mu^\pi(x,y|w) \tag{$=: \term_4$}
\end{align*}
We bound $\term_1$ to $\term_4$ separately as below: 
\begin{align*}
    \term_1 \leq \sum_{u,v,w} \frac{\mu^\pi(u,v)\iota}{n_t(u,v)} \leq S\sum_{u,v} \frac{\mu^\pi(u,v)\iota}{n_t(u,v)}. 
\end{align*}

\begin{align*}
    \term_2& =\sum_{u,v,w}\sum_{x,y,z} \sqrt{\frac{\mu^\pi(u,v) P(z|x,y) \mu^\pi(x,y|w) \iota}{n_t(u,v)}}\sqrt{\frac{\mu^\pi(u,v) P(w|u,v) \mu^\pi(x,y|w) \iota}{n_t(x,y)}} \\
    &\leq \sqrt{\sum_{u,v,w}\sum_{x,y,z}\frac{\mu^\pi(u,v) P(z|x,y) \mu^\pi(x,y|w) \iota}{n_t(u,v)}}\sqrt{\sum_{u,v,w}\sum_{x,y,z}\frac{\mu^\pi(u,v) P(w|u,v) \mu^\pi(x,y|w) \iota}{n_t(x,y)}} \tag{AM-GM} \\
    &\leq \sqrt{H\sum_{u,v,w}\frac{\mu^\pi(u,v)  \iota}{n_t(u,v)}}\sqrt{H\sum_{x,y,z}\frac{\mu^\pi(x,y) \iota}{n_t(x,y)}}  \\
    &\leq HS \sum_{u,v} \frac{\mu^\pi(u,v)\iota}{n_t(u,v)}.
\end{align*}
\begin{align*}
    \term_3 
    &\leq \sum_{u,v,w}\sum_{x,y,z} \mu^\pi(u,v)\left(P(w|u,v) + \frac{\iota}{n_t(u,v)}\right)\mu^\pi(x,y|w)  \min\left\{\frac{\iota}{n_t(x,y)}, \ 1\right\} \\
    &\leq \sum_{u,v,w}\sum_{x,y,z} \mu^\pi(u,v)P(w|u,v)\mu^\pi(x,y|w)  \frac{\iota}{n_t(x,y)} + \sum_{u,v,w}\sum_{x,y,z} \mu^\pi(u,v)\frac{\iota}{n_t(u,v)}\mu^\pi(x,y|w) \\
    &\leq H\sum_{x,y,z}\mu^\pi(x,y)\frac{\iota}{n_t(x,y)} + HS\sum_{u,v,w}\mu^\pi(u,v)\frac{\iota}{n_t(u,v)} \\
    &\leq HS \sum_{x,y}\frac{\mu^\pi(x,y)\iota}{n_t(x,y)} + HS^2\sum_{u,v}\frac{\mu^\pi(u,v)\iota}{n_t(u,v)}. 
\end{align*}
Similarly, 
\begin{align*}
    \term_4 \leq HS\sum_{u,v,w}\mu^\pi(u,v)\frac{\iota}{n_t(u,v)} \leq HS^2\sum_{u,v}\frac{\mu^\pi(u,v)\iota}{n_t(u,v)}. 
\end{align*}
Collecting all terms we obtain \eqref{eq: tmp11}. Thus,  
\begin{align}
        &\sum_{t=1}^T \sum_s \left| \mu^{\pi_t}(s) -  \mu^{\tildeP_t^s, \pi_t }(s)\right|g_t(s)  \nonumber\\
        &\leq \sum_{t=1}^T \sum_s \left[ \sum_{u,v,w} \mu^{\pi_t}(u,v)\sqrt{\frac{P(w|u,v)\iota}{n_t(u,v)}} \mu^{\pi_t}(s|w) + HS^2 \sum_{u,v} \frac{\mu^{\pi_t}(u,v)\iota}{n_t(u,v)}\right] g_t(s)  \nonumber \\
        &\leq \underbrace{\sum_{t=1}^T \sum_s \left[ \sum_{u,v,w} \mu^{\pi_t}(u,v)\sqrt{\frac{P(w|u,v)\iota}{n_t(u,v)}}\mu^{\pi_t}(s|w)\right] g_t(s)}_{(\star)} + HS^3G \sum_{t=1}^T \sum_{u,v} \frac{\mu^{\pi_t}(u,v)\iota}{n_t(u,v)} \label{eq: tmptmptmpt}
    \end{align}
    Fix an $h$, we consider the summation $(\star)$ restricted to $(u,v,w)\in\calT_h\triangleq \calS_h\times \calA\times\calS_{h+1}$. That is, 
    \begin{align*}
        &\sum_{t=1}^T \sum_s \left[ \sum_{(u,v,w)\in\calT_h} \mu^{\pi_t}(u,v)\sqrt{\frac{P(w|u,v)\iota}{n_t(u,v)}}\mu^{\pi_t}(s|w)\right] g_t(s) \\
        &\leq \sum_{t=1}^T \sum_s \left[ \sum_{(u,v,w)\in\calT_h} \mu^{\pi_t}(u,v)\left(\alpha P(w|u,v) g_t(s)^2 + \frac{\iota}{\alpha n_t(u,v)}\right)\mu^{\pi_t}(s|w)\right]  \tag{holds for any $\alpha>0$ by AM-GM}\\
        &\leq \alpha \sum_{t=1}^T\sum_s\sum_{(u,v,w)\in\calT_h} \mu^{\pi_t}(u,v)P(w|u,v) \mu^{\pi_t}(s|w)g_t(s)^2 + \frac{1}{\alpha}\sum_{t=1}^T \sum_s \sum_{(u,v,w)\in\calT_h} \frac{\mu^{\pi_t}(u,v)\iota}{n_t(u,v)}\mu^{\pi_t}(s|w)  \\
        &\leq \alpha \sum_{t=1}^T\sum_s  \mu^{\pi_t}(s)g_t(s)^2 + \frac{H|\calS_{h+1}|}{\alpha}\sum_{t=1}^T  \sum_{u,v} \frac{\mu^{\pi_t}(u,v)\iota}{n_t(u,v)}  \\
        &\lesssim \alpha \sum_{t=1}^T\sum_s  \mu^{\pi_t}(s)g_t(s)^2 + \frac{H|\calS_{h+1}||\calS_h|A \ln(T)\iota}{\alpha} + \frac{H|\calS_{h+1}|\ln(1/\delta)\iota}{\alpha}  \tag{by \pref{lem: pigeon hole} and the assumption that $\calE$ holds}\\
        &= \sqrt{H |\calS_h||\calS_{h+1}|A  \ln(T)\iota \sum_{t=1}^T \sum_s \mu^{\pi_t}(s)g_t(s)^2}   \tag{picking the optimal $\alpha$ and using our choice of $\delta=\frac{1}{T^3}$} \\
        &\leq (|\calS_h|+|\calS_{h+1})\sqrt{H A  \ln(T)\iota \sum_{t=1}^T \sum_s \mu^{\pi_t}(s)g_t(s)^2}. 
    \end{align*}
    Continue from \eqref{eq: tmptmptmpt}: 
    \begin{align*}
        &\sum_{t=1}^T \sum_s \left| \mu^{\pi_t}(s) -  \mu^{\tildeP_t^s, \pi_t }(s)\right|g_t(s) \\ &\lesssim  \sum_h (|\calS_h|+|\calS_{h+1})\sqrt{H A  \ln(T)\iota \sum_{t=1}^T \sum_s \mu^{\pi_t}(s)g_t(s)^2} + HS^4 AG\ln(T)\iota \tag{by \pref{lem: pigeon hole} and the assumption that $\calE$ holds} \\
        &\lesssim S\sqrt{H A  \ln(T)\iota \sum_{t=1}^T \sum_s \mu^{\pi_t}(s)g_t(s)^2} + HS^4 AG\ln(T)\iota. 
    \end{align*}

    %\begin{align*}
    %    &\leq \sum_{t=1}^T \sum_s \left[ \sum_{u,v,w} \mu^{\pi_t}(u,v)\left(\alpha P(w|u,v) g_t(s)^2 + \frac{\iota}{\alpha n_t(u,v)}\right)\mu^{\pi_t}(s|w)\right]  + HS^3G \sum_{t=1}^T \sum_{u,v} \frac{\mu^{\pi_t}(u,v)\iota}{n_t(u,v)} \tag{holds for any $\alpha>0$ by AM-GM}\\
    %    &\leq \alpha \sum_{t=1}^T\sum_s\sum_{u,v,w} \mu^{\pi_t}(u,v)P(w|u,v) \mu^{\pi_t}(s|w)g_t(s)^2 + \frac{1}{\alpha}\sum_{t=1}^T \sum_s \sum_{u,v,w} \frac{\mu^{\pi_t}(u,v)\iota}{n_t(u,v)}\mu^{\pi_t}(s|w) + HS^3G\sum_{t=1}^T \sum_{u,v} \frac{\mu^{\pi_t}(u,v)\iota}{n_t(u,v)} \\
    %    &\leq \alpha H\sum_{t=1}^T\sum_s  \mu^{\pi_t}(s)g_t(s)^2 + \frac{HS}{\alpha}\sum_{t=1}^T  \sum_{u,v} \frac{\mu^{\pi_t}(u,v)\iota}{n_t(u,v)} + HS^3G\sum_{t=1}^T \sum_{u,v} \frac{\mu^{\pi_t}(u,v)\iota}{n_t(u,v)} \\
    %    &\leq \alpha H\sum_{t=1}^T\sum_s  \mu^{\pi_t}(s)g_t(s)^2 + \frac{HS^2A \ln(T)\iota}{\alpha} + \frac{HS\ln(H/\delta)\iota}{\alpha} +  HS^4AG\ln(T)\iota + HS^3G\ln(H/\delta)\iota   \tag{by \pref{lem: pigeon hole}}\\
    %    &= \sqrt{H^2S^2A  \ln(T)\iota \sum_{t=1}^T \sum_s \mu^{\pi_t}(s)g_t(s)^2} + HS^4AG\ln(T)\iota. \tag{using $\ln(H/\delta)\lesssim \ln(ST)\lesssim S\ln (T)$ and choosing the optimal $\alpha$}
    %\end{align*}
\end{proof}

\begin{lemma}\label{lem: policy difference absolute}
    For any $\pi_1, \pi_2$, 
    \begin{align*}
        \sum_{s,a} \left| \mu^{\pi_1}(s,a) - \mu^{\pi_2}(s,a) \right| &\leq H\sum_{s,a}\mu^{\pi_1}(s)\left|\pi_1(a|s) - \pi_2(a|s)\right|
    \end{align*}
\end{lemma}
\begin{proof}
   For any $s,a$, we can view $\mu^{\pi}(s,a)$ as $V^{\pi}(s_0;\one_{s,a})$ where $\one_{s,a}$ is the loss function that takes the value of $1$ on $(s,a)$ and $0$ on other state-actions. By the performance difference lemma (\pref{lem: performance diff}), 
   \begin{align*}
       \left| \mu^{\pi_1}(s,a) - \mu^{\pi_2}(s,a) \right| \leq  \sum_{s',a'}\mu^{\pi_1}(s')\left| \pi_1(a'|s') - \pi_2(a'|s') \right| Q^{\pi_2}(s',a'; \one_{s,a}). % \leq 2\sum_{s'}\sum_{a'\neq \pi^\star(s')} \mu^{\pi_t}(s') \pi_t(a'|s')
   \end{align*}
   Therefore, 
   \begin{align*}
       \sum_{s,a} \left| \mu^{\pi_1}(s,a) - \mu^{\pi_2}(s,a) \right| 
       &\leq \sum_{s',a'}\mu^{\pi_1}(s')\left| \pi_1(a'|s') - \pi_2(a'|s') \right| \sum_{s,a}Q^{\pi_2}(s',a'; \one_{s,a}) \\
       &= \sum_{s',a'}\mu^{\pi_1}(s')\left| \pi_1(a'|s') - \pi_2(a'|s') \right| Q^{\pi_2}(s',a'; \one) \tag{$\one$ is the loss function that takes a constant value $1$}\\
       &\leq H  \sum_{s',a'}\mu^{\pi_1}(s')\left| \pi_1(a'|s') - \pi_2(a'|s') \right|. 
   \end{align*}
\end{proof}

\section{FTRL Regret Bounds}\label{app: FTRL analysis}
The lemmas in this section are standard results for FTRL, which can be found in e.g. \citet{lattimore2018bandit, zimmert2019optimal, ito2021parameter, luo2022homework3}. We list the results here for completeness. 

\begin{lemma}\label{lem: FTRL}
    The FTRL algorithm: 
    \begin{align*}
    p_t = \argmin_{p\in\Omega}\left\{ \left\langle p, \sum_{\tau=1}^{t-1}\ell_\tau\right\rangle + \psi_t(p) \right\} 
\end{align*}
guarantees the following:
\begin{align*}
    \sum_{t=1}^{T}\langle p_t-u, \ell_t \rangle  
    &\leq \underbrace{\psi_0(u)- \min_{p\in\Omega} \psi_0(p) + \sum_{t=1}^{T} \left(\psi_t(u)-\psi_t(p_t)-\psi_{t-1}(u) + \psi_{t-1}(p_t)\right)}_{\text{penalty term}} \\
    &\qquad + \underbrace{\sum_{t=1}^{T}\max_{p\in\Omega}\left(\langle p_t-p, \ell_t\rangle -  D_{\psi_t}(p, p_t)\right)}_{\text{stability term}}. 
\end{align*} 
\end{lemma}

\begin{proof}
Let $L_t\triangleq \sum_{\tau=1}^t \ell_\tau$. 
Define $F_t(p) = \left\langle p, L_{t-1}\right\rangle + \psi_t(p)$ and $G_t(p)=\left\langle p, L_t\right\rangle + \psi_t(p)$. Therefore, $p_t$ is the minimizer of $F_t$. Let $p_{t+1}'$ be minimizer of $G_t$. Then by the first-order optimality condition, we have
\begin{align}
    F_t(p_t) - G_t(p_{t+1}') &\leq F_t(p_{t+1}') - G_t(p_{t+1}') - D_{\psi_t}(p_{t+1}', p_t)
    = -\langle p_{t+1}', \ell_t \rangle - D_{\psi_t}(p_{t+1}', p_t).     \label{eq: eq1}
\end{align}
By definition, we also have
\begin{align}
    G_t(p_{t+1}') - F_{t+1}(p_{t+1}) &\leq G_t(p_{t+1}) - F_{t+1}(p_{t+1}) = \psi_t(p_{t+1}) - \psi_{t+1}(p_{t+1}).  \label{eq: eq2}
\end{align}
Thus, 
\begin{align*}
    &\sum_{t=1}^{T} \langle p_t, \ell_t\rangle \\ &\leq \sum_{t=1}^{T} \left(\langle p_t - p_{t+1}', \ell_t\rangle - D_{\psi_t}(p_{t+1}', p_t) + G_t(p_{t+1}') - F_t(p_t) \right)  \tag{by \eqref{eq: eq1}}  \\
    &= \sum_{t=1}^{T} \left(\langle p_t - p_{t+1}', \ell_t\rangle - D_{\psi_t}(p_{t+1}', p_t) + G_{t-1}(p_{t}') - F_t(p_t) \right) + G_{T}(p_{T+1}') - G_0(p_1') \\
    &\leq \sum_{t=1}^{T} \left(\max_{p}\Big\{\langle p_t - p, \ell_t\rangle - D_{\psi_t}(p, p_t)\Big\} - \psi_t(p_t) + \psi_{t-1}(p_t) \right) + G_{T}(u) - \min_p \psi_0(p)   \tag{by \eqref{eq: eq2}, using that $p'_{T+1}$ is the minimizer of $G_{T}$} \\
    &= \sum_{t=1}^{T} \left(\max_{p}\Big\{\langle p_t - p, \ell_t\rangle - D_{\psi_t}(p, p_t)\Big\} - \psi_t(p_t) + \psi_{t-1}(p_t) \right) + \sum_{t=1}^{T}\langle u, \ell_t\rangle + \psi_T(u)- \min_p \psi_0(p) \\
    &= \sum_{t=1}^{T} \left(\max_{p}\Big\{\langle p_t - p, \ell_t\rangle - D_{\psi_t}(p, p_t)\Big\} + \psi_t(u) - \psi_t(p_t) - \psi_{t-1}(u) + \psi_{t-1}(p_t) \right) + \sum_{t=1}^{T}\langle u, \ell_t\rangle + \psi_0(u) - \min_p \psi_0(p).  
\end{align*} 
Re-arranging finishes the proof. 
\end{proof}

\begin{lemma}[Stability under Tsallis entropy]\label{lem: stabiity hw}
    Let $\psi_t(p)=-\frac{2}{\eta_t}\sum_a \sqrt{p(a)}$, and let $\ell_t\in\mathbb{R}^A$ be such that $\eta_t\sqrt{p(a)}\ell_t(a)\geq -\frac{1}{2}$. Then 
    \begin{align*}
        \max_{p\in\triangle(\calA)}\left\{\langle p_t-p,\ell_t\rangle -  D_{\psi_t}(p,p_t)\right\}\leq 2\eta_t\sum_a p_t(a)^{\frac{3}{2}}\ell_t(a)^2. 
    \end{align*}
\end{lemma}
\begin{proof}
    The proof can be found in the Problem 1 of \citet{luo2022homework3}. 
\end{proof}

\begin{lemma}[Stability under Shannon entropy]\label{lem: stabiity hw 3}
    Let $\psi_t(p)=\sum_a \frac{1}{\eta_t(a)} p(a)\ln p(a)$, and let $ \ell_t\in\mathbb{R}^A$ be such that $\eta(a) \ell_t(a)\geq -1$. Then 
    \begin{align*}
        \max_{p\in\triangle(\calA)}\left\{\langle p_t-p,\ell_t\rangle -  D_{\psi_t}(p,p_t)\right\}\leq \sum_a\eta_t(a) p_t(a)\ell_t(a)^2. 
    \end{align*}
\end{lemma}
\begin{proof}
    The proof can be found in the Proof of Lemma 1 in \citet{chen2021impossible}. 
\end{proof}

\begin{lemma}[Stability under log barrier]\label{lem: stabiity hw 2}
    Let $\psi_t(p)=\sum_a \frac{1}{\eta_t(a)}\ln\frac{1}{p(a)}$, and let $ \ell_t\in\mathbb{R}^A$ be such that $\eta_t(a) p(a)\ell_t(a)\geq -\frac{1}{2}$. Then 
    \begin{align*}
        \max_{p\in\triangle(\calA)}\left\{\langle p_t-p,\ell_t\rangle -  D_{\psi_t}(p,p_t)\right\}\leq \sum_a\eta_t(a) p_t(a)^{2}\ell_t(a)^2. 
    \end{align*}
\end{lemma}

\begin{proof}
    \begin{align*}
        &\max_{p\in\triangle(\calA)}\left\{ \langle p_t - p,\ell_t\rangle - D_{\psi_t}(p,p_t)\right\} 
        \leq \max_{q\in\mathbb{R}^A_+}\left\{ \langle p_t - q,\ell_t\rangle - D_{\psi_t}(q,p_t)\right\}
    \end{align*}
    Define $f(q)=\langle p_t - q,\ell_t\rangle - D_{\psi_t}(q,p_t)$. Let $q^\star$ be the solution in the last expression. Next, we verify that under the specified conditions, we have $\nabla f(q^\star)=0$. It suffices to show that there exists $q\in\mathbb{R}^A_+$ such that $\nabla f(q)=0$ since if such $q$ exists, then it must the maximizer of $f$ and thus $q^\star=q$. 
    
%    \paragraph{Tsallis entropy} 
%    \begin{align*}
%        [\nabla f(q)]_a = -\ell_t(a) - [\nabla \psi(q)]_a + [\nabla \psi(p_t)]_a = -\ell_t(a) + \frac{1}{\eta(a)\sqrt{q(a)}} - \frac{1}{\eta(a)\sqrt{p_t(a)}}
%    \end{align*}
%    By the condition, we have $-\frac{1}{\eta(a)\sqrt{p_t(a)}}-\ell_t(a) < 0$ for all $a$. and so $\nabla f(q)=\mathbf{0}$ has solution in $\mathbb{R}_+$, which is $q(a)=\left(\frac{1}{\sqrt{p_t(a)}} + \eta(a) \ell_t(a)\right)^{-\frac{1}{2}}$.   
 %   \paragraph{log barrier} 
    \begin{align*}
        [\nabla f(q)]_a = -\ell_t(a) - [\nabla \psi_t(q)]_a + [\nabla \psi_t(p_t)]_a = -\ell_t(a) + \frac{1}{\eta_t(a)q(a)} - \frac{1}{\eta_t(a)p_t(a)}
    \end{align*}
    By the condition, we have $-\frac{1}{\eta_t(a)p_t(a)}-\ell_t(a) < 0$ for all $a$. and so $\nabla f(q)=\mathbf{0}$ has solution in $\mathbb{R}_+$, which is $q(a)=\left(\frac{1}{p_t(a)}+\eta_t(a)\ell_t(a)\right)^{-1}$. 
    
 %   \paragraph{Shannon entropy} 
 %   \begin{align*}
 %       [\nabla f(q)]_a = -\ell_t(a) - [\nabla \psi(q)]_a + [\nabla \psi(p_t)]_a = -\ell(a) - \frac{1}{\eta(a)}\ln q(a) + \frac{1}{\eta(a)}\ln p_t(a)
 %   \end{align*}
 %   $\nabla f(q)=\mathbf{0}$ can always be achieved by $q(a)=p_t(a)e^{-\eta_t(a)\ell_t(a)}$.  
    
    Therefore, $\nabla f(q^\star)= -\ell_t - \nabla\psi_t(q^\star) + \nabla \psi_t(p_t)=0$, and we have 
    \begin{align*}
        \max_{q\in\mathbb{R}^A_+}\left\{ \langle p_t - q,\ell_t\rangle - D_{\psi_t}(q,p_t)\right\} = \langle p_t - q^\star, \nabla \psi_t(p_t) - \nabla\psi_t(q^\star)\rangle - D_{\psi_t}(q^\star,p_t) = D_{\psi_t}(p_t, q^\star). 
    \end{align*}
    It remains to bound $D_{\psi_t}(p_t, q^\star)$, which by definition can be written as 
    \begin{align*}
        D_{\psi_t}(p_t, q^\star) = \sum_{a} \frac{1}{\eta_t(a)}h\left(\frac{p_t(a)}{q^\star(a)}\right)
    \end{align*}
    where $h(x)=x-1-\ln(x)$. By the relation between $q^\star(a)$ and $p_t(a)$ we just derived, it holds that $\frac{p_t(a)}{q^\star(a)}=1+\eta_t(a)p_t(a)\ell_t(a)$. By the fact that $\ln(1+x)\geq x-x^2$ for all $x\geq -\frac{1}{2}$, we have 
    \begin{align*}
        h\left(\frac{p_t(a)}{q^\star(a)}\right)=\eta_t(a)p_t(a)\ell_t(a) - \ln(1+\eta_t(a)p_t(a)\ell_t(a)) \leq \eta_t(a)^2p_t(a)^2\ell_t(a)^2
    \end{align*}
    which gives the desired bound. 
%    It remains to bound $D_{\psi}(p_t,q^\star)$ in all cases. 
    
%    \paragraph{Tsallis entropy}

    %By the optimality of $q^\star$, we have $\ell_t = \nabla \psi_t(p_t) - \nabla \psi_t(q^\star)$. 
    
\end{proof}

\begin{lemma}[FTRL with Tsallis entropy]\label{lem: Tsallis-INF regret}
    Let $\psi_t(p)=-\frac{2}{\eta_t}\sum_{a} \sqrt{p(a)}$ for non-increasing $\eta_t$, and let $x_t$ be such that $\eta_t\sqrt{p_t(a)}(\ell_t(a)+x_t)\geq -\frac{1}{2}$ for all $t,a$. Then the FTRL algorithm in \pref{lem: FTRL} ensures for any $u\in\triangle(\calA)$, 
    \begin{align*}
        \sum_{t=1}^T \langle p_t-u, \ell_t \rangle \leq \frac{2\sqrt{A}}{\eta_0} + 2\sum_{t=1}^T \left(\frac{1}{\eta_t} - \frac{1}{\eta_{t-1}}\right)\xi_t + 2\sum_{t=1}^T \eta_t \sum_a p_t(a)^{\frac{3}{2}}\left(\ell_t(a)+x_t\right)^2,  
    \end{align*}
    where $\xi_t=\sum_a \sqrt{p_t(a)}(1-p_t(a))$. 
\end{lemma}
\begin{proof}
    We use \pref{lem: FTRL}, and bound the penalty term and stability individually. 
    \begin{align*}
        \text{penalty term} 
        &=  \frac{2}{\eta_0} \max_{p\in\triangle(\calA)} \sum_a \left(\sqrt{p(a)}- \sqrt{u(a)}\right)  + 2\sum_{t=1}^T \left(\frac{1}{\eta_t} - \frac{1}{\eta_{t-1}}\right)\sum_{a}\left(\sqrt{p_t(a)} - \sqrt{u(a)}\right)  \\
        &\leq \frac{2\sqrt{A}}{\eta_0} + 2\sum_{t=1}^T \left(\frac{1}{\eta_t} - \frac{1}{\eta_{t-1}}\right)\left( \sum_{a}\sqrt{p_t(a)} - 1\right)
        \\
        &\leq \frac{2\sqrt{A}}{\eta_0} + 2\sum_{t=1}^T  \left(\frac{1}{\eta_t} - \frac{1}{\eta_{t-1}}\right)\xi_t. 
    \end{align*}
    Bounding the stability term: 
    \begin{align*}
        \text{stability term} 
        &=  \sum_{t=1}^T \max_{p\in\triangle(\calA)}\Big\{\langle p_t - p, \ell_t + x_t\one\rangle - D_{\psi_t}(p,p_t) \Big\} \leq 2\sum_{t=1}^T  \eta_t\sum_a p_t(a)^{\frac{3}{2}}\left(\ell_t(a)+ x_t\right)^2 
    \end{align*}
    where the first equality is because $\langle p_t - p, \one\rangle=0$ for $p_t, p\in\triangle(\calA)$, and the last inequality is by \pref{lem: stabiity hw}. 
\end{proof}

\begin{lemma}[FTRL with Shannon entropy]\label{lem: Shannon-INF regret} Let $\psi_t(p)=\sum_a \frac{1}{\eta_t(a)} p(a)\ln p(a)$, for non-increasing $\eta_t(a)$ such that $\eta_0(a)=\eta_0$ for all $a$. Assume that $\eta_t(a)\ell_t(a)\geq -1$ for all $t,a$, and assume $A\leq T$. Then for any $u\in\triangle(\calA)$, 
\begin{align*}
    \sum_{t=1}^T \langle p_t-u, \ell_t\rangle \leq \frac{\ln A}{\eta_0} +  6\sum_{t=1}^T\sum_{a} \left(\frac{1}{\eta_t(a)} - \frac{1}{\eta_{t-1}(a)}\right)\xi_t(a) + \sum_{t=1}^T \sum_a \eta_t(a)p_t(a)\ell_t(a)^2 + \frac{1}{T^2}\sum_{t=1}^T \left\langle -u + \frac{1}{A}\one, \ell_t\right\rangle. 
\end{align*}
where $\xi_t(a)=\min\left\{p_t(a)\ln (T), 1-p_t(a)\right\}$. 
\end{lemma}
\begin{proof}
    Let $u'=\left(1-\frac{1}{T^2}\right)u + \frac{1}{AT^2}\one$. We use \pref{lem: FTRL}, and bound the penalty term and stability individually (with respect to $u'$). 
    \begin{align*}
        &\text{penalty term} \\ &=\frac{1}{\eta_0} \max_p \sum_a \left(p(a)\ln\frac{1}{p(a)} - u'(a)\ln\frac{1}{u'(a)}\right) + \sum_{t=1}^T \sum_a \left(\frac{1}{\eta_t(a)} - \frac{1}{\eta_{t-1}(a)}\right)\left(p_t(a)\ln\frac{1}{p_t(a)} - u'(a)\ln\frac{1}{u'(a)}\right)\\
        &\leq \frac{\ln A}{\eta_0} + \sum_{t=1}^T \sum_a \left(\frac{1}{\eta_t(a)} - \frac{1}{\eta_{t-1}(a)}\right)\left(p_t(a)\ln\frac{1}{p_t(a)} - u'(a)\ln\frac{1}{u'(a)}\right). 
    \end{align*}
    To bound $p_t(a)\ln\frac{1}{p_t(a)} - u'(a)\ln\frac{1}{u'(a)}$, first observe that $p_t(a)\ln\frac{1}{p_t(a)}=p_t(a)\ln\left(1+\frac{1-p_t(a)}{p_t(a)}\right)\leq p_t(a)\cdot\frac{1-p_t(a)}{p_t(a)}\leq 1-p_t(a)$ because $\ln(1+x)\leq x$. By the definition of $u'$, we have 
    \begin{align*}
        u'(a)\ln\frac{1}{u'(a)} \geq \min\left\{ \frac{1}{AT^2}\ln(AT^2), \left(1-\frac{1}{T^2}\right)\ln\frac{1}{1-\frac{1}{T^2}} \right\}\geq \min\left\{\frac{1}{AT^2}, \left(1-\frac{1}{T^2}\right)\frac{1}{T^2}\right\} = \frac{1}{AT^2}. 
    \end{align*}
    If $p_t(a)\leq \frac{1}{A^2 T^4}$, then 
    \begin{align*}
        p_t(a)\ln\frac{1}{p_t(a)}-u'(a)\ln\frac{1}{u'(a)} \leq \frac{1}{A^2 T^4} \ln(A^2 T^4) - \frac{1}{AT^2}=\frac{2\ln(AT^2)-AT^2}{A^2 T^4}\leq 0
    \end{align*}
    where the first inequality is because $x\ln(x)$ is increasing for $x\leq e^{-1}$, and last inequality is because $2\ln(x)-x<0$ for all $x\in\mathbb{R}$. 
    If $p_t(a)>\frac{1}{A^2T^4}$, then $p_t(a)\ln\frac{1}{p_t(a)}\leq p_t(a)\ln(A^2T^4)\leq 6p_t(a)\ln(T)$ by the assumption $A\leq T$. Combining all arguments above, we get 
    \begin{align*}
        \text{penalty term} \leq \frac{\ln A}{\eta_0} + 6\sum_{t=1}^T \sum_a \left(\frac{1}{\eta_t(a)} - \frac{1}{\eta_t(a)}\right) \min\left\{1-p_t(a), p_t(a)\ln(T)\right\}. 
    \end{align*}
    Bounding the stability term: 
    \begin{align*}
        \text{stability term} 
        &=  \sum_{t=1}^T \max_{p\in\triangle(\calA)}\Big\{\langle p_t - p, \ell_t \rangle - D_{\psi_t}(p,p_t) \Big\} \leq \sum_{t=1}^T  \sum_a \eta_t(a) p_t(a)\ell_t(a)^2 
    \end{align*}
    where the last inequality is by \pref{lem: stabiity hw 3}. 
    
    Therefore, 
    \begin{align*}
        \sum_{t=1}^T \langle p_t-u', \ell_t\rangle \leq \frac{\ln A}{\eta_0} +  6\sum_{t=1}^T\sum_{a} \left(\frac{1}{\eta_t(a)} - \frac{1}{\eta_{t-1}(a)}\right)\xi_t(a) + \sum_{t=1}^T \sum_a \eta_t(a)p_t(a)\ell_t(a)^2
    \end{align*}
    Then noticing that 
    \begin{align*}
        \sum_{t=1}^T \langle p_t-u, \ell_t\rangle  
        &= \sum_{t=1}^T \langle p_t-u', \ell_t\rangle + \sum_{t=1}^T \langle u'-u, \ell_t\rangle \\
        &= \sum_{t=1}^T \langle p_t-u', \ell_t\rangle + \frac{1}{T^2}\sum_{t=1}^T \left\langle -u + \frac{1}{A}\one, \ell_t\right\rangle
    \end{align*}
    finishes the proof. 
\end{proof}

\begin{lemma}[FTRL with log barrier]\label{lem: log barrier regret}  Let $\psi_t(p)=\sum_a \frac{1}{\eta_t(a)}\ln\frac{1}{p(a)}$ for non-increasing $\eta_t(a)$ with $\eta_0(a)=\eta_0$ for all $a$, and let $x_t$ be such that $\eta_t(a) p_t(a)(\ell_t(a)+x_t)\geq -\frac{1}{2}$ for all $t,a$. Then for any $u\in\triangle(\calA)$, 
\begin{align*}
    \sum_{t=1}^T \langle p_t-u, \ell_t\rangle \leq \frac{3A\ln T}{\eta_0} +  4\sum_{a} \left(\frac{1}{\eta_t(a)} - \frac{1}{\eta_{t-1}(a)}\right)\ln(T) + \sum_{t=1}^T \sum_a \eta_t(a)p_t(a)\ell_t(a)^2 + \frac{1}{T^3}\sum_{t=1}^T \left\langle -u + \frac{1}{A}\one, \ell_t\right\rangle. 
\end{align*}
\end{lemma}

\begin{proof}
    
    Let $u'=\left(1-\frac{1}{T^3}\right)u + \frac{1}{AT^3}\one$. We use \pref{lem: FTRL}, and bound the penalty term and stability individually (with respect to $u'$). 
    
    \begin{align*}
        \text{penalty term} &\leq \frac{A\ln (T^3)}{\eta_0} + \sum_{t=1}^T \sum_a \left(\frac{1}{\eta_t(a)} - \frac{1}{\eta_{t-1}(a)}\right)\left(\ln\frac{1}{u'(a)} - \ln\frac{1}{p_t(a)}\right)\\
        &\leq \frac{3A\ln T}{\eta_0} + \sum_{t=1}^T \sum_a \left(\frac{1}{\eta_t(a)} - \frac{1}{\eta_{t-1}(a)}\right)\ln(AT^3) \\
        &\leq \frac{3A\ln T}{\eta_0} + 4\sum_{t=1}^T \sum_a \left(\frac{1}{\eta_t(a)} - \frac{1}{\eta_{t-1}(a)}\right)\ln(T) \tag{because $A\leq T$}
    \end{align*}
    Bounding the stability term: 
    \begin{align*}
        \text{stability term} 
        &=  \sum_{t=1}^T \max_{p\in\triangle(\calA)}\Big\{\langle p_t - p, \ell_t + x_t\one\rangle - D_{\psi_t}(p,p_t) \Big\} \leq \sum_{t=1}^T  \sum_a \eta_t(a) p_t(a)^{2}\left(\ell_t(a)+ x_t\right)^2 
    \end{align*}
    where the first equality is because $\langle p_t-p, \one\rangle=0$, and the last inequality is by \pref{lem: stabiity hw 2}. Then noticing that 
    \begin{align*}
        \sum_{t=1}^T \langle p_t-u, \ell_t\rangle  
        &= \sum_{t=1}^T \langle p_t-u', \ell_t\rangle + \sum_{t=1}^T \langle u'-u, \ell_t\rangle \\
        &= \sum_{t=1}^T \langle p_t-u', \ell_t\rangle + \frac{1}{T^3}\sum_{t=1}^T \left\langle -u + \frac{1}{A}\one, \ell_t\right\rangle 
    \end{align*}
    finishes the proof. 
\end{proof}

\section{Analysis for FTRL Regret Bound (\pref{lem: regterm unknown})} \label{app: FTRL bound for alg}  

\subsection{Tsallis entropy}
\begin{proof}[Proof of \pref{lem: regterm unknown} (Tsallis entropy)]
We focus on a particular $s$, and use $\pi_t(a)$, $\hatQ_t(a)$, $B_t(a), C_t(a)$, $\eta_t$, $\mu_t$, $\xi_t$, $b_t$ to denote $\pi_t(a|s)$, $\hatQ_t(s,a)$, $B_t(s,a)$, $C_t(s,a)$, $\eta_t(s)$, $\mu_t(s)$, $\xi_t(s)$, $b_t(s)$, respectively. 

By \pref{lem: Tsallis-INF regret}, we have for any $\pi$
\begin{align}
    &\E\left[\sum_{t=1}^T \inner{\pi_t-\pi, \hatQ_t - B_t - C_t}\right] \\
    &\leq \frac{2\sqrt{A}}{\eta_0} +  \E\left[2\sum_{t=1}^T \left(\frac{1}{\eta_{t}} - \frac{1}{\eta_{t-1}}\right)\xi_t + 2\sum_{t=1}^T \sum_a  \eta_t\pi_t(a)^{\frac{3}{2}}\left(\hatQ_t(a) - B_t(a) - C_t(a) + x_t\right)^2\right]
    \label{eq: Tsallis INF first regret bound}
\end{align}
        for arbitrary $x_t\in\mathbb{R}$ such that $\eta_t\sqrt{\pi_t(a|s)}(\hatQ_t(a)-B_t(a)-C_t(a)+x_t)\geq -\frac{1}{2}$ for all $t,a$. Our choice of $x_t$ is the following:   
        \begin{align}
            x_t = -\left\langle \pi_t,\hatQ_t\right\rangle Y_t.   \label{eq: define x_t for Tsallis}
        \end{align}
        with $Y_t\triangleq \ind\left[\frac{\eta_t}{\mu_t}\leq \frac{1}{8H}\right]$. 
        Below, we verify that $\eta_t\sqrt{\pi_t(a)}\left(\hatQ_t(a)-B_t(a)-C_t(a)+x_t\right)\geq -\frac{1}{2}$: 
        \begin{align*}
            &\eta_t\sqrt{\pi_t(a)}\left(\hatQ_t(a)-B_t(a)-C_t(a)+x_t\right) \\ 
            &\geq \eta_t\sqrt{\pi_t(a)} \left(-B_t(a) - C_t(a) - \left\langle \pi_t,\hatQ_t\right\rangle Y_t\right) \tag{using \eqref{eq: define x_t for Tsallis} and $\hatQ_t(a)\geq 0$} \\
            &\geq -\eta_t B_t(a) - \eta_t C_t(a) - \eta_t \sum_{a'} \pi_t(a')\frac{H \ind_t(s,a')}{\mu_t\pi_t(a')}Y_t \tag{by the definition of $\hatQ_t(a)$}\\
            &\geq -\frac{1}{8H} - \frac{1}{4H^2} - \frac{H\eta_t}{\mu_t}Y_t   \tag{using \pref{lem: Tsallis eta B bound}, $C_t(a)\leq H^2$ and $\eta_t\leq \frac{1}{4H^4}$} \\
            &\geq -\frac{1}{2}. \tag{by the definition of $Y_t=\ind\left[\frac{\eta_t}{\mu_t}\leq \frac{1}{8H}\right]$}
        \end{align*}
    Continued from \eqref{eq: Tsallis INF first regret bound} with the choice of $x_t$: 
    \begin{align}
        &\E\left[\sum_{t=1}^T \inner{\pi_t-\pi, \hatQ_t - B_t - C_t}\right] \nonumber\\ 
        &\leq \frac{2\sqrt{A}}{\eta_0} +\E\left[2\sum_{t=1}^T \left(\frac{1}{\eta_{t}} - \frac{1}{\eta_{t-1}}\right)\xi_t + 2\sum_{t=1}^T \sum_a  \eta_t\pi_t(a)^{\frac{3}{2}}\left(\hatQ_t(a) - \inner{\pi_t, \hatQ_t}Y_t - B_t(a) - C_t(a)\right)^2\right]\nonumber \\ 
        &\leq O(H^4A) + \E\Bigg[2\sum_{t=1}^T \left(\frac{1}{\eta_{t}} - \frac{1}{\eta_{t-1}}\right)\xi_t + 8\sum_{t=1}^T \sum_a  \eta_t\pi_t(a)^{\frac{3}{2}}\left(\left(\hatQ_t(a) - \inner{\pi_t, \hatQ_t}\right)^2 + \hatQ_t(a)^2 Y_t'+ B_t(a)^2+C_t(a)^2\right)\Bigg] \tag{define $Y_t'=1-Y_t$}\\ 
        &\leq O(H^4A) +\E\Bigg[2\sum_{t=1}^T \left(\frac{1}{\eta_{t}} - \frac{1}{\eta_{t-1}}\right)\xi_t + \underbrace{8\sum_{t=1}^T \sum_a  \eta_t\pi_t(a)^{\frac{3}{2}}\left(\left(\hatQ_t(a) - \inner{\pi_t, \hatQ_t}\right)^2 + \hatQ_t(a)^2 Y_t'\right)}_{\term_1}\Bigg] \nonumber\\
        &\qquad + \E\left[\frac{1}{H}\sum_{t=1}^T \sum_a \pi_t(a) B_t(a)\right] + \E\Bigg[8\sum_{t=1}^T \sum_a \eta_t\pi_t(a)C_t(a)^2 \Bigg]. \tag{using \pref{lem: Tsallis eta B bound}} \\
        &   \label{eq: decomposition Tsallis tmp}
    \end{align}
    To bound $\term_1$, notice that 
    \begin{align*}
        \E_t\left[ \left(\hatQ_t(a) - \inner{\pi_t,\hatQ_t}\right)^2 \right] 
        &= \E_t\left[ \left(\frac{\ind_t(s,a) L_{t,h}}{\mu_t \pi_t(a)} - \frac{\ind_t(s)L_{t,h}}{\mu_t}\right)^2 \right]   \tag{assume $s\in\calS_h$}\\
        &\leq \mu_t\pi_t(a) \left(\frac{H}{\mu_t \pi_t(a)} - \frac{H}{\mu_t}\right)^2 + \mu_t(1-\pi_t(a))\left(\frac{H}{\mu_t}\right)^2 \\
        &= \frac{1}{\mu_t\pi_t(a)}(1-\pi_t(a))^2 H^2 + \frac{1}{\mu_t}(1-\pi_t(a))H^2 \\
        &= \frac{1-\pi_t(a)}{\mu_t\pi_t(a)}H^2
    \end{align*}
    and that 
    \begin{align*}
        \E_t\left[\hatQ_t(a)^2 Y_t'\right] = \E_t\left[\left(\frac{\ind_t(s,a)L_{t,h}}{\mu_t\pi_t(a)}\right)^2\right]Y_t' \leq \frac{H^2}{\mu_t\pi_t(a)}Y_t'. 
    \end{align*}
    Therefore, 
    \begin{align*}
        \E[\term_1] &\leq \E\left[ 8H^2\sum_{t=1}^T \sum_a  \eta_t\pi_t(a)^{\frac{3}{2}}\left( \frac{1-\pi_t(a)}{\mu_t\pi_t(a)}  + \frac{1}{\mu_t\pi_t(a)} Y_t'\right)\right] \\
        &\leq \E\left[8H^2 \sum_{t=1}^T \frac{\eta_t}{\mu_t} \sum_a \left(\sqrt{\pi_t(a)}(1-\pi_t(a)) + \sqrt{\pi_t(a)}Y_t'\right)  \right]\\
        &\leq \E\left[8H^2 \sum_{t=1}^T \frac{\eta_t}{\mu_t} \left(\xi_t  + \sqrt{A}Y_t'\right)\right].  %\\
        %&\leq  \E\left[\sum_{t=1}^T \left(\frac{1}{\eta_t} - \frac{1}{\eta_{t-1}}\right) \left(\xi_t + \sqrt{A}Y_t'\right)\right]. 
    \end{align*}
    Notice that 
    \begin{align*}
        \frac{8H^2\eta_t}{\mu_t}\leq 2H\frac{\frac{1}{\mu_t}}{\sqrt{\sum_{\tau=1}^t \frac{1}{\mu_\tau}}} \leq 4H\left(\sqrt{\sum_{\tau=1}^t \frac{1}{\mu_\tau} } - \sqrt{\sum_{\tau=1}^{t-1}  \frac{1}{\mu_\tau} }\right) \leq  \frac{1}{\eta_t}-\frac{1}{\eta_{t-1}}
    \end{align*}
    Thus 
    \begin{align*}
        \E[\term_1]\leq \E\left[\sum_{t=1}^T \left(\frac{1}{\eta_t}-\frac{1}{\eta_{t-1}}\right)\left(\xi_t + \sqrt{A}Y_t'\right)\right], 
    \end{align*}
    and continuing from \eqref{eq: decomposition Tsallis tmp} we have 
    \begin{align*}
        &\E\left[\sum_{t=1}^T \inner{\pi_t-\pi, \hatQ_t - B_t - C_t}\right]\\
        &\leq O(H^4A) + 3\E\left[\sum_{t=1}^T \left(\frac{1}{\eta_t}-\frac{1}{\eta_{t-1}}\right)\left(\xi_t + \sqrt{A}Y_t'\right)\right]+ \E\left[\frac{1}{H}\sum_{t=1}^T \sum_a \pi_t(a) B_t(a)\right] + \E\Bigg[8\sum_{t=1}^T \sum_a \eta_t\pi_t(a)C_t(a)^2 \Bigg] \\
        &\leq O(H^4A) + \E\left[\sum_{t=1}^T b_t\right]+ \E\left[\frac{1}{H}\sum_{t=1}^T \sum_a \pi_t(a) B_t(a)\right] 
    \end{align*}
    with $b_t$ defined in \eqref{eq: b_t Tsallis}. 
    This finishes the proof. 
    
\end{proof}

\begin{lemma}[Tsallis entropy]\label{lem: Tsallis eta B bound}
    $\eta_t(s)B_t(s,a)\leq \frac{1}{8H}$. 
\end{lemma}

\begin{proof}
   By the definition of $b_t(s)$ in \eqref{eq: b_t Tsallis}, we have 
    \begin{align*}
        b_t(s)
        &\leq  8\sqrt{A}\left(\frac{1}{\eta_t(s)}-\frac{1}{\eta_{t-1}(s)}\right)  + 8\eta_t(s)H^4  \tag{$C_t(s,a)\leq H^2$}\\
        &= 32H\sqrt{A}\left(\sqrt{\sum_{\tau=1}^t \frac{1}{\mu_\tau(s)} } - \sqrt{\sum_{\tau=1}^{t-1} \frac{1}{\mu_\tau(s)} }\right)  + 8\eta_t(s)  H^4   \\
        &\leq 32H\sqrt{A}\times \frac{\frac{1}{\mu_t(s)}}{\sqrt{\sum_{\tau=1}^t \frac{1}{\mu_\tau(s)}}} + 8\times \frac{1}{4H^4}\times H^4  \\
        &\leq 32H\sqrt{\frac{A}{\mu_t(s)}} + 2  \leq 34 H\sqrt{\frac{A}{\gamma_t}}.   
    \end{align*}
    
    Therefore, 
    \begin{align*}
        \eta_t(s)B_t(s,a) 
        &\leq \eta_t(s)\left(1+\frac{1}{H}\right)^H H\max_{s'}b_t(s') \\
        &\leq \min\left\{\frac{1}{1600H^4\sqrt{A}}, \frac{1}{4H\sqrt{t}}\right\}\times 34eH^2 \sqrt{\frac{A}{\gamma_t}} \\ 
        &\leq 100H^2\min\left\{\frac{1}{1600H^4\sqrt{A}}, \frac{1}{4H\sqrt{t}}\right\}\times \max\left\{\sqrt{\frac{At}{10^6H^4A^2}}, \sqrt{A}\right\} \tag{by the definition of $\gamma_t$} \\
        &\leq \frac{1}{8H}.
    \end{align*}
\end{proof}

\subsection{Shannon entropy} 
\begin{proof}[Proof of \pref{lem: regterm unknown} (Shannon entropy)]
     We focus on a particular $s$, and use $\pi_t(a)$, $\hatQ_t(a)$, $B_t(a)$, $\eta_t(a)$, $\mu_t$, $b_t$ to denote $\pi_t(a|s)$, $\hatQ_t(s,a)$, $B_t(s,a)$, $\eta_t(s,a)$, $\mu_t(s)$, $b_t(s)$, respectively. 
     
     Notice that for any $t,a$, since $\hatQ_t(a)\geq 0$, $\eta_t(a)B_t(a)\leq \frac{1}{4H}$ (by \pref{lem: eta B for Shannon}), and $\eta_t(a)C_t(a)\leq \frac{1}{4H^4}\times H^2=\frac{1}{4H^2}$ (because $\eta_t(a)\leq \eta_0(a)=\frac{1}{4H^4}$ and $C_t(a)\leq H^2$), we have 
     \begin{align*}
         \eta_t(a)(\hatQ_t(a)-B_t(a)-C_t(a)) \geq -\frac{1}{4H} - \frac{1}{4H^2}\geq -1. 
     \end{align*}
     Besides, for any $a$, 
     \begin{align*}
         \left|\E\left[\sum_{t=1}^T \hatQ_t(a)\right]\right| &\leq \E\left[\sum_{t=1}^T \frac{H}{\mu_t}\right] \leq \sum_{t=1}^T\frac{H}{\gamma_t}\leq HT^2  \tag{by the definition of $\gamma_t$}\\
         \E\left[\sum_{t=1}^T B_t(a) \right] &\leq 400T^2\sqrt{\log T} \tag{by \pref{lem: eta B for Shannon}} \\
         \E\left[\sum_{t=1}^T C_t(a) \right] &\leq H^2T  \tag{$C_t(a)\leq H^2$}
     \end{align*}
     
     %\begin{align*}
     %    \E\left[\sum_{t=1}^T \left(\hatQ_t(a)-B_t(a)-C_t(a)\right)\right] &\leq \sum_{t=1}^T \left(HT + 3H\max_s b_t(s) + 4\times \frac{1}{4H^4}\times H^4\right) \\
     %    &\leq \sum_{t=1}^T \left(HT + 3H\times 16A \times 5H\sqrt{T\log T} + 1\right) \\
     %    &\leq 240H^2AT^2. 
     %\end{align*}
With these inequalities, by \pref{lem: Shannon-INF regret}, the following holds for any $\pi$: 
    \begin{align}
        &\E\left[\sum_{t=1}^T \inner{\pi_t-\pi, \hatQ_t-B_t-C_t}\right] \nonumber \\
        &\leq \sum_a \frac{\ln A}{\eta_0(a)} + \E\left[6\sum_{t=1}^T \sum_a \left(\frac{1}{\eta_{t}(a)} - \frac{1}{\eta_{t-1}(a)}\right)\xi_t(a) + \sum_{t=1}^T \sum_a \eta_t(a)\pi_t(a)\left(\hatQ_t(a) - B_t(a) - C_t(a)\right)^2\right] \\
        &\qquad + \frac{2}{T^2}\max_a \left|\E\left[\sum_{t=1}^T \left(\hatQ_t(a)-B_t(a)-C_t(a)\right) \right]\right| \nonumber \\
        &\leq O(H^4A\ln(T)) +\E\left[6\sum_{t=1}^T \sum_a \left(\frac{1}{\eta_{t}(a)} - \frac{1}{\eta_{t-1}(a)}\right)\xi_t(a) + 3\sum_{t=1}^T \sum_a \eta_t(a)\pi_t(a)\left(\hatQ_t(a)^2 + B_t(a)^2 + C_t(a)^2\right)\right]  \nonumber \\
        &\leq O(H^4A\ln(T)) + \E\left[6\sum_{t=1}^T \sum_a \left(\frac{1}{\eta_t(a)}-\frac{1}{\eta_{t-1}(a)}\right)\xi_t(a) + 3\sum_{t=1}^T \sum_a\eta_t(a)  \left(\frac{H^2}{\mu_t } + \pi_t(a)B_t(a)^2 + \pi_t(a)C_t(a)^2\right)\right]  \nonumber \\
        &\leq O(H^4A\ln(T)) + \E\left[6\sum_{t=1}^T \sum_a \left(\frac{1}{\eta_t(a)}-\frac{1}{\eta_{t-1}(a)}\right)\xi_t(a) + 3\sum_{t=1}^T \sum_a \frac{H^2\eta_t(a)}{\mu_t}\right]  \nonumber \\
        &\qquad + \E\left[\frac{1}{H}\sum_{t=1}^T \sum_a \pi_t(a)B_t(a) + 3\sum_{t=1}^T \sum_a \eta_t(a)\pi_t(a)C_t(a)^2 \right] \tag{by \pref{lem: eta B for Shannon}} \\
        &\label{eq: temp bound in exponential weight}
    \end{align}
    By the update $\eta_t(a)$, 
    \begin{align*}
        \frac{1}{\eta_t(a)}\geq 4H\sqrt{\log T}\sum_{\tau=1}^t \frac{1}{\mu_\tau}\times \frac{1}{\sqrt{\sum_{\tau=1}^T \frac{\xi_\tau(a)}{\mu_\tau} + \max_{\tau\in[T]}\frac{1}{\mu_\tau}}}.  
    \end{align*}
    Therefore, 
    \begin{align*}
        3H^2\sum_{t=1}^T \sum_a \frac{\eta_t(a)}{\mu_t} 
        &\leq \frac{H}{\sqrt{\log T}} \sum_a \sqrt{\sum_{\tau=1}^T \frac{\xi_\tau(a)}{\mu_\tau} + \max_{\tau\in[T]}\frac{1}{\mu_\tau} }\times \sum_{t=1}^T  \frac{\frac{1}{\mu_t}}{\sum_{\tau=1}^t \frac{1}{\mu_\tau}} \\
        &\leq 2H\sqrt{\log T} \sum_a \sqrt{\sum_{\tau=1}^T \frac{\xi_\tau(a)}{\mu_\tau} + \max_{\tau\in[T]}\frac{1}{\mu_\tau}} \\
        &= 2H\sqrt{\log T}\sum_a \sum_{t=1}^T \left(\sqrt{\sum_{\tau=1}^t \frac{\xi_\tau(a)}{\mu_\tau} + \max_{\tau\in[t]}\frac{1}{\mu_\tau}} - \sqrt{\sum_{\tau=1}^{t-1} \frac{\xi_\tau(a)}{\mu_\tau} + \max_{\tau\in[t-1]}\frac{1}{\mu_\tau}}\right) \\
        &\leq 2H\sqrt{\log T}\sum_a \sum_{t=1}^T \frac{\frac{\xi_t(a)}{\mu_t} + \max_{\tau\in[t]}\frac{1}{\mu_\tau} - \max_{\tau\in[t-1]}\frac{1}{\mu_\tau}}{\sqrt{\sum_{\tau=1}^t \frac{\xi_\tau(a)}{\mu_\tau}+\max_{\tau\in[t]}\frac{1}{\mu_\tau} }} \\
        &= 2H\sqrt{\log T}\sum_a \sum_{t=1}^T \frac{\frac{\xi_t(a)}{\mu_t} + \frac{1}{\mu_t}\left(1-\frac{\min_{\tau\in[t]}\mu_\tau}{\min_{\tau\in[t-1]}\mu_\tau}\right)}{\sqrt{\sum_{\tau=1}^t \frac{\xi_\tau(a)}{\mu_\tau}+\max_{\tau\in[t]}\frac{1}{\mu_\tau} }} \\
        &\leq \sum_{t=1}^T \sum_a \left(\frac{1}{\eta_t(a)} - \frac{1}{\eta_{t-1}(a)}\right)\left(\xi_t(a) + 1-\frac{\min_{\tau\in[t]}\mu_\tau}{\min_{\tau\in[t-1]}\mu_\tau}\right)
    \end{align*}
    where we use \eqref{eq: eta for Shannon} in the last inequality. 
    Using this in \eqref{eq: temp bound in exponential weight}, we get 
    \begin{align*}
         \E\left[\sum_{t=1}^T \inner{\pi_t-\pi, \hatQ_t-B_t-C_t}\right] &\leq O(H^4A\ln(T)) + \E\left[7\sum_{t=1}^T \sum_a \left(\frac{1}{\eta_t(a)} - \frac{1}{\eta_{t-1}(a)}\right)\left(\xi_t(a) + 1-\frac{\min_{\tau\in[t]}\mu_\tau}{\min_{\tau\in[t-1]}\mu_\tau}\right)\right] \\
         &\qquad + \E\left[\frac{1}{H}\sum_{t=1}^T \sum_a \pi_t(a)B_t(a) + 3\sum_{t=1}^T \sum_a \eta_t(a)\pi_t(a)C_t(a)^2 \right] \\
         &\leq O(H^4A\ln(T)) + \E\left[\sum_{t=1}^T b_t + \frac{1}{H}\sum_{t=1}^T \sum_a \pi_t(a)B_t(a)\right], 
    \end{align*}
    where we use the definition of $b_t$ in \eqref{eq: def bt in shannon}. This finishes the proof.

    %and using the definition of $b_t$ in \pref{fig: psi and b choices} (Shannon entropy), we get 
    %\begin{align*}
    %    \E\left[\sum_{t=1}^T \inner{\pi_t-\pi, \hatQ_t-B_t-C_t}\right]&\leq \E\left[\sum_{t=1}^T b_t + \frac{1}{H}\sum_{t=1}^T \sum_a \pi_t(a)B_t(a)\right]. 
    %\end{align*}
    %where we use 
%\begin{align*}
%    \frac{1}{\eta_t(a)} \geq \frac{H}{\sqrt{\sum_{\tau=1}^t \frac{\phi_\tau(a)}{\mu_\tau}}} \times \sum_{\tau=1}^t \frac{1}{\mu_\tau}. 
%\end{align*}
\end{proof}

\begin{lemma}[Shannon entropy]\label{lem: eta B for Shannon}
    $\eta_t(s,a)B_t(s,a)\leq \frac{1}{4H}$ and $B_t(s,a)\leq 400\sqrt{T\log T}$. 
\end{lemma}

\begin{proof}
   By the definition of $b_t(s)$ in \eqref{eq: def bt in shannon}, we have \begin{align*}
       b_t(s) 
       &\leq 16\sum_a\left(\frac{1}{\eta_t(s,a)} - \frac{1}{\eta_{t-1}(s,a)}\right)+ 8\sum_{a}\eta_t(s,a)\pi_t(a|s)H^4 \tag{$C_t(s,a)\leq H^2$}\\
       &\leq 64\sum_a \left(\frac{H}{\mu_t(s) \sqrt{\sum_{\tau=1}^{t-1} \frac{\xi_\tau(s,a)}{\mu_\tau(s)}+ \frac{1}{\mu_t(s)}} } + \frac{H}{\sqrt{t}}\right)\sqrt{\log T} + 2 \tag{using \eqref{eq: eta for Shannon} and $\eta_t(s,a)\leq \frac{1}{4H^4}$} \\
       &\leq 64\left(\frac{HA}{\sqrt{\mu_t(s)}}   + \frac{HA}{\sqrt{t}} \right)\sqrt{\log T}+2\leq \frac{132HA\sqrt{\log T}}{\sqrt{\gamma_t}}. 
    \end{align*}
   Further notice that
   \begin{align*}
       \frac{1}{\eta_t(s,a)}\geq  4\sum_{\tau=1}^t \frac{H\sqrt{\log T}}{\sqrt{\tau}} \geq 4H\sqrt{t\log T}. 
   \end{align*}
    
    Therefore, 
    \begin{align*}
        B_t(s,a)  &\leq   H\left(1+\frac{1}{H}\right)^H\max_{s}b_t(s)\leq \frac{396H^2A\sqrt{\log T}}{\sqrt{\gamma_t}} \leq 400H^2A\sqrt{T\log T} \\
        \eta_t(s,a)B_t(s,a) &\leq \min\left\{\frac{1}{1600H^4A\sqrt{\log T}}, \frac{1}{4H\sqrt{t\log T}} \right\} \times\frac{396H^2A\sqrt{\log T}}{\sqrt{\gamma_t}} \\
        &\leq \min\left\{\frac{1}{1600H^4A\sqrt{\log T}}, \frac{1}{4H\sqrt{t\log T}} \right\} \times\max\left\{\frac{396H^2A\sqrt{t\log T}}{\sqrt{10^6H^4A^2}}, 396H^2A\sqrt{\log T}\right\} \tag{by the definition of $\gamma_t$}\\
        &\leq \frac{1}{4H}
    \end{align*}
    by the definition of $\gamma_t$. 
\end{proof}

\subsection{Log barrier} 

\begin{proof}[Proof of \pref{lem: regterm unknown} (log barrier)]
We focus on a particular $s$, and use $\pi_t(a)$, $\hatQ_t(a)$, $B_t(a), C_t(a)$, $\eta_t$, $\mu_t$, $\zeta_t(a)$, to denote $\pi_t(a|s)$, $\hatQ_t(s,a)$, $B_t(s,a)$, $C_t(s,a)$, $\eta_t(s)$, $\mu_t(s)$, $\zeta_t(s,a)$, respectively. 

By \pref{lem: log barrier regret}, 
       \begin{align}
        &\E\left[\sum_{t=1}^T \inner{\pi_t-\pi, \hatQ_t - B_t - C_t}\right]\nonumber  \\
        &\leq O(H^4A\ln(T)) + \E\left[4\sum_{t=1}^T \sum_a \left(\frac{1}{\eta_{t}(a)} - \frac{1}{\eta_{t-1}(a)}\right)\log(T) + \sum_{t=1}^T \sum_a  \eta_t(a)\pi_t(a)^{2}\left(\hatQ_t(a) - B_t(a) - C_t(a) + x_t\right)^2\right] \nonumber \\
        &\qquad + \frac{2}{T^3}\max_a \left|\E\left[\sum_{t=1}^T \hatQ_t(a)-B_t(a)-C_t(a) \right]\right|\label{eq: log barrier first ineql}
        \end{align}
        for arbitrary $x_t\in\mathbb{R}$ such that $\eta_t(a)\pi_t(a)(\hatQ_t(a)-B_t(a)-C_t(a)+x_t)\geq -1$. 
        Recall that with log barrier, there are real episodes and virtual episodes in which $\ell_t(s,a)=0$ for all $(s,a)$. Let $Y_t=0$ if $t$ is a virtual episode, and $Y_t=1$ otherwise.  
        
        We define 
        \begin{align}
            x_t = -\left\langle \pi_t,\hatQ_t\right\rangle.   \label{eq: define x_t for log}
        \end{align}
        Below, we verify that $\eta_t(a)\pi_t(a)\left(\hatQ_t(a)-B_t(a)-C_t(a)+x_t\right)\geq -\frac{1}{2}$: 
        \begin{align*}
            &\eta_t(a)\pi_t(a)\left(\hatQ_t(a)-B_t(a)-C_t(a)+x_t\right) \\ 
            &\geq \eta_t(a)\pi_t(a) \left(-B_t(a) - C_t(a) - \left\langle \pi_t,\hatQ_t\right\rangle \right) \tag{using \eqref{eq: define x_t for log} and $\hatQ_t(a)\geq 0$} \\
            &\geq -\eta_t(a)\pi_t(a|s) B_t(a) - \eta_t(a) C_t(a) - \eta_t(a) \sum_{a'} \pi_t(a')\frac{H \ind_t(s,a')}{\mu_t\pi_t(a')}Y_t \tag{when $Y_t=0$, $\hatQ_t(a)=0$}\\
            &\geq -\frac{1}{8H} - \frac{1}{4H^2} - \frac{H\eta_t}{\mu_t}Y_t   \tag{by \pref{lem: eta B log barrier} and that $C_t(a)\leq H^2$ and $\eta_t(a)\leq \frac{1}{4H^4}$} \\
            &\geq -\frac{1}{2}. \tag{when $Y_t=1$ (real episode), $\frac{\eta_t(a)}{\mu_t}\leq \frac{1}{8H}$}
        \end{align*}
     Besides, for any $a$, 
     \begin{align*}
         \left|\E\left[\sum_{t=1}^T \hatQ_t(a)\right]\right| &\leq \E\left[\sum_{t=1}^T \frac{H}{\mu_t}\right] \leq \sum_{t=1}^T \frac{H}{\gamma_t}\leq HT^2 \tag{by the definition of $\gamma_t$}\\
         \E\left[\sum_{t=1}^T B_t(a) \right] &\leq  15ST^2 \tag{by \pref{lem: eta B log barrier}} \\
         \E\left[\sum_{t=1}^T C_t(a) \right] &\leq H^2 T \tag{$C_t(a)\leq H^2$}
     \end{align*}

        Below, we continue from \eqref{eq: log barrier first ineql} with our choice of $x_t$: 
        \begin{align*}
        & \E\left[\sum_{t=1}^T \inner{\pi_t-\pi, \hatQ_t - B_t - C_t}\right]\\
        &\leq O(H^4SA\ln(T)) +\E\Bigg[4\sum_{t=1}^T\sum_a \left(\frac{1}{\eta_{t}(a)} - \frac{1}{\eta_{t-1}(a)}\right)\log(T)  \\
        &\qquad + 3\sum_{t=1}^T \sum_a  \eta_t(a)\pi_t(a)^{2}\left(\left(\hatQ_t(a) - \inner{\pi_t, \hatQ_t}\right)^2 + B_t(a)^2 + C_t(a)^2\right)\Bigg] \\ 
        &\leq O(H^4SA\ln(T)) + \E\Bigg[\underbrace{4\sum_{t=1}^T\sum_a  \left(\frac{1}{\eta_{t}(a)} - \frac{1}{\eta_{t-1}(a)}\right)\log(T)}_{\term_1} + \underbrace{3\sum_{t=1}^T \sum_a  \eta_t(a)\pi_t(a)^{2}\left(\hatQ_t(a) - \inner{\pi_t, \hatQ_t}\right)^2 }_{\term_2}\Bigg] \\
        &\qquad + \E\left[\frac{1}{H}\sum_{t=1}^T \sum_a \pi_t(a) B_t(a)\right] + \E\Bigg[\underbrace{3\sum_{t=1}^T \sum_a \eta_t(a)\pi_t(a)C_t(a)^2 }_{\term_3}\Bigg]. \tag{by \pref{lem: eta B log barrier}}
    \end{align*}
    We further manipulate $\term_2$ (suppose that $s\in\calS_h$). In virtual episodes, $\term_2=0$, and in real episodes, 
    \begin{align*}
         \eta_t(a) \pi_t(a)^2\left(\hatQ_t(a) - \inner{\pi_t,\hatQ_t}\right)^2  
        &=  \eta_t(a)\pi_t(a|s)^2\left(\frac{\ind_t(s,a) L_{t,h}}{\mu_t \pi_t(a)} - \frac{\ind_t(s)L_{t,h}}{\mu_t}\right)^2  \\
        &=  \eta_t(a)\left(\frac{\ind_t(s,a) L_{t,h}}{\mu_t } - \frac{\pi_t(a|s)\ind_t(s)L_{t,h}}{\mu_t}\right)^2 \\
        &=  \frac{\eta_t(a)}{\mu_t^2}\left(\ind_t(s,a) - \pi_t(a|s)\ind_t(s)\right)L_{t,h}^2 \\
        &=  \frac{\eta_t(a)\zeta_t(a)}{\mu_t^2} \\
        &\leq \frac{\log T}{4}\left(\frac{1}{\eta_{t+1}(a)} - \frac{1}{\eta_t(a)}\right)  \tag{by \pref{eq: log barrier eat def}}
    \end{align*}
    By the definition of $b_t$ in \eqref{eq: b_t for log barrier}, we have $\E[\term_1 + \term_2 + \term_3] \leq \E\left[\sum_{t=1}^T b_t\right]$, which finishes the proof. 
    
\end{proof}

\begin{lemma}[log barrier]\label{lem: eta B log barrier}
    $\eta_t(s,a)\pi_t(a|s)B_t(s,a)\leq \frac{1}{8H}$ and $B_t(s,a)\leq 15ST$. 
\end{lemma}
\begin{proof}
    If $t$ is a real episode, 
    \begin{align}
        b_t(s) &= 8\sum_a \left(\frac{1}{\eta_{t+1}(s,a)} - \frac{1}{\eta_t(s,a)}\right) \nonumber \\
        &= 32\sum_a \frac{\eta_t(s,a)\ind_t(s,a)L_t(s,a)^2}{\mu_t(s)^2} \leq 32H^2\times \max_a \frac{\eta_t(s,a)}{\mu_t(s)}\times \frac{1}{\mu_t(s)} \leq \frac{1}{\mu_t(s)}\times 32H^2\max_{s',a'}\left(\frac{\eta_t(s',a')}{\mu_t(s')}\right).   \label{eq: real episode b bound}
    \end{align}
    Therefore, 
    \begin{align}
        B_t(s,a) &\leq b_t(s) +  3\sum_{s': h(s')>h(s)} \mu^{\tildeP_t, \pi_t}(s'|s,a)b_t(s') \nonumber \\
        &\leq \left(\frac{1}{\mu_t(s)} + 3\sum_{s': h(s')>h(s)}\mu^{\tildeP_t, \pi_t}(s'|s,a)\frac{1}{\mu_t(s')}\right)\times 32H^2\max_{s',a'}\left(\frac{\eta_t(s',a')}{\mu_t(s')}\right) \nonumber  \\
        &\leq \left(\frac{1}{\mu_t(s)}  + 3\sum_{s': h(s')>h(s)} \mu^{\tildeP_t, \pi_t}(s'|s,a)\times \frac{1}{\overline{\mu}_t^{\pi_t}(s)\pi_t(a|s) \mu^{\tildeP_t,\pi_t}(s'|s,a) + \gamma_t} \right)\times 32H^2\max_{s',a'}\left(\frac{\eta_t(s',a')}{\mu_t(s')}\right) \nonumber \\
        &\leq 3\sum_{s'}  \frac{1}{\mu^{\pi_t}(s)\pi_t(a|s) + \gamma_t}\times 32H^2\max_{s',a'}\left(\frac{\eta_t(s',a')}{\mu_t(s')}\right) \nonumber \\
        &\leq \frac{S}{\mu_t(s)\pi_t(a|s)}\times 96H^2\max_{s',a'}\left(\frac{\eta_t(s',a')}{\mu_t(s')}\right) \label{eq: b to B calculation} 
    \end{align}
    and thus
    \begin{align}
        \eta_t(s,a)\pi_t(a|s)B_t(s,a)\leq 96H^2S\max_{s',a'}\left(\frac{\eta_t(s',a')}{\mu_t(s')}\right)^2 \leq \frac{1}{8H}   \label{eq: b to B calculation 2}
    \end{align}
    where the last inequality is because $\frac{\eta_t(s',a')}{\mu_t(s')}\leq \frac{1}{60\sqrt{H^3S}}$ in real episodes. 
    
    From the second-to-last step in \eqref{eq: b to B calculation}, we also have 
    \begin{align*}
        B_t(s,a)\leq\frac{3S}{\gamma_t}\times 32H^2\max_{s',a'}\left(\frac{\eta_t(s',a')}{\mu_t(s')}\right)\leq \frac{2\sqrt{HS}}{\gamma_t}\leq 2ST. 
    \end{align*}
    
    In virtual episodes, 
    \begin{align*}
        b_t(s) &\leq  \sum_a \left(\frac{1}{\eta_{t+1}(s,a)} - \frac{1}{\eta_t(s,a)}\right)\log(T)\\
        &\leq \sum_a  \frac{\ind\{(s_t^\dagger, a_t^\dagger)=(s,a)\}}{24\eta_t(s,a) H\log T}\times \log T \\
        &=\sum_a  \frac{\ind\{(s_t^\dagger, a_t^\dagger)=(s,a)\}}{24\mu_t(s) H} \times \frac{1}{\max_{s',a'}\left(\frac{\eta_t(s',a')}{\mu_t(s')}\right)} \tag{by the definition of $(s_t^\dagger, a_t^\dagger)$} \\
        &\leq \frac{\ind\{s_t^\dagger=s\}}{24\mu_t(s)H}  \times \frac{1}{\max_{s',a'}\left(\frac{\eta_t(s',a')}{\mu_t(s')}\right)} \\
        &\leq \frac{\ind\{s_t^\dagger=s\}}{\mu_t(s)}  \times \frac{1}{24HM_t} 
    \end{align*}
    where we define $M_t=\max_{s',a'}\frac{\eta_t(s',a')}{\mu_t(s')}$. Similar to \eqref{eq: b to B calculation}:  
    \begin{align}
        B_t(s,a) &\leq b_t(s) +  3\sum_{s':h(s')>h(s)} \mu^{\tildeP_t, \pi_t}(s'|s,a)b_t(s') \nonumber \\
        &\leq \left(\frac{\ind\{s_t^\dagger=s\}}{\mu_t(s)} + 3\sum_{s'}\mu^{\tildeP_t, \pi_t}(s'|s,a)\frac{\ind\{s_t^\dagger=s'\}}{\mu_t(s')}\right)\times \frac{1}{24HM_t} \nonumber  \\
        &\leq \left(\frac{\ind\{s_t^\dagger=s\}}{\mu_t(s)}  + 3\sum_{s':h(s')>h(s)} \mu^{\tildeP_t, \pi_t}(s'|s,a)\times \frac{\ind\{s_t^\dagger=s'\}}{\overline{\mu}_t^{\pi_t}(s)\pi_t(a|s) \mu^{\tildeP_t,\pi_t}(s'|s,a) + \gamma_t} \right)\times  \frac{1}{24HM_t} \nonumber \\
        &\leq 3\sum_{s'}  \frac{\ind\{s_t^\dagger=s'\}}{\mu^{\pi_t}(s)\pi_t(a|s) + \gamma_t}\times  \frac{1}{24HM_t} \nonumber \\
        &\leq  \frac{1}{\mu_t(s)\pi_t(a|s)}\times \frac{1}{8HM_t} \label{eq: b to B calculation virtual} 
    \end{align}
    and thus
    \begin{align*}
        \eta_t(s,a)\pi_t(a|s)B_t(s,a)\leq \frac{\eta_t(s,a)}{\mu_t(s)}\times \frac{1}{8HM_t}\leq \frac{1}{8H} 
    \end{align*}
    where the last step uses the definition of $M_t$. 
    
    From the second-to-last step in \eqref{eq: b to B calculation virtual} , we also have 
    \begin{align*}
        B_t(s,a)\leq \frac{1}{8\gamma_t HM_t}\leq \frac{15\sqrt{HS}}{\gamma_t}\leq 15ST 
    \end{align*}
    where we use that $M_t\geq \frac{1}{60\sqrt{H^3S}}$ in vitrual episodes. 
\end{proof}

\section{Analysis for the Bias (\pref{lem: biasterm known})}\label{app: bias part}

\begin{proof}[Proof of \pref{lem: biasterm known}]
    
   \begin{align}
       &\E\left[\sum_s \mu^{\pi}(s)\biasterm^\pi(s)\right] \nonumber \\
       &\leq \E\left[\sum_{t=1}^T \sum_{s,a} \mu^{\pi}(s)\left(\pi_t(a|s) -  \pi(a|s)\right) \left(Q^{\pi_t}(s,a;\ell_t) -  \hatQ_t(s,a) + C_t(s,a)\right)\right] \\
       &= \E\left[\sum_{t=1}^T \sum_{s,a} \mu^{\pi}(s)\left(\pi_t(a|s) -  \pi(a|s)\right) \left(Q^{\pi_t}(s,a;\ell_t) -  \frac{ \mu^{\pi_t}(s)}{\mu_t(s)}Q^{\pi_t}(s,a;\ell_t) + C_t(s,a)\right)\right]  \nonumber \\
       &= \E\left[\sum_{t=1}^T \sum_{s,a} \mu^{\pi}(s)\left(\pi_t(a|s) -  \pi(a|s)\right) \left(  \frac{ \mu_t(s) -  \mu^{\pi_t}(s)}{\mu_t(s)}Q^{\pi_t}(s,a;\ell_t) + C_t(s,a)\right)\right]  \nonumber  \\
       &= \E\left[\sum_{t=1}^T \sum_{s,a} (\mu^{\pi_t}(s,a) - \mu^{\pi}(s,a)) z_t(s,a)\right]   \label{eq: bias decomposition}
   \end{align}
   with $z_t(s,a)$ defined as the following based on \pref{lem: variant performance diff lemma}:  
   \begin{align*}
       z_t(s,a) 
       &\triangleq  \frac{ \mu_t(s) -  \mu^{\pi_t}(s)}{\mu_t(s)}Q^{\pi_t}(s,a;\ell_t) + C_t(s,a) - \E_{s'\sim P(\cdot|s,a), a'\sim\pi_t(\cdot|s')}\left[ \frac{ \mu_t(s') -  \mu^{\pi_t}(s')}{\mu_t(s')}Q^{\pi_t}(s',a';\ell_t) + C_t(s',a') \right] 
   \end{align*}
   
   Recall the high probability event $\calE$ defined in \pref{def: good event}. Notice that 
   \begin{align}
       &\E\left[\sum_{t=1}^T \sum_{s,a} (\mu^{\pi_t}(s,a) - \mu^{\pi}(s,a)) z_t(s,a)\right] \nonumber \\
       &= \Pr(\calE)\E\left[\sum_{t=1}^T \sum_{s,a} (\mu^{\pi_t}(s,a) - \mu^{\pi}(s,a)) z_t(s,a)~\bigg|~\calE\right] + \Pr(\overline{\calE})\E\left[\sum_{t=1}^T \sum_{s,a} (\mu^{\pi_t}(s,a) - \mu^{\pi}(s,a)) z_t(s,a)~\bigg|~\overline{\calE}\right] \nonumber \\
       &\leq \Pr(\calE)\E\left[\sum_{t=1}^T \sum_{s,a} (\mu^{\pi_t}(s,a) - \mu^{\pi}(s,a)) z_t(s,a)~\bigg|~\calE\right] + O(H\delta)\times O(TH\times TH^2) \tag{because $|z_t(s,a)|\leq O(TH^2)$ almost surely} \\
       &\leq \Pr(\calE)\E\left[\sum_{t=1}^T \sum_{s,a} (\mu^{\pi_t}(s,a) - \mu^{\pi}(s,a)) z_t(s,a)~\bigg|~\calE\right] + O\left(\frac{H^4}{T}\right).  \tag{$\delta=\frac{1}{T^3}$}   \\
       &\label{eq: middle point}
   \end{align}
   
   From now on, it suffices to bound $\sum_{t=1}^T \sum_{s,a} (\mu^{\pi_t}(s,a) - \mu^{\pi}(s,a)) z_t(s,a)$ assuming $\calE$ holds (i.e., $P\in\calP_t$ for all $t$). 
   
   By the definition of $C_t(s,a)$, we have 
   \begin{align}
       z_t(s,a)   
       &= \frac{ \mu_t(s) -  \mu^{\pi_t}(s)}{\mu_t(s)}Q^{\pi_t}(s,a;\ell_t) + \max_{\tildeP\in\calP_t} \E_{s'\sim \tildeP(\cdot|s,a), a'\sim\pi_t(\cdot|s')}\left[ \frac{ \mu_t(s') -  \underline{\mu}_t^{\pi_t}(s')}{\mu_t(s')}H + C_t(s',a') \right]  \nonumber \\
       &\qquad - \E_{s'\sim P(\cdot|s,a), a'\sim\pi_t(\cdot|s')}\left[ \frac{ \mu_t(s') -  \mu^{\pi_t}(s')}{\mu_t(s')}Q^{\pi_t}(s',a';\ell_t) + C_t(s',a') \right]  \nonumber \\
       &\geq \frac{ \mu_t(s) -  \mu^{\pi_t}(s)}{\mu_t(s)}Q^{\pi_t}(s,a;\ell_t) \geq 0.   \label{eq: z positive}
   \end{align}
   On the other hand, 
   \begin{align}
       z_t(s,a) 
       &\leq \frac{\mu_t(s) - \mu^{\pi_t}(s)}{\mu_t(s)}H + C_t(s,a) - \E_{s'\sim P(\cdot|s,a), a'\sim\pi_t(\cdot|s')}\left[C_t(s',a')\right] \nonumber  \\
       &\leq c_t(s) + \E_{s'\sim \overline{P}_t(\cdot|s,a), a'\sim \pi_t(\cdot|s')}\left[c_t(s') + C_t(s',a')\right] - \E_{s'\sim P(\cdot|s,a), a'\sim\pi_t(\cdot|s')}\left[C_t(s',a')\right] \tag{let $\overline{P}_t$ be the transition that attains the maximum in \eqref{eq: C def} } \\
       &\leq c_t(s) + \E_{s'\sim P(\cdot|s,a)}[c_t(s')] + \sum_{s',a'}\left|\overline{P}_t(s'|s,a) - P(s'|s,a)\right| \pi_t(a'|s')\left(c_t(s') + C_t(s',a')\right)   \nonumber\\
       %&\leq c_t(s) + \E_{s'\sim P(\cdot|s,a)}[c_t(s')] + \sum_{s',a'} e_t(s'|s,a)\pi_t(a'|s')(c_t(s')+C_t(s',a')) \nonumber \\
       &\leq c_t(s) + \E_{s'\sim P(\cdot|s,a)}[c_t(s')] + \sum_{s',a'} e_t(s'|s,a)\pi_t(a'|s')(c_t(s')+C_t(s',a')) \label{eq: z upper bound}
   \end{align}
   where we define $e_t(s'|s,a) =  \left|\overline{P}_t(s'|s,a)-P(s'|s,a)\right|$. 
   
   Observe that by the definition of $C_t(s,a)$, it holds that 
   \begin{align*}
       C_t(s,a) = \sum_{s': h(s')>h(s)} \mu^{\overline{P}_t, \pi_t}(s'|s,a)c_t(s'),  
   \end{align*}
   and therefore, 
   \begin{align*}
       c_t(s) + C_t(s,a) = \sum_{s'} \mu^{\overline{P}_t, \pi_t}(s'|s,a)c_t(s')  
   \end{align*}
   and 
   \begin{align*}
       \sum_a \pi_t(a|s)\left(c_t(s)+C_t(s,a)\right) = \sum_{s'} \mu^{\overline{P}_t, \pi_t}(s'|s)c_t(s').  
   \end{align*}
   Thus we can thus rewrite \eqref{eq: z upper bound} as 
   \begin{align}
       z_t(s,a)\leq c_t(s) + \E_{s'\sim P(\cdot|s,a)}[c_t(s')] + \sum_{s'} e_t(s'|s,a)\sum_{s''} \mu^{\overline{P}_t,\pi_t}(s''|s')c_t(s''). \label{eq: transformed}
   \end{align}
   Continue from the previous calculation in \eqref{eq: middle point}:  
   \begin{align*}
       &\sum_{t=1}^T \sum_{s,a} \left(\mu^{\pi_t}(s,a) - \mu^{\pi}(s,a)\right) z_t(s,a) \\ 
       &\leq \sum_{t=1}^T \sum_{s,a} \left[\mu^{\pi_t}(s,a) - \mu^{\pi}(s,a)\right]_+ z_t(s,a)   \tag{by \eqref{eq: z positive}}\\
       &\leq \underbrace{\sum_{t=1}^T \sum_{s,a} \left[\mu^{\pi_t}(s,a) - \mu^{\pi}(s,a)\right]_+ c_t(s)}_{\term_1} + \underbrace{\sum_{t=1}^T \sum_{s,a} \left[\mu^{\pi_t}(s,a) - \mu^{\pi}(s,a)\right]_+ \E_{s'\sim P(\cdot|s,a)}[c_t(s')]}_{\term_2} \\
       &\qquad + \underbrace{\sum_{t=1}^T \sum_{s,a} \left[\mu^{\pi_t}(s,a) - \mu^{\pi}(s,a)\right]_+   \sum_{s'} e_t(s'|s,a)\sum_{s''} \mu^{\overline{P}_t,\pi_t}(s''|s')c_t(s'')}_{\term_3}.    \tag{by \eqref{eq: transformed}}
   \end{align*}

\noindent \boxed{\text{Known transition case}} \\
For the known transition case, we have 
\begin{align*}
    c_t(s)\leq \frac{\mu^{\pi_t}(s)+\gamma_t - \mu^{\pi_t}(s)}{\mu_t(s)}H =  \frac{\gamma_t}{\mu_t(s)}H
\end{align*}
and $e_t(s'|s,a)=0$. Thus, 
\begin{align*}
    &\E\left[\sum_s \mu^{\pi}(s)\biasterm^\pi(s)\right] \lesssim \sum_{t=1}^T \sum_{s}\mu^{\pi_t}(s)\times \frac{\gamma_t}{\mu_t(s)}H \leq HS\sum_{t=1}^T \gamma_t =O\left(H^5SA^2\ln(T)\right). 
\end{align*}

\noindent \boxed{\text{Unknown transition case}} \\
\textbf{Upper bounding $\term_1$.\ \ }  
By the definition of $c_t(s)$,  
\begin{align*}
     \term_1 
     &\leq H\sum_{t=1}^T \sum_{s,a} \left[\mu^{\pi_t}(s,a) - \mu^{\pi}(s,a)\right]_+ \left(\frac{\overline{\mu}_t^{\pi_t}(s) - \underline{\mu}_t^{\pi_t}(s) + \gamma_t}{\mu_t(s)}\right) \\
     &\leq \underbrace{H\sum_{t=1}^T \sum_{s,a} \left[\mu^{\pi_t}(s,a) - \mu^{\pi}(s,a)\right]_+ \left(\frac{\overline{\mu}_t^{\pi_t}(s) - \mu^{\pi_t}(s)}{\mu_t(s)}\right)}_{\term_{1a}} \\
     &\qquad + \underbrace{H\sum_{t=1}^T \sum_{s,a} \left[\mu^{\pi_t}(s,a) - \mu^{\pi}(s,a)\right]_+ \left(\frac{\mu^{\pi_t}(s) - \underline{\mu}_t^{\pi_t}(s)}{\mu_t(s)}\right)}_{\term_{1b}} + \underbrace{\sum_{t=1}^T \sum_{s,a} \mu^{\pi_t}(s,a) \left(\frac{H\gamma_t}{\mu_t(s)}\right)}_{\term_{1c}}. 
\end{align*}

   To bound $\term_{1a}$, we apply \pref{lem: complicated lemma} with 
   \begin{align*}
       g_t(s) = \sum_a \frac{\left[\mu^{\pi_t}(s,a) - \mu^{\pi}(s,a)\right]_+}{\mu_t(s)} \leq 1,  
   \end{align*}
   which gives 
   \begin{align*}
       \term_{1a}
       &\leq  \sqrt{H^3S^2A\sum_{t=1}^T  \sum_{s,a} \mu^{\pi_t}(s)\frac{\left[\mu^{\pi_t}(s,a) - \mu^{\pi}(s,a)\right]_+}{\mu_t(s)} \ln(T)\iota} + H^2S^4A\ln(T)\iota \\
       &\leq \sqrt{H^3S^2A\sum_{t=1}^T  \sum_{s,a} \left[\mu^{\pi_t}(s,a) - \mu^{\pi}(s,a)\right]_+ \ln(T)\iota} + H^2S^4A\ln(T)\iota.  
   \end{align*}
   
   $\term_{1b}$ can be bound in the same way and admits the same upper bound. $\term_{1c}\leq HS\sum_{t=1}^T \gamma_t=O\left(H^5SA^2\ln(T)\right)$. Combining $\term_1, \term_2, \term_3$, we get 
   \begin{align*}
       \term_1 \lesssim \sqrt{H^3S^2A\sum_{t=1}^T  \sum_{s,a} \left[\mu^{\pi_t}(s,a) - \mu^{\pi}(s,a)\right]_+ \ln(T)\iota} + H^2S^4A^2\ln(T)\iota.  
   \end{align*}

   \paragraph{Upper bounding $\term_2$. } This is very similar to the procedure of bounding $\term_1$. We perform a similar decomposition: 
   \begin{align*}
       \term_2 
       &\leq \underbrace{H\sum_{t=1}^T \sum_{s,a} \left[\mu^{\pi_t}(s,a) - \mu^{\pi}(s,a)\right]_+ \E_{s'\sim P(\cdot|s,a)}\left[\frac{\overline{\mu}_t^{\pi_t}(s') - \mu^{\pi_t}(s')}{\mu_t(s')}\right]}_{\term_{2a}} \\
       &\qquad + \underbrace{H\sum_{t=1}^T \sum_{s,a} \left[\mu^{\pi_t}(s,a) - \mu^{\pi}(s,a)\right]_+ \E_{s'\sim P(\cdot|s,a)}\left[\frac{\mu^{\pi_t}(s') - \underline{\mu}_t^{\pi_t}(s')}{\mu_t(s')}\right]}_{\term_{2b}} + \underbrace{\sum_{t=1}^T \sum_{s,a} \mu^{\pi_t}(s,a) \E_{s'\sim P(\cdot|s,a)}\left[\frac{H\gamma_t}{\mu_t(s')}\right]}_{\term_{2c}}. 
   \end{align*}
   To bound $\term_{2a}$, we apply \pref{lem: complicated lemma} with 
   \begin{align*}
       g_t(s') = \sum_{s,a} \frac{\left[\mu^{\pi_t}(s,a) - \mu^{\pi}(s,a)\right]_+}{\mu_t(s')}P(s'|s,a) \leq \frac{\sum_{s,a}\mu^{\pi_t}(s,a)P(s'|s,a)}{\mu_t(s')} \leq \frac{\mu^{\pi_t}(s')}{\mu_t(s')}\leq 1, 
   \end{align*}
   which gives 
   \begin{align*}
       \term_{2a} &\leq \sqrt{H^3S^2A\sum_{t=1}^T \sum_{s'}\mu^{\pi_t}(s') \sum_{s,a} \frac{\left[\mu^{\pi_t}(s,a) - \mu^{\pi}(s,a)\right]_+}{\mu_t(s')} P(s'|s,a)\ln(T)\iota} + H^2S^4A\ln(T)\iota \\
       &\leq \sqrt{H^3S^2A\sum_{t=1}^T \sum_{s'} \sum_{s,a} \left[\mu^{\pi_t}(s,a) - \mu^{\pi}(s,a)\right]_+ P(s'|s,a)\ln(T)\iota} + H^2S^4A\ln(T)\iota \\
       &= \sqrt{H^3S^2A\sum_{t=1}^T  \sum_{s,a} \left[\mu^{\pi_t}(s,a) - \mu^{\pi}(s,a)\right]_+ \ln(T)\iota} + H^2S^4A\ln(T)\iota.  
   \end{align*}
   
   which is same as the bound for $\term_{1a}$. Also, $\term_{2b}$ can be handled in the same way as $\term_{2a}$, and $\term_{2c}\leq \sum_{t=1}^T \sum_{s'}\mu^{\pi_t}(s')\times \frac{H\gamma_t}{\mu_t(s')}\leq HS\sum_{t=1}^T \gamma_t$. Overall, $\term_2$ can be bounded by the same order as $\term_1$.

   \paragraph{Upper bounding $\term_3$. }
   %We consider $\term_3$ restricted to $(s,a,s')\in\calT_h=(\calS_h\times \calA\times \calS_{h+1})$. 
   
   %\begin{align*}
    %   \term_3 &\leq H^2\sum_{t=1}^T \sum_{(s,a,s')\in\calT_h}\left[\mu^{\pi_t}(s,a) - \mu^{\pi}(s,a))\right]_+ \left(\sqrt{\frac{P(s'|s,a)\iota}{n_t(s,a)}} + \frac{\iota}{n_t(s,a)}\right) \\
    %   &\leq H^2\sum_{t=1}^T \sum_{(s,a,s')\in\calT_h}\left[\mu^{\pi_t}(s,a) - \mu^{\pi}(s,a))\right]_+ \left(P(s'|s,a)\alpha + \frac{\iota}{n_t(s,a) \alpha}\right)   \tag{for any $\alpha\in(0,1]$} \\
    %   &\leq \alpha H^2 \sum_{t=1}^T \sum_{(s,a)\in\calS_h\times \calA} \left[\mu^{\pi_t}(s,a) - \mu^{\pi}(s,a))\right]_+ + \frac{H^2|\calS_h|}{\alpha}\sum_{s,a} \frac{\mu^{\pi_t}(s,a)\iota}{n_t(s,a)} \\
    %   &\leq \alpha H^2 \sum_{t=1}^T \sum_{(s,a)\in\calS_h\times \calA} \left[\mu^{\pi_t}(s,a) - \mu^{\pi}(s,a))\right]_+ + \frac{H^2|\calS_h|^2A\ln(T)\iota}{\alpha} \tag{by \pref{lem: pigeon hole}} \\
    %   &= \sqrt{H^3S^2A  \sum_{t=1}^T  \sum_{s,a} \left[\mu^{\pi_t}(s,a) - \mu^{\pi}(s,a))\right]_+\ln(T)\iota} + H^2S^2 A\ln(T)\iota,  
   %\end{align*}
   
   \begin{align*}
       \term_3 &\leq \sum_{t=1}^T \sum_{s,a,s'}\left[\mu^{\pi_t}(s,a) - \mu^{\pi}(s,a)\right]_+ \left(\sqrt{\frac{\overline{P}_t(s'|s,a)\iota}{n_t(s,a)}} + \frac{\iota}{n_t(s,a)}\right) \sum_{s''} \mu^{\overline{P}_t,\pi_t}(s''|s')c_t(s'')\\
       &\lesssim \sum_{t=1}^T \sum_{s,a,s'}\left[\mu^{\pi_t}(s,a) - \mu^{\pi}(s,a)\right]_+ \left(\overline{P}_t(s'|s,a)\alpha + \frac{\iota}{n_t(s,a) \alpha}\right)\sum_{s''} \mu^{\overline{P}_t,\pi_t}(s''|s')c_t(s'')   \tag{for any $\alpha\in(0,1]$} \\
       &= \alpha\sum_{t=1}^T \sum_{s,a,s'}\left[\mu^{\pi_t}(s,a) - \mu^{\pi}(s,a)\right]_+ \overline{P}_t(s'|s,a)\sum_{s''} \mu^{\overline{P}_t,\pi_t}(s''|s')c_t(s'') \\
       &\qquad + \frac{1}{\alpha}\sum_{t=1}^T \sum_{s,a,s'}\left[\mu^{\pi_t}(s,a) - \mu^{\pi}(s,a)\right]_+   \frac{\iota}{n_t(s,a)}\sum_{s''} \mu^{\overline{P}_t,\pi_t}(s''|s')c_t(s'') \\
       &\leq \alpha\sum_{t=1}^T \sum_{s,a}\left[\mu^{\pi_t}(s,a) - \mu^{\pi}(s,a)\right]_+ \sum_{s'} \mu^{\overline{P}_t,\pi_t}(s'|s,a)c_t(s')  + \frac{H^2}{\alpha}\sum_{t=1}^T \sum_{s,a,s'}\left[\mu^{\pi_t}(s,a) - \mu^{\pi}(s,a))\right]_+   \frac{\iota}{n_t(s,a)} \\
       &= \alpha H\sum_{t=1}^T \sum_{s,a,s'}\left[\mu^{\pi_t}(s,a) - \mu^{\pi}(s,a)\right]_+  \mu^{\overline{P}_t,\pi_t}(s'|s,a)\frac{\mu_t(s')-\underline{\mu}_t(s')}{\mu_t(s')}  + \frac{H^2S}{\alpha}\sum_{t=1}^T \sum_{s,a}   \frac{\mu^{\pi_t}(s,a)\iota}{n_t(s,a)} \\
       &\leq \underbrace{\alpha H\sum_{t=1}^T \sum_{s,a,s'}\left[\mu^{\pi_t}(s,a) - \mu^{\pi}(s,a)\right]_+  \mu^{\overline{P}_t,\pi_t}(s'|s,a)\frac{\overline{\mu}_t^{\pi_t}(s')-\mu^{\pi_t}(s')}{\mu_t(s')}}_{\term_{3a}}\\
       &\qquad + \underbrace{\alpha H\sum_{t=1}^T \sum_{s,a,s'}\left[\mu^{\pi_t}(s,a) - \mu^{\pi}(s,a)\right]_+  \mu^{\overline{P}_t,\pi_t}(s'|s,a)\frac{\mu^{\pi_t}(s')-\underline{\mu}_t^{\pi_t}(s')}{\mu_t(s')}}_{\term_{3b}} \\
       &\qquad + \underbrace{\alpha H\sum_{t=1}^T \sum_{s,a,s'}\left[\mu^{\pi_t}(s,a) - \mu^{\pi}(s,a)\right]_+  \mu^{\overline{P}_t,\pi_t}(s'|s,a)\frac{\gamma_t}{\mu_t(s')}}_{\term_{3c}} + \frac{H^2S^2A\ln(T)\iota}{\alpha}  \tag{by \pref{lem: pigeon hole} and the assumption that $\calE$ holds.}
    \end{align*} 
    For $\term_{3a}$ we apply \pref{lem: complicated lemma} with 
    \begin{align*}
        g_t(s') = \frac{\sum_{s,a}[\mu^{\pi_t}(s,a)-\mu^\pi(s,a)]_+ \mu^{\overline{P}_t,\pi_t}(s'|s,a)}{\mu_t(s')} \leq \frac{\sum_{s,a}\mu^{\pi_t}(s,a) \mu^{\overline{P}_t,\pi_t}(s'|s,a)}{\mu_t(s')}  \leq H, 
    \end{align*}
    and we get 
    \begin{align*}
        \term_{3a}&\leq \alpha H \sqrt{H^2S^2A\ln(T)\iota\sum_{t=1}^T \sum_{s'}\mu^{\pi_t}(s') \frac{\sum_{s,a}[\mu^{\pi_t}(s,a)-\mu^\pi(s,a)]_+ \mu^{\overline{P}_t,\pi_t}(s'|s,a)}{\mu_t(s')} } + \alpha H\cdot H^2S^4A\ln(T)\iota \\
        &\leq \alpha H \sqrt{H^3S^2A \ln(T)\iota\sum_{t=1}^T \sum_{s,a} \left[\mu^{\pi_t}(s,a) - \mu^{\pi}(s,a)\right]_+ } + \alpha H^3S^4A\ln(T)\iota
    \end{align*}
    The same bound applies to $\term_{3b}$, too. 
    \begin{align*}
        \term_{3c} \leq \alpha H\sum_{t=1}^T \sum_{s,a,s'}\mu^{\pi_t}(s,a)\mu^{\overline{P}_t, \pi_t}(s'|s,a)\frac{\gamma_t}{\mu_t(s')} \leq \alpha H^2 \sum_{s'}\gamma_t \lesssim \alpha H^6SA^2. 
    \end{align*}
    Picking $\alpha=\frac{1}{H}$,  combining $\term_{3a}$ and $\term_{3b}$, and using $H\leq S$, we get 
    \begin{align*}
        \term_3 \leq \sqrt{H^3S^2A \ln(T)\iota\sum_{t=1}^T \sum_{s,a} \left[\mu^{\pi_t}(s,a) - \mu^{\pi}(s,a)\right]_+ } + H^2S^4A^2\ln(T)\iota  
    \end{align*}
   which is also of the same order as $\term_1$. 

   Combining $\term_1, \term_2, \term_3$, we get that if $\calE$ holds, then 
   \begin{align*}
       \sum_{t=1}^T \sum_{s,a}\left(\mu^{\pi_t}(s,a) - \mu^\pi(s,a)\right)z_t(s,a) \lesssim \sqrt{H^3S^2A  \sum_{t=1}^T  \sum_{s,a} \left[\mu^{\pi_t}(s,a) - \mu^{\pi}(s,a))\right]_+\ln(T)\iota} + H^2S^4 A^2\ln(T)\iota
   \end{align*}
   Using this in \eqref{eq: middle point} finishes the proof.

\end{proof}
\section{Bounding $\sum_s V^{\pi_t}(s_0;b_t)$ (\pref{lem: V(bt) known Tsallis}, \pref{lem: V(bt) known shannon}, \pref{lem: log barrier bt lemma})} \label{app: b_t part}

We first show \pref{lem: common lemma for final} and 
\pref{lem: nu term final} which are common among different regularizers.
\begin{lemma}\label{lem: common lemma for final}
    \begin{align*}
        &\E\left[\sum_{s,a} \sqrt{\sum_{t=1}^T \mu_t(s)\pi_t(a|s)(1-\pi_t(a|s))}\right]\\ 
        &\lesssim \E\left[\sum_{s,a} \sqrt{\sum_{t=1}^T \mu^{\pi_t}(s)\pi_t(s)(1-\pi_t(a|s))}\right] + \sqrt{H^4S^2A^3\ln (T)} + \ind\{\text{unknown transition}\}\sqrt{HS^5A^3\ln(T)\iota}.  
    \end{align*}    
\end{lemma}
\begin{proof}
    Define $\phi(s,a)=\pi_t(a|s)(1-\pi_t(a|s))$. 
    \begin{align*}
    \sum_{s,a}\sqrt{\sum_{t=1}^T \mu_t(s)\phi(s,a)} &\leq \sum_{s,a}\sqrt{\sum_{t=1}^T \mu^{\pi_t}(s)\phi(s,a)} +  \underbrace{\sum_{s,a}\sqrt{\sum_{t=1}^T \gamma_t\phi_t(s,a)}}_{\term_{1}} +\underbrace{\sum_{s,a}\sqrt{\sum_{t=1}^T \left|\overline{\mu}_t^{\pi_t}(s)-\mu^{\pi_t}(s)\right|\phi_t(s,a)}}_{\term_{2}}
\end{align*}

\begin{align*}
    \term_{1}
    &\leq  \sum_{s,a} \sqrt{\sum_{t=1}^T \gamma_t\phi_t(s,a)} 
    \leq  \sqrt{SA\sum_{t=1}^T \gamma_t \sum_{s,a}\phi_t(s,a)} \leq S \sqrt{A\sum_{t=1}^T \gamma_t } \leq  \sqrt{H^4S^2A^3\ln (T)}. 
\end{align*}

$\term_{2}$ is zero in the known transition case, and in the unknown transition case, if $\calE$ defined in \pref{def: good event} holds, then
\begin{align*}
    \term_2 &\leq   \sum_{s,a} \left(\alpha\sum_{t=1}^T(\overline{\mu}_t^{\pi_t}(s) - \mu^{\pi_t}(s))\phi_t(s,a) + \frac{1}{\alpha}\right) \tag{for any $\alpha>0$} \\
    &\leq \alpha  \left(\sqrt{HS^2A\ln(T)\iota \sum_{t=1}^T \sum_s \mu^{\pi_t}(s)\left(\sum_a \phi_t(s,a)\right)^2} + HS^4A\ln(T)\iota\right) + \frac{SA}{\alpha}   \tag{by \pref{lem: complicated lemma} with $g_t(s)=\sum_a \phi_t(s,a)$}\\
    &\leq \alpha  \left(\sqrt{HS^2A\ln(T)\iota \sum_{t=1}^T \sum_{s,a} \mu^{\pi_t}(s)\phi_t(s,a)} + HS^4A\ln(T)\iota\right) + \frac{SA}{\alpha} \\
    &\lesssim \sqrt{\sum_{t=1}^T \sum_{s,a} \mu^{\pi_t}(s)\phi_t(s,a)}  + \sqrt{HS^5A^3\ln(T)\iota}  \tag{choosing $\alpha=\frac{1}{\sqrt{HS^3A\ln(T)\iota}}$} \\
    &\leq \sum_{s,a} \sqrt{\sum_{t=1}^T  \mu^{\pi_t}(s)\phi_t(s,a)} +  \sqrt{HS^5A^3\ln(T)\iota}. 
\end{align*}
%Overall, 
%\begin{align*}
%    \term_2 
%    &\leq H\sum_{s,a} \sqrt{\ln(T)\sum_{t=1}^T \mu^{\pi_t}(s)\phi_t(s,a)} + H^3SA\ln (T) + \ind\{\text{unknown transition}\}HS^3A^2\ln(T)\sqrt{\iota}. 
%\end{align*}
If $\calE$ does not hold (which happens with probability $\leq O(H/T^3)$), then $\term_2\leq O(SA\sqrt{T})$. Overall, 
\begin{align*}
    \E[\term_2]\lesssim \E\left[\sum_{s,a} \sqrt{\sum_{t=1}^T  \mu^{\pi_t}(s)\phi_t(s,a)}\right] +  \sqrt{HS^5A^3\ln(T)\iota} + \frac{HSA}{T^{2.5}}.  
\end{align*}

Collecting terms and using $H\leq S$ finishes the proof. 
\end{proof}

\begin{lemma}\label{lem: nu term final}
    With known transition, 
    \begin{align*}
        \sum_{t=1}^T \sum_s \mu^{\pi_t}(s)\nu_t(s)\lesssim H^4SA^2\ln(T). 
    \end{align*}
    For Tsallis entropy or Shannon entropy with unknown transition, 
    \begin{align*}
        \E\left[\sum_{t=1}^T \sum_s \mu^{\tildeP_t, \pi_t}(s)\nu_t(s)\right] \leq HS^4A^2\ln(T)\iota.  
    \end{align*}
    
\end{lemma}
\begin{proof}
With known transition, we have 
\begin{align*}
    &\sum_{t=1}^T \sum_s \mu^{ \pi_t}(s)\nu_t(s) \\
    &\leq \frac{1}{H^4}\sum_{t=1}^T  \sum_{s,a} \mu^{ \pi_t}(s)\pi_t(a|s)C_t(s,a)^2 \tag{$\eta_t(s,a)\leq \frac{1}{H^4}$ or $\eta_t(s)\leq \frac{1}{H^4}$}  \\
    &\leq \frac{1}{H^2}\sum_{t=1}^T  \sum_{s,a} \mu^{ \pi_t}(s)\pi_t(a|s)C_t(s,a) \tag{$C_t(s,a)\leq H^2$} \\
    &\leq \frac{1}{H}\sum_{t=1}^T \sum_{s,a} \mu^{ \pi_t}(s)\pi_t(a|s)\sum_{s'} \mu^{\pi_t}(s'|s,a)\frac{\mu_t(s')-\mu^{\pi_t}(s')}{\mu_t(s')}  \tag{by the definition of $C_t(s,a)$} \\
    &\leq  \sum_{t=1}^T \sum_{s'} \mu^{\pi_t}(s')\times \frac{\mu_t(s')-\mu^{\pi_t}(s')}{\mu_t(s')} \\
    &\leq \sum_{t=1}^T S\gamma_t \\
    &\lesssim H^4SA^2\ln(T). 
\end{align*}

With unknown transitions, notice that for Tsallis entropy we have $\eta_t(s)\leq \min\left\{\frac{1}{H^4}, \frac{1}{H\sqrt{t}}\right\}$ and for Shannon entropy we have $\eta_t(s,a)\leq \min\left\{\frac{1}{H^4}, \frac{1}{H\sqrt{t}}\right\}$. Therefore, in both cases, suppose that $\calE$ holds, 
    \begin{align*}
    &\sum_{t=1}^T \sum_s \mu^{\tildeP_t, \pi_t}(s)\nu_t(s) \\
    &\leq \sum_{t=1}^T  \min\left\{\frac{1}{H^4}, \frac{1}{H\sqrt{t}}\right\}\sum_{s,a}  \mu^{\tildeP_t, \pi_t}(s)\pi_t(a|s)C_t(s,a)^2 \\
    &\leq \sum_{t=1}^T \min\left\{\frac{1}{H^2}, \frac{H}{\sqrt{t}}\right\} \sum_{s,a} \mu^{\tildeP_t, \pi_t}(s)\pi_t(a|s)C_t(s,a) \\
    &\leq \sum_{t=1}^T\min\left\{\frac{1}{H}, \frac{H^2}{\sqrt{t}}\right\} \sum_{s,a} \mu^{\tildeP_t, \pi_t}(s)\pi_t(a|s)\sum_{s'} \mu^{\overline{P}_t, \pi_t}(s'|s,a)\frac{\mu_t(s')-\underline{\mu}_t^{\pi_t}(s')}{\mu_t(s')}   \tag{let $\overline{P}_t$ be the $\tildeP$ attaining maximum in \eqref{eq: C def} }\\
    &\leq \sum_{t=1}^T \min\left\{1, \frac{H^3}{\sqrt{t}}\right\}\sum_{s'} \overline{\mu}_t^{\pi_t}(s')\times \frac{\mu_t(s')-\underline{\mu}_t^{\pi_t}(s')}{\mu_t(s')} \\
    &\leq \sum_{t=1}^T \min\left\{1, \frac{H^3}{\sqrt{t}}\right\}\sum_{s'} \left(\overline{\mu}_t^{\pi_t}(s')-\underline{\mu}_t^{\pi_t}(s')\right) + \sum_{t=1}^T S\gamma_t \\
    &\leq \sum_{t=1}^T \min\left\{1, \frac{H^3}{\sqrt{t}}\right\}\sum_{s'} \left(\overline{\mu}_t^{\pi_t}(s')-\underline{\mu}_t^{\pi_t}(s')\right) + H^4SA^2\ln(T)
\end{align*}
By \pref{lem: complicated lemma}, the first part above can be upper bounded by  
\begin{align*}
    &\sqrt{HS^2 A\ln(T)\iota \sum_{t=1}^T \sum_s \mu^{\pi_t}(s) \times \min\left\{1, \frac{H^6}{t}\right\} } + HS^4A\ln(T)\iota \\
    &\lesssim  \sqrt{H^8S^2A\iota}\ln(T) + HS^4A\ln(T)\iota \lesssim HS^4A\ln(T)\iota 
\end{align*}
where we use $H\leq S$. 

Suppose that $\calE$ does not hold (happens with probability $O(H/T^3)$), we still have 
\begin{align*}
    \sum_{t=1}^T \sum_s \mu^{\tildeP_t,\pi_t}(s)\nu_t(s) &\leq \sum_{t=1}^T  \min\left\{\frac{1}{H^4}, \frac{1}{H\sqrt{t}}\right\}\sum_{s,a}  \mu^{\tildeP_t, \pi_t}(s)\pi_t(a|s)C_t(s,a)^2 \\
    &\leq O\left(T\times\frac{1}{H^4}\times H \left(H^2\right)^2\right)\leq O(HT)
\end{align*}
because $|C_t(s,a)|\leq H^2$ with probability $1$.  

Combining all terms and taking expectation, we conclude that 
\begin{align*}
    \E\left[\sum_{t=1}^T \sum_s \mu^{ \tildeP_t,\pi_t}(s)\nu_t(s)\right] \lesssim HS^4A^2\ln(T)\iota. 
\end{align*}

\end{proof}

\subsection{Tsallis entropy}
\begin{proof}[Proof of \pref{lem: V(bt) known Tsallis}]
     \begin{align*}
    &\sum_{t=1}^T V^{\pi_t, \tildeP_t}(s_0; b_t) \\ 
    &=\sum_{t=1}^T \sum_s \mu^{\tildeP_t, \pi_t}(s)b_t(s) \\
    &\leq \sum_{t=1}^T \sum_s \mu^{\tildeP_t, \pi_t}(s)\left[ \nu_t(s) + \left(\frac{1}{\eta_t(s)} - \frac{1}{\eta_{t-1}(s)}\right)\left(\xi_t(s) + \sqrt{A}\cdot\ind\left[\frac{\eta_t(s)}{\mu_t(s)} >\frac{1}{8H}\right]\right)\right] \tag{by \eqref{eq: b_t Tsallis}}  \\
    &\lesssim \sum_{t=1}^T \sum_s \mu^{\tildeP_t, \pi_t}(s)\nu_t(s) + H\sum_{t=1}^T \sum_s \mu^{\tildeP_t, \pi_t}(s)\frac{\frac{1}{\mu_t(s)}}{\sqrt{\sum_{\tau=1}^t \frac{1}{\mu_\tau(s)}}} \left(\xi_t(s)  
    + \sqrt{A}\cdot\frac{8H\eta_t(s)}{\mu_t(s)}\right) \tag{by \eqref{eq: Tsallis eta}}\\
    &\lesssim \underbrace{\sum_{t=1}^T \sum_s \mu^{\tildeP_t, \pi_t}(s)\nu_t(s)}_{\term_1} + \underbrace{H\sum_{t=1}^T \sum_s \frac{\xi_t(s)}{\sqrt{\sum_{\tau=1}^t \frac{1}{\mu_\tau(s)}}}}_{\term_2} + \underbrace{H^2\sqrt{A}\sum_{t=1}^T \sum_s \frac{\frac{1}{\mu_t(s)}\cdot\eta_t(s)}{\sqrt{\sum_{\tau=1}^t \frac{1}{\mu_\tau(s)}}}}_{\term_3} \\
\end{align*}
\paragraph{Bounding $\term_1$. }
$\term_1$ can be bounded using \pref{lem: nu term final}, which gives 
\begin{align*}
    \E[\term_1] \lesssim H^4SA^2\ln(T) + \ind\{\text{unknown transition}\}HS^4A^2 \ln(T)\iota. 
\end{align*}

\paragraph{Bounding $\term_2$. }
\begin{align*}
    \term_2 &\leq  H\sum_{t=1}^T \sum_s \sqrt{\mu_t(s)}\xi_t(s)\sqrt{\frac{\frac{1}{\mu_t(s)}}{\sum_{\tau=1}^t\frac{1}{\mu_\tau(s)}}} \\
    &\leq  H\sum_{t=1}^T \sum_{s,a} \sqrt{\mu_t(s)\pi_t(a|s)}(1-\pi_t(a|s))\sqrt{\frac{\frac{1}{\mu_t(s)}}{\sum_{\tau=1}^t\frac{1}{\mu_\tau(s)}}} \\
    &\leq  H\sum_{s,a} \sqrt{\sum_{t=1}^T \mu_t(s)\pi_t(s,a)(1-\pi_t(a|s))}\sqrt{\sum_{t=1}^T \frac{\frac{1}{\mu_t(s)}}{\sum_{\tau=1}^t\frac{1}{\mu_\tau(s)}}}\\
    &\lesssim  H\sqrt{\ln T}\sum_{s,a} \sqrt{\sum_{t=1}^T \mu_t(s)\pi_t(a|s)(1-\pi_t(a|s))}. 
\end{align*}
By \pref{lem: common lemma for final}, we can bound the last expression by 
\begin{align*}
    \E\left[H\sum_{s,a} \sqrt{\ln(T)\sum_{t=1}^T \mu^{\pi_t}(s)\pi_t(s,a)(1-\pi_t(s,a))}\right] + \sqrt{H^6S^2A^3}\ln (T) + \ind\{\text{unknown transition}\}\sqrt{H^3S^5A^3}\ln(T)\iota. 
\end{align*}

\paragraph{Bounding $\term_3$. } By \eqref{eq: Tsallis eta}, 
\begin{align*}
    \term_3 &\leq  H\sqrt{A}\sum_{t=1}^T \sum_s \frac{\frac{1}{\mu_t(s)}}{\sum_{\tau=1}^t \frac{1}{\mu_\tau(s)}}\leq HS\sqrt{A}\ln(T). 
 \end{align*}
 
 Combining $\term_1, \term_2, \term_3$ finishes the proof. 

\end{proof}

\subsection{Shannon entropy}

\begin{proof}[Proof of \pref{lem: V(bt) known shannon}]
\begin{align*}
    &\sum_{t=1}^T V^{\tildeP_t, \pi_t}(s_0; b_t) \\
    &\lesssim H\sqrt{\ln T}\sum_{t,s,a} \mu^{\tildeP_t, \pi_t}(s) \left( \frac{1}{\mu_t(s)\sqrt{\sum_{\tau=1}^{t-1} \frac{\xi_\tau(s,a)}{\mu_\tau(s)} + \frac{1}{\mu_t(s)} }} + \frac{1}{\sqrt{t}} \right)\left(\xi_t(s,a)+1-\frac{\min_{\tau\in[t]}\mu_\tau(s)}{\min_{\tau\in[t-1]}\mu_\tau(s)}\right) \\
    &\qquad + \sum_{t,s}\mu^{\tildeP_t, \pi_t}(s) \nu_t(s) \\
    &\leq \sum_{t,s,a} \frac{H\sqrt{\ln T}}{\sqrt{\sum_{\tau=1}^{t-1} \frac{\xi_\tau(s,a)}{\mu_\tau(s)} + \frac{1}{\mu_t(s)} }} \xi_t(s,a)+ \sum_{t,s,a} \frac{H\sqrt{\ln T}}{\sqrt{\sum_{\tau=1}^{t-1} \frac{\xi_\tau(s,a)}{\mu_\tau(s)} + \frac{1}{\mu_t(s)} }}\left(1-\frac{\min_{\tau\in[t]}\mu_\tau(s)}{\min_{\tau\in[t-1]}\mu_\tau(s)}\right) \\
    &\qquad + \sum_{t,s,a} \frac{H\sqrt{\ln T}\mu_t(s)}{\sqrt{t}}\xi_t(s,a) + H\sqrt{\ln T}\sqrt{\sum_{t,s,a} \mu^{\tildeP_t, \pi_t}(s) \frac{1}{t} }\sqrt{\sum_{t,s,a} \mu^{\tildeP_t, \pi_t}(s)\left(1-\frac{\min_{\tau\in[t]}\mu_\tau(s)}{\min_{\tau\in[t-1]}\mu_\tau(s)}\right)^2 }  \\
    &\qquad  + \sum_{t,s} \mu^{\tildeP_t, \pi_t}(s) \nu_t(s) \\
    &\leq H\sqrt{\ln T}\sum_{s,a} \sqrt{\sum_t\frac{\frac{\xi_t(s,a)}{\mu_t(s)}}{\sum_{\tau=1}^{t-1} \frac{\xi_\tau(s,a)}{\mu_\tau(s)} + \frac{1}{\mu_t(s)}}}\sqrt{\sum_t \mu_t(s)\xi_t(s,a)} + H\sqrt{\ln T}\sum_{t,s,a}\ln\left(\frac{\min_{\tau\in[t-1]}\mu_\tau(s)}{\min_{\tau\in[t]}\mu_\tau(s)}\right) \\
    &\qquad + H\sqrt{\ln T}\sum_{s,a}\sqrt{\sum_t \frac{\mu_t(s)\xi_t(s,a)}{t}}\sqrt{\sum_t \mu_t(s)\xi_t(s,a)} \\
    &\qquad + H\sqrt{\ln T}\sqrt{HA\ln(T)}\sqrt{A\sum_{t,s}\ln\left(\frac{\min_{\tau\in[t-1]}\mu_\tau(s)}{\min_{\tau\in[t]}\mu_\tau(s)}\right)}  + \sum_{t,s} \mu^{\tildeP_t, \pi_t}(s) \nu_t(s) \\
    &\lesssim H\sqrt{\ln T}\sum_{s,a}\sqrt{\ln(T)\sum_t \mu_t(s)\xi_t(s,a)} + \sum_{t,s}\mu^{\tildeP_t, \pi_t}(s) \nu_t(s) + H^{2}SA\ln^{\frac{3}{2}}(T)\\
    &\lesssim H\sum_{s,a}\sqrt{\ln^3(T)\sum_{t=1}^T \mu_t(s)\pi_t(a|s)(1-\pi_t(a|s)) } +  \sum_{s,t}\mu^{\tildeP_t, \pi_t}(s) \nu_t(s) + H^{2}SA\ln^{\frac{3}{2}}(T)
\end{align*}
By \pref{lem: common lemma for final} and \pref{lem: nu term final}, the expectation of this can be upper bounded by 
\begin{align*}
    &\E\left[H\sum_{s,a} \sqrt{\ln^3(T)\sum_{t=1}^T \mu^{\pi_t}(s)\pi_t(a|s)(1-\pi_t(a|s))}\right] \\
    &\qquad + \sqrt{H^6S^2A^3\ln^3(T)} + \ind\{\text{unknown transition}\}\sqrt{H^3S^5A^3\ln^3(T)\iota} \\
    &\qquad + H^4SA^2\ln^{\frac{3}{2}}(T) + \ind\{\text{unknown transition}\}HS^4A^2\ln(T)\iota \\
    &\lesssim \E\left[H\sum_{s,a} \sqrt{\ln^3(T)\sum_{t=1}^T \mu^{\pi_t}(s)\pi_t(a|s)(1-\pi_t(a|s))}\right] \\
    &\qquad + H^4SA^2\sqrt{\ln^{3}(T)} + \ind\{\text{unknown transition}\}HS^4A^2\ln(T)\iota. \tag{using $H\leq S$ and $\log(T)\lesssim \iota$} 
\end{align*}
\end{proof}

\subsection{Log barrier}
\begin{lemma}\label{lem: log barrier ito lemma}
    Let $\eta_1>0, \eta_2, \eta_3, \ldots$ be updated by 
    \begin{align*}
        \frac{1}{\eta_{t+1}} = \frac{1}{\eta_t} + \eta_t \phi_t\qquad \forall t\geq 1
    \end{align*}
    with $0\leq \phi_t\leq \eta_t^{-2}$. 
    Then 
    \begin{align*}
        \frac{1}{\eta_{t+1}} \geq \frac{1}{2}\sqrt{\sum_{\tau=1}^{t+1} \phi_\tau}. 
    \end{align*}
    
\end{lemma}

\begin{proof}
    %By the condition on $\phi_t$, we have $\frac{1}{\eta_{t+1}}=\frac{1}{\eta_t}+\eta_t\phi_t\leq \frac{2}{\eta_t}$. 
    By the update rule,  
    \begin{align*}
        \frac{1}{\eta_{t+1}^2} - \frac{1}{\eta_t^2} =\left(\frac{1}{\eta_{t+1}}+\frac{1}{\eta_t}\right)\left(\frac{1}{\eta_{t+1}}-\frac{1}{\eta_t}\right) =  \left(\frac{1}{\eta_{t+1}}+\frac{1}{\eta_t}\right)\eta_t\phi_t \geq \phi_t, % = \phi_t\left(1+\frac{\eta_t}{\eta_{t+1}}\right)\leq 3\phi_t
    \end{align*} 
    which implies
    \begin{align*}
        \frac{1}{\eta_{t+1}}\geq \sqrt{\frac{1}{\eta_1^2}+\sum_{\tau=1}^t\phi_\tau} \geq  \sqrt{\sum_{\tau=1}^t \phi_\tau}. 
    \end{align*}
    
    By the condition on $\phi_t$, we also have 
    \begin{align*}
        \frac{1}{\eta_{t+1}}\geq \sqrt{\phi_{t+1}}. 
    \end{align*}
    Combining the two inequalities finishes the proof. 
\end{proof}

\begin{proof}[Proof of \pref{lem: log barrier bt lemma}]
    In this proof we only focus on the know transition case. We use $\calT_r$ and $\calT_v$ to denote the set of real and virtual episodes, respectively. 
    
    Let $\phi_t(s,a)=\frac{4\zeta_t(s,a)}{\mu_t(s)^2\log(T)}$ in real episodes and $\phi_t(s,a)=\frac{\ind\{(s_t^\dagger,a_t^\dagger)=(s,a)\}}{24\eta_t(s,a)^2H\log T}$ in virtual episodes. We first show that $\phi_t(s,a)\leq \frac{1}{\eta_t(s,a)^2}$, which allows us to apply \pref{lem: log barrier ito lemma} because $\frac{1}{\eta_{t+1}(s,a)}=\frac{1}{\eta_{t}(s,a)}+\eta_t(s,a)\phi_t(s,a)$ by our update rule. This is clear for virtual episodes. For real episodes, 
    \begin{align*}
        \phi_t(s,a)\eta_t(s,a)^2 = \frac{4\eta_t(s,a)^2\zeta_t(s,a)}{\mu_t(s)^2\log T}\leq \frac{H^2}{\log T}\times \frac{1}{H^3S}\leq 1
    \end{align*}
    because $\frac{\eta_t(s,a)}{\mu_t(s)}\leq \frac{1}{60\sqrt{H^3S}}$ in real episodes. 
    
    \begin{align*}
        &\sum_{t=1}^T  V^{\pi_t}(s_0;b_t)\\
        &\lesssim\sum_{t\in\calT_r} \sum_{s} \mu^{\pi_t}(s)\sum_a \left(\frac{\eta_t(s,a)\zeta_t(s,a)}{\mu_t(s)^2\log(T)}\log(T)\right) + \sum_{t\in\calT_v} \mu^{\pi_t}(s_t^\dagger) \frac{1}{\eta_t(s_t^\dagger, a_t^\dagger)H\log T}\log T + \sum_{t=1}^T \sum_s \mu^{\pi_t}(s)\nu_t(s)\\
        &\lesssim \sum_{t\in\calT_r}\sum_{s,a}\eta_t(s,a)\frac{\zeta_t(s,a)}{\mu_t(s)} + \frac{\sqrt{H^3S}}{H}|\calT_v| + H^4SA\ln(T)  \tag{in virtual episodes, $\frac{\eta_t(s_t^\dagger, a_t^\dagger)}{\mu_t(s_t^\dagger)}\geq \frac{1}{\sqrt{H^3S}}$, and we use \pref{lem: nu term final} to bound the last term} \\
        &\leq \sqrt{\log(T)}\sum_{t\in\calT_r}\sum_{s,a} \frac{\frac{\zeta_t(s,a)}{\mu_t(s)}}{\sqrt{\sum_{\tau\leq t: \tau\in\calT_r} \frac{\zeta_\tau(s,a)}{\mu_\tau(s)^2} }} + \sqrt{HS}|\calT_v| + + H^4SA^2\ln(T)  \tag{by \pref{lem: log barrier ito lemma} and the condition verified at the beginning of the proof}\\
        &\leq \sqrt{\log T}\sum_{s,a}\sqrt{\sum_{t\in\calT_r}\frac{\frac{\zeta_t(s,a)}{\mu_t(s)^2}}{ \sum_{\tau\leq t: \tau\in\calT_r} \frac{\zeta_\tau(s,a)}{\mu_\tau(s)^2}}}\sqrt{\sum_{t\in\calT_r} \zeta_t(s,a)} + \sqrt{HS}|\calT_v| + + H^4SA^2\ln(T) \\
        &\leq \log(T)\sum_{s,a}\sqrt{\sum_{t\in\calT_r} \zeta_t(s,a)} + \sqrt{HS}|\calT_v| + H^4SA^2\ln(T).  
    \end{align*}
    Now we bound the number of virtual episodes. Notice that each time a virtual episode happens, there exist $s,a$ such that $\frac{\eta_t(s,a)}{\mu_t(s)}\geq \frac{1}{60\sqrt{H^3S}}$, and  $\eta_t(s,a)$ will shrink by a factor of $(1+\frac{1}{24H\log T})$ after the virtual episode.  Since $\mu_t(s)\geq \gamma_t$, this event cannot happen if $\eta_t(s,a)\leq \frac{\gamma_t}{60\sqrt{H^3S}}$. Thus, the number of virtual episodes is upper bounded by
    \begin{align*}
        |\calT_v|\lesssim SA\times \frac{\log \frac{60\sqrt{H^3S}}{\gamma_t}}{\log\left(1+\frac{1}{24H\log T}\right)} \lesssim HSA \ln(T)\ln(SAT). 
    \end{align*}
    Applying this bound in the last expression and using $H\leq S$ finishes the proof. 
\end{proof}

\section{Final Regret Bounds through Self-Bounding (\pref{thm: known main lemma}, \pref{thm: unknown}, \pref{thm: small loss})}\label{app: final reg}
\begin{proof}[Proof of \pref{thm: known main lemma}]
    Let $\picirc=\argmax_{\pi}\Reg(\pi)$. 
    By \eqref{eq: final regret tmp}, \pref{lem: biasterm known}, and \pref{lem: V(bt) known Tsallis}, under known transition and Tsallis entropy, we have 
    \begin{align*}
        \Reg(\picirc)\lesssim H\sum_{s,a} \sqrt{\E\left[\sum_{t=1}^T \mu^{\pi_t}(s)\pi_t(a|s)(1-\pi_t(a|s))\right]\ln(T)} + H^5SA^2\ln(T) 
    \end{align*}
    For the adversarial regime, we bound the above by 
    \begin{align*}
        H \sqrt{SA\E\left[\sum_{t=1}^T \sum_{s,a}\mu^{\pi_t}(s)\pi_t(a|s))\right]\ln(T)} + H^5SA\ln(T) = \sqrt{H^3SAT} + H^5SA^2\ln(T).  
    \end{align*}
    For the stochastic regime, notice that $\Reg(\picirc)\geq \Reg(\pi^\star)\geq \E\left[\sum_{t=1}^T \mu^{\pi_t}(s)\pi_t(a|s)\Delta(s,a) \right] - \calC$, and we have  
    \begin{align*}
        \Reg(\picirc) 
        &\leq  c_1 H\sum_{s}\sum_{a\neq \pi^\star(s)} \sqrt{\E\left[\sum_{t=1}^T \mu^{\pi_t}(s)\pi_t(a|s)\right]\ln(T)} + c_2 H^5SA^2\ln(T)  \tag{for some universal constants $c_1, c_2$} \\ 
        &\leq H\sum_{s}\sum_{a\neq \pi^\star(s)}\left(\frac{\alpha}{H}\E\left[\sum_{t=1}^T \mu^{\pi_t}(s)\pi_t(a|s)\Delta(s,a) \right] + \frac{ c_1^2 H\ln(T)}{\alpha\Delta(s,a)}\right) + c_2 H^5SA^2\ln(T)   \tag{for arbitrary $\alpha>0$}\\
        &\leq \alpha \E\left[\sum_{t=1}^T \mu^{\pi_t}(s)\pi_t(a|s)\Delta(s,a) \right] + O\left(\sum_{s}\sum_{a\neq \pi^\star(s)}\frac{ H^2\ln(T)}{\alpha\Delta(s,a)} + H^5SA^2\ln(T)\right) \\
        &\leq \alpha(\Reg(\picirc) + \calC)  + O\left(\sum_{s}\sum_{a\neq \pi^\star(s)}\frac{ H^2\ln(T)}{\alpha\Delta(s,a)} + H^5SA^2\ln(T)\right)
    \end{align*}
    Picking $\alpha=\min\left\{\frac{1}{2}, \calC^{-\frac{1}{2}}\left(\frac{H^2\ln(T)}{\Delta(s,a)}\right)^{\frac{1}{2}}\right\}$ leads to the bound
    \begin{align*}
        \Reg(\picirc)\lesssim U + \sqrt{U\calC} + H^5SA^2\ln(T)
    \end{align*}
    where $U=\sum_{s}\sum_{a\neq \pi^\star(s)}\frac{H^2\ln(T)}{\Delta(s,a)}$. Finally, using that $\Reg(\pi)\leq \Reg(\picirc)$ for all $\pi$ finishes the proof. 
\end{proof}

\begin{proof}[Proof of \pref{thm: unknown}]
    By \eqref{eq: final regret tmp}, \pref{lem: biasterm known}, and \pref{lem: V(bt) known Tsallis}, under unknown transition and Tsallis entropy, we have 
    \begin{align*}
        \Reg(\pi)&\leq \underbrace{c_1\sqrt{H^3S^2A\E\left[\sum_{t=1}^T  \sum_{s,a} \left[\mu^{\pi_t}(s,a) - \mu^{\pi}(s,a)\right]_+ \right]\ln(T)\iota}}_{\term_1}\\
        &\qquad + \underbrace{c_2H\sum_{s,a} \sqrt{\E\left[\sum_{t=1}^T \mu^{\pi_t}(s)\pi_t(a|s)(1-\pi_t(a|s))\right]\ln(T)\iota}}_{\term_2} + c_3 H^2S^4A^2\ln(T)\iota  \tag{for universal constants $c_1, c_2, c_3$} 
    \end{align*}
    In the adversarial regime, we can bound it by the order of
    \begin{align*}
        \sqrt{H^4S^2A T\ln(T)\iota} + H^2S^4A^2\ln(T)\iota
    \end{align*}
    To get a bound in the stochastic regime, we first argue that it suffices to show the desired bound for all $\pi$ that satisfies $\Reg(\pi)\geq \Reg(\pi^\star)$. This is because we can then bound $\Reg(\pi)$ for $\pi$ such that $\Reg(\pi)< \Reg(\pi^\star)$ by 
    \begin{align*}
        \Reg(\pi) < \Reg(\pi^\star) \lesssim U+\sqrt{U(\calC+\calC(\pi^\star))} + \poly(H,S,A)\ln(T)\iota = U+\sqrt{U\calC} + \poly(H,S,A)\ln(T)\iota
    \end{align*}
    because $\calC(\pi^\star)=0$ by definition. 
    
    Below we assume that $\Reg(\pi)\geq \Reg(\pi^\star)$. 
    Note that by \pref{lem: policy difference absolute}, for any $\pi$, 
   \begin{align*}
       \sum_{s,a} \left| \mu^{\pi}(s,a) - \mu^{\pi^\star}(s,a) \right| 
       &\leq H  \sum_{s,a}\mu^{\pi}(s)\left| \pi(a|s) - \pi^\star(a|s) \right| \\
       &= H  \sum_{s}\sum_{a\neq \pi^\star(s)}\mu^{\pi}(s)\pi(a|s) +  H \sum_s \mu^{\pi}(s)(1-\pi(~\pi^\star(s)~|s))\\ 
       &= 2H \sum_{s}\sum_{a\neq \pi^\star(s)} \mu^{\pi}(s) \pi(a|s). %\label{eq: same calculat}
   \end{align*}
   Hence, 
   \begin{align*}
       \term_1&\leq  c_1\sqrt{H^3S^2A\E\left[\sum_{t=1}^T  \sum_{s,a} \left|\mu^{\pi_t}(s,a) - \mu^{\pi}(s,a)\right| \right]\ln(T)\iota} \\ 
       &\leq  c_1\sqrt{H^3S^2A\E\left[\sum_{t=1}^T  \sum_{s,a} \left|\mu^{\pi_t}(s,a) - \mu^{\pi^\star}(s,a)\right| \right]\ln(T)\iota} + c_1\sqrt{H^3S^2A\sum_{t=1}^T  \sum_{s,a} \left|\mu^{\pi}(s,a) - \mu^{\pi^\star}(s,a)\right| \ln(T)\iota}\\ 
       &\leq c_1\sqrt{2H^4S^2A\E\left[\sum_{t=1}^T \sum_{s}\sum_{a\neq \pi^\star(s)}\mu^{\pi_t}(s,a)\right] \ln(T)\iota} + c_1\sqrt{2H^4S^2A\sum_{t=1}^T \sum_{s}\sum_{a\neq \pi^\star(s)}\mu^{\pi}(s,a) \ln(T)\iota}\iota \\
       &\leq \alpha \E\left[\sum_{t=1}^T \sum_{s}\sum_{a\neq \pi^\star(s)} \mu^{\pi_t}(s,a)\Delta_{\min}\right] +\alpha \sum_{t=1}^T \sum_{s}\sum_{a\neq \pi^\star(s)} \mu^{\pi}(s,a)\Delta_{\min} +  O\left(\frac{H^4S^2A\ln(T)\iota}{\alpha\Delta_{\min}}\right) \tag{by AM-GM}\\
       &\leq \alpha (\Reg(\pi^\star) + \calC) + \alpha(\Reg(\pi^\star) - \Reg(\pi) + \calC(\pi)) + O\left(  \frac{H^4S^2A\ln(T)\iota}{\alpha\Delta_{\min}}\right) \tag{see explanation below}\\
       &\leq \alpha\Reg(\pi) + \alpha(\calC+\calC(\pi)) + O\left(\frac{H^4S^2A\ln(T)\iota}{\alpha\Delta_{\min}}\right)  \tag{by the assumption $\Reg(\pi^\star)\leq \Reg(\pi)$}
   \end{align*}
   where in the second-to-last inequality we use the property: 
   \begin{align*}
       \Reg(\pi^\star) - \Reg(\pi) &= \E\left[\sum_{t=1}^T V^{\pi}(s_0;\ell_t) - V^{\pi^\star}(s_0;\ell_t) \right] \\
       &= \sum_{t=1}^T \sum_{s}\sum_{a\neq \pi^\star(s)}\mu^\pi(s,a)\Delta(s,a) - \sum_{t=1}^T \lambda_t(\pi) \\
       &\geq \sum_{t=1}^T \sum_{s}\sum_{a\neq \pi^\star(s)}\mu^\pi(s,a)\Delta_{\min} - \calC(\pi)
   \end{align*}
   For $\term_2$, similar to before, 
   \begin{align*}
       \term_2 
       &\leq c_2 H\sum_{s,a} \sqrt{\E\left[\sum_{t=1}^T \mu^{\pi_t}(s)\pi_t(a|s)(1-\pi_t(a|s))\right]\ln(T)} \\
       &\leq \alpha(\Reg(\pi) + \calC)  + O\left(\sum_{s}\sum_{a\neq \pi^\star(s)}\frac{ H^2\ln(T)}{\alpha\Delta(s,a)} + H^5SA^2\ln(T)\right) \\
       &\leq \alpha(\Reg(\pi) + \calC)  + O\left(\sum_{s}\sum_{a\neq \pi^\star(s)}\frac{ H^2\ln(T)}{\alpha\Delta(s,a)} + H^5SA^2\ln(T)\right)\\
       &\leq \alpha(\Reg(\pi) + \calC)  + O\left(\frac{ H^4S^2A\ln(T)}{\alpha\Delta_{\min}} + H^2S^4A^2\ln(T)\right)
   \end{align*}
   Combining $\term_1$ and $\term_2$, we get 
   \begin{align*}
       \Reg(\pi)\leq 2\alpha \Reg(\pi) + 2\alpha (\calC+\calC(\pi)) + O\left(\frac{H^4S^2A\ln(T)\iota}{\alpha\Delta_{\min}} + H^2S^4A^2\ln(T)\iota\right)
   \end{align*}
   Picking $\alpha=\min\left\{\frac{1}{4}, (\calC+\calC(\pi))^{-\frac{1}{2}}\left(\frac{H^4S^2A\ln(T)\iota}{\Delta_{\min}}\right)^{\frac{1}{2}}\right\}$ leads to the desired bound. 
\end{proof}

\begin{proof}[Proof of \pref{thm: small loss}]  
    \begin{align*}
        \Reg(\pi)\lesssim \sum_{s,a}\sqrt{\ln^2(T)\E\left[\sum_{t=1}^T (\ind_t(s,a)-\pi_t(a|s)\ind_t(s))^2L_{t,h(s)}^2\right]} + H^3S^2A^2\ln(T)\ln(SAT) 
    \end{align*}
    In the adversarial regime, 
    \begin{align*}
        \Reg(\pi)
        &\leq \sqrt{HSA\ln^2(T)\E\left[\sum_{t=1}^T \sum_{s,a} (\ind_t(s,a)-\pi_t(a|s)\ind_t(s))^2L_{t,h(s)}\right]} + H^3S^2A^2\ln(T)\ln(SAT) \\
        &\leq \sqrt{HSA\ln^2(T)\E\left[\sum_{t=1}^T \sum_{s,a} \ind_t(s,a)L_{t,h(s)}\right]} + H^3S^2A^2\ln(T)\ln(SAT) \\
        &\leq \sqrt{H^2SA\ln^2(T)\E\left[\sum_{t=1}^T V^{\pi_t}(s_0;\ell_t)\right]} + H^3S^2A^2\ln(T)\ln(SAT)  
    \end{align*}
    On the other hand, $\Reg(\pi)=\E\left[\sum_{t=1}^T V^{\pi_t}(s_0;\ell_t) - \sum_{t=1}^T V^{\pi}(s_0;\ell_t) \right]$. Solving the inequality, we get 
    \begin{align*}
        \Reg(\pi) \lesssim \sqrt{H^2SA\ln^2(T)\sum_{t=1}^T V^{\pi}(s_0;\ell_t)} + H^3S^2A^2\ln(T)\ln(SAT).    
    \end{align*}
    
    In the stochastic regime, 
    \begin{align*}
        \Reg(\pi)
        &\lesssim \sum_{s,a}\sqrt{\ln^2(T)\E\left[\sum_{t=1}^T (\ind_t(s,a)-\pi_t(a|s)\ind_t(s))^2L_{t,h(s)}^2\right]} + H^3S^2A^2\ln(T)\ln(SAT) \\
        &\leq \sum_{s,a}\sqrt{H^2\ln^2(T)\E\left[\sum_{t=1}^T \mu^{\pi_t}(s)\pi_t(a|s)(1-\pi_t(a|s))\right]} + H^3S^2A^2\ln(T)\ln(SAT),  
    \end{align*}
    which is similar to the stochastic bound in \pref{thm: known main lemma}. Following the same self-bounding analysis in the proof of \pref{thm: known main lemma} we can get the desired bound. 
\end{proof}

To get regret bounds for the Shannon entropy version under known and unknown transitions, we use  \pref{lem: biasterm known} and \pref{lem: V(bt) known shannon} and follow exactly the same procedure as in the proofs of \pref{thm: known main lemma} and \pref{thm: unknown}. This leads to the following guarantees: 
\begin{theorem}\label{thm: known main lemma shannon}
    Under known transitions, \pref{alg: template} with Shannon entropy regularizer ensures for any $\pi$ 
    \begin{align*}
        \Reg(\pi) \lesssim \sqrt{H^3SAT\ln^3 (T)} +\poly(H,S,A)\ln^2(T)
    \end{align*}
    in the adversarial case, 
    and 
    \begin{align*}
        &\Reg(\pi) \lesssim U+\sqrt{U\calC} + \poly(H,S,A)\ln^2(T)
    \end{align*}
    in the stochastic case, where
    $U=\sum_{s}\sum_{a\neq \pi^\star(s)} \frac{H^2\ln^3(T)}{\Delta(s,a)}$.  
\end{theorem}

\begin{theorem}\label{thm: unknown shannon}
    Under unknown transitions, \pref{alg: template} with Shannon entropy regularizer ensures for any $\pi$ 
    \begin{align*}
        \Reg(\pi) \lesssim \sqrt{H^4S^2AT\ln^2(T)\iota}  +\poly(H,S,A)\ln(T)\iota
    \end{align*}
    in the adversarial case, 
    and 
    \begin{align*}
        &\Reg(\pi) \lesssim U+\sqrt{U(\calC+\calC(\pi))}  +\poly(H,S,A)\ln(T)\iota 
    \end{align*}
    in the stochastic case, where
    $U=\frac{H^4S^2A\ln^2(T)\iota}{\Delta_{\min}}$. 
\end{theorem}

\newpage

\end{document}